\documentclass[12pt]{article}
\makeatletter
\def\maxwidth{ %
  \ifdim\Gin@nat@width>\linewidth
    \linewidth
  \else
    \Gin@nat@width
  \fi
}
\makeatother

\usepackage{latexsym}
\usepackage{amssymb,amsmath, bm, bbm}
\usepackage{graphicx}
\usepackage{marvosym}
\usepackage{lscape}
\usepackage{rotating}
\usepackage[nodisplayskipstretch]{setspace}
\usepackage{threeparttable}
\usepackage{booktabs}
\usepackage[compact]{titlesec}
\usepackage{multirow}
\usepackage{booktabs}
\usepackage{longtable}
\usepackage{array}
\usepackage[table]{xcolor}

\usepackage{caption}
\usepackage{subcaption}
\usepackage{enumitem}

\usepackage{algorithm}
\usepackage{algpseudocode}

\usepackage{arydshln}

\usepackage{natbib}

\usepackage{color}
\usepackage[pdftex, bookmarksopen=true, bookmarksnumbered=true,
pdfstartview=FitH, breaklinks=true, urlbordercolor={0 1 0}, citebordercolor={0 0 1}, colorlinks,allcolors=blue]{hyperref}

\usepackage{dcolumn}
\newcolumntype{.}{D{.}{.}{-1}}
\newcolumntype{d}[1]{D{.}{.}{#1}}

\usepackage{amsthm}
\theoremstyle{definition}
\newtheorem{theorem}{Theorem}[section]
\newtheorem{lemma}[theorem]{Lemma}

\newtheorem{corollary}[theorem]{Corollary}

\newtheorem{assumption}{Assumption}

\newcommand\indep{\protect\mathpalette{\protect\independenT}{\perp}}

\DeclareMathOperator{\argmin}{arg\min}
\def\independenT#1#2{\mathrel{\rlap{$#1#2$}\mkern2mu{#1#2}}}

\newcommand{\R}{\ensuremath{\mathbb{R}}}

\newcommand{\bbone}{\ensuremath{\mathbbm{1}}}

\newcommand{\E}{\ensuremath{\mathbb{E}}}

\newcommand{\symm}{\text{symm}}

\newcommand{\argmax}{\text{argmax}}

\def\b1{\boldsymbol{1}}

\def\spacingset#1{\renewcommand{\baselinestretch}%
{#1}\small\normalsize} \spacingset{1}

\usepackage{times}

\newcommand\edit{}

\begin{document}

\pagestyle{plain}

\newcommand{\blind}{0}

\newcommand{\tit}{\LARGE Policy Learning with Asymmetric
  Counterfactual Utilities}

\if0\blind

{\title{\tit\thanks{
  We acknowledge the partial support from Cisco Systems, Inc. (CG
  \#2370386), National Science Foundation (SES--2051196), Sloan
  Foundation (Economics Program; 2020--13946), National Natural Science of China (Grant No. 12371285, 12292984), and Fundamental Research Funds for the Central Universities, Sun Yat-sen University (Grant No. 23hytd010).  We also thank the
  IQSS's Alexander and Diviya Magaro Peer Pre-Review Program for feedback.}}
  \author{Eli Ben-Michael\thanks{Assistant Professor, Department of Statistics \& Data Science and Heinz College of Information Systems \& Public Policy, Carnegie Mellon University. 4800 Forbes Avenue, Hamburg Hall, Pittsburgh PA 15213.  Email: \href{mailto:ebenmichael@cmu.edu}{ebenmichael@cmu.edu} URL:
  \href{https://ebenmichael.github.io}{ebenmichael.github.io}} 
  \and  Kosuke
      Imai\thanks{Professor, Department of Government and Department of
        Statistics, Harvard University.  1737 Cambridge Street,
        Institute for Quantitative Social Science, Cambridge MA 02138.
        Email: \href{mailto:imai@harvard.edu}{imai@harvard.edu} URL:
        \href{https://imai.fas.harvard.edu}{https://imai.fas.harvard.edu}}
      \and Zhichao Jiang\thanks{Professor, School of mathematics, Sun Yat-sen University,
        Guangzhou Guangdong 510275. Email:
        \href{mailto:jiangzhch7@mail.sysu.edu.cn}{jiangzhch7@mail.sysu.edu.cn} URL:
        \href{https://zhichaoj-git.github.io}{https://zhichaoj-git.github.io}}}

\date{\today}

\maketitle
}\fi

\if1\blind
\title{\bf \tit}
\maketitle
\fi

\begin{abstract}

  Data-driven decision making plays an important role even in high
  stakes settings like medicine and public policy.  Learning optimal
  policies from observed data requires a careful formulation of the
  utility function whose expected value is maximized across a
  population.  Although researchers typically use utilities that
  depend on observed outcomes alone, in many settings the decision
  maker's utility function is more properly characterized by the joint
  set of potential outcomes under all actions.  For example, the
  Hippocratic principle to ``do no harm'' implies that the cost of
  causing death to a patient who would otherwise survive without
  treatment is greater than the cost of forgoing life-saving
  treatment.  We consider optimal policy learning with asymmetric counterfactual
  utility functions of this form that consider the joint set of potential
  outcomes.  We show that asymmetric counterfactual utilities
  lead to an unidentifiable expected utility function, and so we first
  partially identify it.  Drawing on statistical decision theory, we
  then derive minimax decision rules by minimizing the maximum expected utility loss
  relative to different alternative policies.  We show that one can learn
  minimax loss decision rules from observed data by solving intermediate
  classification problems, and establish that the finite sample
  excess expected utility loss of this procedure is bounded by the
  regret of these intermediate classifiers.  We apply this conceptual
  framework and methodology to the decision about whether or not to
  use right heart catheterization for patients with possible pulmonary
  hypertension.
  
  \end{abstract}
  
  \begin{center}
  \end{center}
  
\clearpage
\spacingset{1.8} 

\section{Introduction}
\label{sec:introduction}

The well-known Trolley Problem in ethics goes as follows:
\begin{quote}
  \singlespace
  Edward is the driver of a trolley, whose brakes have just failed. On
  the track ahead of him are five people; the banks are so steep that
  they will not be able to get off the track in time. The track has a
  spur leading off to the right, and Edward can turn the trolley onto
  it. Unfortunately there is one person on the right-hand
  track. Edward can turn the trolley, killing the one; or he can
  refrain from turning the trolley, killing the five
  \citep[][p. 206]{jarv:76}.
\end{quote}
Should Edward turn the trolley?  Is killing someone worse than letting
them die?  Such ethical dilemmas frequently confront us in moral and
legal debates concerning various issues that range from abortion to
self-driving cars \citep[e.g.,][]{foot:67,lin:16}.  Similarly, in the
ethics of modern medicine, the Hippocratic principle of ``do no harm''
remains influential \citep[e.g.,][]{jonsen1978,smit:05,wiens2019}.  In
the language of utility theory, a physician may assign a utility loss
of greater magnitude to the case where a new drug harms a patient than
to the case where not providing the new drug leads to the failure to
save a patient \citep[e.g.,][]{bord:09}.

These examples illustrate the potential applications of
\emph{asymmetric counterfactual} utilities
that depend not only on the observed outcome, but also on the counterfactual
outcome that could occur under a different action, and treat actions
differently depending on their corresponding potential outcomes.
Yet, to the best of
our knowledge, the existing literature on data-driven decision making
and algorithmic policy learning assumes that the decision maker's
utility function only depends on the observed outcome.

In this paper, we develop the methodological framework for optimal
policy learning with asymmetric counterfactual utilities, which includes standard utilities based on marginal outcomes as a special case.
We show that in general,
asymmetric counterfactual utilities lead to an unidentifiable expected utility
function.  Therefore, we partially identify the expected utility and propose to
minimize the maximum expected utility loss relative to a
particular comparison policy.  We consider the maximum expected utility loss relative
to constant policies such as always-treat and never-treat policies as
well as the oracle policy that has complete knowledge of the unidentifiable
terms in the expected utility function.
We demonstrate that one can learn minimax decision
rules from observed data by solving intermediate classification
problems.  We also establish that the finite sample regret of this
procedure is bounded by regret of these intermediate classifiers.

We use this framework to re-assess the use of Right Heart
Catheterization (RHC), an invasive diagnostic tool
\citep{Connors1996}.  We learn decision rules based on clinical
variables as we vary the asymmetry in the costs between
failing to prevent a patient's death and causing it via RHC.  These
decision rules differ depending on whether we minimize the worst-case
expected utility loss relative to a constant policy (always or never
using RHC) or the oracle policy that uses RHC optimally.  We inspect
how the choice of utility function and comparator affect the learned
decision rules, finding substantial variability based on these
choices.  Finally, we translate these findings into directly
interpretable patient outcomes, exhibiting a trade-off between limiting
the worst-case proportions of patients that the policy harms or fails to save.

\paragraph{Related Literature.} 
Recent years have seen an increased interest in algorithmic policy
learning from randomized control trials or observational data. Many of
these approaches follow a similar structure. First, quantify the
expected utility of a policy based on the marginal distributions of the
potential outcomes. Then, show how to identify the expected utility or
regret from
observable data and find a policy that optimizes an empirical analog.
These approaches typically use inverse propensity score
weighting or double-robust methods for the identification and
estimation steps \citep[see ][among others]{Zhao2012,Kitagawa2018,
  Athey2021}.  There is also a related literature that focuses on
identifying and estimating optimal policies in settings with
unmeasured confounding via instrumental variables
\citep[see][]{Cui2021_iv,Qiu2021_iv}.

More immediately relevant to our discussion here, recent work builds
off classical ideas in decision theory and treatment choice
\citep[e.g.][]{Manski2004, Manski2005_partial,Manski2011} and
considers scenarios where we cannot point identify the expected utility
function for possible policies.  One strand of work considers choosing
between two treatments or fixed decision rules based on a finite sample of
data, when treatment effects are partially identified
\citep[e.g.][]{Stoye2012,Ishihara2021,Yata2021}.
This work typically involves directly solving an empirical minimax
regret problem, but does not consider optimization over classes of
individualized policies.

In contrast, another line of work
considers learning optimal individualized decision rules in situations where
treatment effects are only partially identified, using an empirical
risk minimization approach.
These include settings with unmeasured confounding
\citep[e.g.,][]{Kallus2021,Pu2021, Han2021_partial,Cui2021_partial} or
limited overlap between different treatment conditions
\citep[e.g.,][]{benmichael2021_safe,Zhang2022_safe}.
\citet{DAdamo2023} considers a general setup where the conditional
expected potential outcomes and treatment effects are partially
identified.
These approaches take a minimax approach at the population level,
deriving the population-level minimum expected utility or maximum regret.
They then treat the population-level maximum regret or
negative minimum expected utility as a risk, and use empirical risk minimization
approaches and propensity score weighting or double-robust methods as above.
Our work is in the vein,
estimating individualized treatment rules via empirical risk
minimization.  However, we consider a different setting where
treatment effects \emph{are} point identified, but the expected utility
function is partially identified because it is a function of the
proportion of units within each principal stratum---an unidentifiable
quantity under standard designs.

Finally, \citet{Babii2021} also consider asymmetric utilities, but only using
observed outcomes.  In contrast, we consider asymmetric {\it
  counterfactual} utilities, which depend on potential outcomes and is
a generalization of \citeauthor{Babii2021}'s approach (see
Appendix~\ref{app:asymmetric} for details).

\paragraph{Paper outline.}
The paper proceeds as follows.
Section \ref{sec:prelim} describes the goal of policy learning with
asymmetric counterfactual utilities and reviews the standard symmetric case.
Section \ref{sec:pop_asymm} discusses partial identification of the expected
utility function and the minimax population policies relative to different 
alternatives. Section \ref{sec:emp_asymm} then shows how to estimate such
policies from data.
Finally,
Section \ref{sec:app} applies this framework to the use of RHC, and
Section \ref{sec:discussion} concludes.

\section{Preliminaries}
\label{sec:prelim}

In this section, we introduce the notation and assumptions used
throughout this paper.  We also discuss the nature of asymmetric counterfactual
utilities before providing a brief review of policy learning with
symmetric utilities, which is a special case of our proposed
framework.

\subsection{Notation and assumptions}

Suppose that we have a simple random sample of $n$ units from a super
population $\mathcal{P}$ where each unit $i=1,\ldots,n$ has a set of
characteristics $X_i \in \mathcal{X}$.  We consider a binary treatment
assignment decision $D_i \in \{0,1\}$, which can be made by either
individual $i$ or a policy maker.  We assume that the outcome $Y_i$ is
binary with $Y_i = 1$ indicating a desirable outcome (e.g., survival)
and $Y_i = 0$ representing an undesirable outcome (e.g., death).
Under the assumption that there is only one version of treatment and
no interference across units, we have two binary potential outcomes
for each unit $i$ where $Y_i(d) \in \{0, 1\}$ represents the potential
outcome under the scenarios where the unit receives the decisions
$D_i = d$ for $d=0,1$.

The setup implies that the observed outcome for unit $i$ can be
written as $Y_i = D_i Y_i(1) + (1-D_i)Y_i(0)$ and the tuple of random
variables $\{X_i, D_i, Y_i(1), Y_i(0)\}$ is assumed to be
independently and identically distributed.  Importantly, under this
setting, each unit belongs to one of the four \textit{principal
  strata} defined by the values of the two potential outcomes, i.e.,
$(Y_i(1), Y_i(0)) = (y_1, y_0) \in \{(0, 0), (0, 1), (1, 0), (1,1)\}$
\citep{fran:rubi:02}.  For example, the principal stratum
$(y_1, y_0) = (1, 0)$ represents a group of units who would yield the
desirable outcome only when they are treated, i.e., $D_i = 1$, whereas
the principal stratum $(y_1, y_0) = (1, 1)$ indicates a group of units
whose outcome is desirable regardless of the treatment decision.
Since we never observe the two potential outcomes at the same time for
any given unit, it is impossible to know which principal stratum each
unit belongs to without additional assumptions.  

Throughout this paper, for notational simplicity, we will drop the
individual $i$ subscript in expressions involving expectations over
the distribution.  We will also assume the strong ignorability and
strict overlap assumptions for observational studies and randomized
control trials.
\begin{assumption}[Strong ignorablity and strict overlap]
  \label{a:ignore} \spacingset{1}
  $\{Y(1), Y(0)\} \indep D \mid X$ and there exists an $\eta > 0$ such that
  $\eta < d(x) < 1 - \eta$ for all $x \in \mathcal{X}$ where $d(x) \equiv \Pr(D = 1 \mid X = x)$
  represents the propensity score.
\end{assumption}

The assumption allows us to identify the expected potential outcome
under decision $d$ given covariates $x$, denoted as
$m(d,x) \equiv \E[Y(d) \mid X = x]$.  However, it is impossible to
identify the {\it principal score}, or the conditional probability of
belonging to a principal stratum given covariates, 
defined as
$e_{y_1y_0}(x) \equiv \Pr(Y(1) = y_1, Y(0) = y_0 \mid X = x)$, because
we do not observe the two potential outcomes at the same time for a
given unit \citep{ding:lu:17, jiang:20}.

\subsection{Asymmetric counterfactual utilities}
\label{sec:assymetric_utilities}

We focus on deterministic individualized policies
$\pi:\mathcal{X} \to \{0,1\}$ that assign a binary treatment decision
to individual units according to their characteristics
$X \in \mathcal{X}$.  To learn optimal policies from the observed
data, we consider a utility function $u(d; y_1, y_0)$ that encodes the
utility for taking treatment decision $d$ for a unit in principal
stratum $(y_1,y_0)$.  Crucially, this utility function depends on the
values of \emph{both} potential outcomes.  This contrasts with the
standard utility function $u(d; y)$, which only depends on the
realized potential outcome $Y_i(d)=y$ under the decision $d$.
We measure the overall quality of a policy $\pi$
 by
its expected utility (also called the \emph{value} or \emph{social
  welfare}),
\begin{align}
    V(\pi) & = \E\left[\sum_{y_1 = 0}^1\sum_{y_0 = 0}^1 \bbone\{Y(1) = y_1, Y(0) = y_0\} \left\{u(0; y_1, y_0)(1 - \pi(X)) + u(1; y_1, y_0)\pi(X)\right\}\right]\nonumber\\
    & = \E\left[\sum_{y_1 = 0}^1\sum_{y_0 = 0}^1 e_{y_1y_0}(X) \pi(X) 
      \left\{u(1; y_1, y_0) - u(0; y_1, y_0)\right\}\right]  + \E\left[\sum_{y_1 = 0}^1\sum_{y_0 = 0}^1   e_{y_1y_0}(X) u(0; y_1, y_0) \right].   \label{eq:welfare}
\end{align}

This setup lets the utility vary across different counterfactual
outcomes even when the treatment decision and the realized outcome are
the same, allowing for a richer specification of the decision problem.
For example, the disutility from assigning treatment to a patient that
is harmed by it (i.e., $Y(1) = 0$ and $Y(0) = 1$) can be larger than
the disutility from assigning treatment to a patient for whom it is
useless (i.e., $Y(1) = Y(0) = 0$), despite the fact that the realized
outcome $Y(1)$ is identical in both cases.  A standard utility
function does not distinguish between these two cases, assigning each
a value of $u(1, Y(1))$.  This utility function also allows for
asymmetry in the utility gain or loss from treating a unit across
principal strata.  Returning to the Hippocratic oath, we can choose
the utility function such that the absolute magnitude of the utility
loss for harming a patient through treatment
($|u(1; 0, 1) - u(0; 0, 1)|$) is greater than that of the utility gain
when the same treatment benefits another patient
($|u(1; 1, 0) - u(0; 1, 0)|$).

To encode this, and focus on key ideas, we parameterize the utility function
as follows:
\begin{enumerate}
\item[(i)] the utility gain associated with a ``useful treatment,''
  i.e., $(y_1,y_0)=(1,0)$, is $u_g - c_g\equiv u(1;1,0) - u(0; 1,0)$
  (e.g., treating with a drug that would benefit the patient)
\item[(ii)] the utility loss associated with a ``harmful treatment,''
  i.e., $(y_1,y_0)=(0,1)$, is $-u_l -c_l \equiv u(1;0,1) - u(0;0,1)$
  (e.g., treating with a drug that would harm the patient)
\item[(iii)] the utility loss of treating with a ``harmless
  treatment,'' i.e., $(y_1,y_0)=(1,1)$, is
  $-c_1 = u(1;1,1) - u(0; 1,1)$ (e.g., treating with a drug that would
  not harm the patient)
\item[(iv)] the utility loss of treating with a ``useless treatment,''
  i.e., $(y_1,y_0)=(0,0)$, is $-c_0 = u(1;0,0) - u(0; 0,0)$ (e.g.,
  treating with a drug that would not benefit the patient)
\end{enumerate}
The values $c_g,c_l,c_1,c_0$ denote the cost of administering the
treatment $d=1$ relative to not doing so $d=0$ in each of the four
principal strata. The values $u_g$ and $u_l$ represent the magnitude
of the utility gain and loss for administering a useful and harmful
treatment, respectively.  In this setting, these utility values are
known and fixed by the decision maker.  Utility functions of this and
more general forms have been considered in the literature on decision
theory \citep[see, e.g.,][]{Stefansson2015,Bradley2017}.  Our focus
is, however, on the estimation of individualized decision rules under
these asymmetric counterfactual utility functions.

\begin{table}[t!]
  \begin{center}
\begin{tabular}{cccc}
\hline
 & $Y_i(0)=1 $ &  $Y_i(0)=0$ \\ \hline 
\multirow{2}{*}{$Y_i(1)=1$} & Harmless & Useful \\
& $-c$ & $u_g-c$ \\ \cdashline{1-3} 
\multirow{2}{*}{$Y_i(1)=0$} & Harmful & Useless \\
  & $-u_l-c$ & $-c$ \\   \hline 
\end{tabular}
\caption{Asymmetric counterfactual utility gain/loss for treating each of
  the four principal strata, relative to not treating.
  Each cell corresponds to the principal stratum defined by the values
  of the two potential outcomes, $Y_i(1)$ and $Y_i(0)$.  Each entry
  represents the utility gain/loss of treatment assignment, relative
  to no treatment, for a
  unit that belongs to the corresponding principal stratum
  ($u(1;y_1, y_0) - u(0; y_1, y_0)$). 
  $c$ is the cost of treatment assignment. We assume that
  $u_l,u_g >0$ and  $c \geq 0$.  Symmetric utilities are a special
  case with $u_g=u_l$.  } \label{tab:utilities}
\end{center}
\end{table}

Throughout, we will assume that the costs are identical, i.e.,
$c_g = c_b = c_1 = c_0 = c$. In addition, we assume $u_g$ and $u_l$
are positive, and $c$ is non-negative.  Table~\ref{tab:utilities}
summarizes this asymmetric counterfactual utility structure.
Fixing the costs to be
identical amounts to restricting the utility loss from a harmless and
useless treatment to be equal. Without this restriction there will be
an additional asymmetry due to the different costs, which would not
affect our development, except to make the notation more cumbersome
and results less interpretable.

Note that our asymmetric counterfactual utilities include symmetric
utilities based on observed outcomes as special cases, thereby
generalizing the standard setting considered in the policy learning
literature.  In Appendix~\ref{app:asymmetric}, we further show that
although it is possible to construct asymmetric utilities without
using principal strata \citep{Babii2021}, doing so still implies some
restrictions on the structure of the resulting counterfactual utilities
and hence they are a special case of our framework.

We will primarily be comparing two policies rather than considering
one in isolation.  We begin by defining the \emph{expected utility
  loss} of policy $\pi$ relative to another policy $\varpi$ as the
difference in values, $V(\varpi) - V(\pi)$.  Using the relations
$m(1, x) = e_{11}(x) + e_{10}(x)$ and
$m(0,x) = e_{11}(x) + e_{01}(x)$, we show in Appendix \ref{sec:proofs}
that we can write the expected utility loss in a simplified form:
\begin{equation}
  \label{eq:regret}
  R_{e_{01}}(\pi, \varpi) \equiv V(\varpi) - V(\pi) = \E\left[(\varpi(X) - \pi(X))\{u_g \tau(X) + (u_g - u_l) e_{01}(X) - c\}\right].
\end{equation}
where $\tau(x) \equiv \E(Y(1)-Y(0) \mid X = x) = m(1,x)-m(0,x)$ is the
conditional average treatment effect (CATE) given the covariates
$X = x$.
Comparing two policies to each other allows us to leave the baseline
utility for not treating a unit in principal stratum $(y_1, y_0)$,
$u(0; y_1, y_0)$ unspecified. In Appendix~\ref{sec:maximin} we directly
consider the expected utility of a single policy, and connect choices of the
baseline utility to our discussion below.

Three components contribute to the expected utility loss in
Eqn~\eqref{eq:regret}. First, the difference in the expected treatment
effects for those treated under policies $\varpi$ and $\pi$, scaled by
the utility gain for a useful treatment $u_g$, represents a symmetric
component of the utility, where we compare the marginal benefits of
the policies.  The second component is an asymmetric adjustment term,
and relates to the probability of belonging to the principal stratum
for whom the treatment is harmful (i.e., $(y_1,y_0)=(0, 1)$). This can
counteract the marginal benefit of treatment and is scaled by the
difference between the utility gain for a useful treatment and the
loss for a harmful treatment, i.e., $u_g - u_l$.  The final term $c$
corresponds to the difference in the overall costs of the two
policies.

The first and third components, the difference in effects and costs,
are point identifiable under Assumption~\ref{a:ignore}.  The second
component, however, is only \emph{partially identifiable} due to the
unidentifiability of the principal score $e_{01}(\cdot)$.  We use the
$e_{01}$ subscript for the expected utility loss in Eqn~\eqref{eq:regret} to
signify this fact.  Therefore, we cannot pinpoint whether any policy
is superior to any other policy in general. The remainder of this
paper focuses on handling this ambiguity.

\subsection{Policy learning with symmetric utilities: A review}
\label{sec:symmetric}

Before discussing policy learning under asymmetric counterfactual  utility
functions, we briefly review policy learning with \emph{symmetric}
utilities --- a special case of our framework --- where the absolute
magnitude of the utility gain when the treatment leads to a desirable
outcome is equal to that of the expected utility loss when it leads to an
undesirable outcome, i.e., $u_g=u_l$.  In this case, a policy can make
up for the loss from harming some units by the gain from benefiting
other units.  This can be seen in the following simplified version of
the expected utility loss in Eqn~\eqref{eq:regret}:
\begin{equation}
  \label{eq:symmetric_regret}
  R_\symm(\pi, \varpi) = \E\left[\{\varpi(X) - \pi(X)\}\{u_g \tau(X) - c\}\right].
\end{equation}

The symmetric utility does not involve the principal score
$e_{10}(\cdot)$, and is identifiable under Assumption~\ref{a:ignore}.
Thus, under this setting, the oracle optimal policy that minimizes the expected utility loss
relative to any other policy is
$\pi^\symm(x) \equiv \bbone\{u_g\tau(x) \geq c\}$.
This oracle policy assigns the treatment to all individuals with
characteristics $x$ if their expected utility gain of assigning the
treatment relative to not assigning it at least makes up for its cost, i.e.,
$u_g\tau(x) \geq c$.  Note that this is equivalent to maximizing the value $V(\pi)$
directly.

Under Assumption~\ref{a:ignore}, therefore, we can write the symmetric
expected utility loss in Eqn~\eqref{eq:symmetric_regret} in terms of
the observed data by using a scoring function $\Gamma_w(x, d, y)$ such
that $\E\left[\Gamma_w(X, D, Y) \mid X = x\right] = m(w, x)$. For
example, the Inverse Probability-of-treatment Weighting (IPW) scoring
function
\edit{uses the IP weighting function $\gamma_w(D, X) \equiv \frac{wD}{d(X)} + \frac{(1-w)(1 - D)}{1 - d(X)}$ to
weight the observed outcome by the inverse probability of
receiving the decision $d$:
$ \Gamma_w^\text{ipw}(X, D, Y) = Y \gamma_w(D, X)$.
  }
   An alternative is the Doubly
Robust (DR) scoring function that combines the observed outcomes and
their conditional expectations: \edit{
$ \Gamma_w^\text{dr}(X, D, Y) = m(w, X) + \{Y - m(w,
X)\}\gamma_w(D, X).  $}
With such a scoring rule, we can then write the symmetric expected
utility loss function as:
$\E\left[\{\varpi(X) - \pi(X)\}\left\{u_g (\Gamma_1(X, D, Y) -
    \Gamma_0(X, D, Y)) - c\right\}\right]$, where the observable
quantity $\Gamma_1(X, D, Y) - \Gamma_0(X, D, Y)$ has replaced the
causal quantity $\tau(X)$. See \citet{Knaus2020_dml} for a recent
review.

In order to empirically find optimal policies from data, recent
approaches estimate the propensity score $\hat{d}(\cdot)$ and/or the
conditional expected potential outcome $\hat{m}(\cdot, \cdot)$ to
create estimated scores $\widehat{\Gamma}(X_i, D_i, Y_i)$.  For
example, we can estimate \edit{the IP weights as $\hat{\gamma}_w(D, X) \equiv \frac{wD}{\hat{d}(X)} + \frac{(1-w)(1 - D)}{1 - \hat{d}(X)},$  the IPW scoring function as
  $
    \widehat{\Gamma}_w^\text{ipw}(X, D, Y) \equiv Y \hat{\gamma}_w(X, D),
  $
  and the DR scoring function as 
  $
  \widehat{\Gamma}_w^\text{dr}(X, D, Y) \equiv \hat{m}(w, X) + \{Y - \hat{m}(w, X)\}\hat{\gamma}_w(X, D).
$}
Then, we solve the sample analog of
Eqn~\eqref{eq:symmetric_regret}.  This leads to finding policy
$\hat{\pi}$ that solves the following optimization problem:
\[
  \min_{\pi \in \Pi} \ -\frac{1}{n}\sum_{i=1}^n \pi(X_i) \left\{u_g\left(\widehat{\Gamma}_1(X_i, D_i, Y_i) - \widehat{\Gamma}_0(X_i, D_i, Y_i)\right) - c\right\},
\]
where $\Pi$ represents the \emph{policy class} and restricts the
functional form of potential policies.  \citet{Athey2021} establish
strong asymptotic guarantees on the regret of the empirical
$\hat{\pi}$ relative to the best-in-class policy when using the DR
approach with appropriately cross-fit models \citep[see
also][]{Zhao2012,Kitagawa2018}.

\section{Policy learning with asymmetric counterfactual utilities}
\label{sec:pop_asymm}

We now turn to the problem of finding optimal policies in the general
asymmetric case where $u_g \neq u_l$.  We will first consider the
identification problems in the population --- i.e., with infinite
data.  We then show how to learn policies empirically from observed
data in Section~\ref{sec:emp_asymm}.
In Appendix~\ref{sec:constrained}, we consider an alternative
formulation as a constrained optimization problem.

\subsection{The oracle policy with an asymmetric counterfactual  utility function}

We begin by considering the oracle policy in the general asymmetric
case.  By direct computation,
the \edit{(unconstrained)} policy that has the maximal possible value 
with an asymmetric counterfactual  utility
function is given by:
\begin{equation}
  \label{eq:asymmetric_oracle}
  \pi^{o} \equiv \underset{\pi}{\argmax} \; V(\pi) =  \bbone\left\{\tau(\cdot) \geq \frac{u_l - u_g}{u_g} e_{01}(\cdot) + \frac{c}{u_g}\right\}.
\end{equation}
\edit{We refer to this as the \emph{oracle} policy, because it has 
access to the unknown (and generally unknowable) principal scores}.
Unlike in the symmetric case,
\edit{this policy} includes the principal score $e_{01}$, which is
unidentifiable under Assumption~\ref{a:ignore}. Since
$0 \leq e_{01}(x) \leq 1$ for all $x$, the asymmetric oracle policy
uses a varying threshold for assigning the treatment where the
threshold depends on the principal score $e_{01}(X)$.

The way in which the oracle policy depends on the principal score is
characterized in part by the nature of the asymmetry in the utility
function. Consider the case where treatment is costless
($c = 0$). When $u_l > u_g$, causing the undesirable outcome by
assigning the treatment is considered worse than failing to prevent
such an outcome by not providing the treatment.
This raises the threshold for assigning the treatment because the
expected effect must be larger in order to compensate for the downside
risk of causing the undesirable outcome.  As a result, this biases the
oracle policy towards inaction. Conversely, when $u_l < u_g$, it is
better to cause the undesirable outcome than fail to prevent it by
inaction. In this case, the threshold for the treatment assignment is
lower, biasing the oracle policy towards action.

\begin{figure}[t!]
  \centering
        \begin{subfigure}[t]{0.45\textwidth}  
    {\centering \includegraphics[width=\maxwidth]{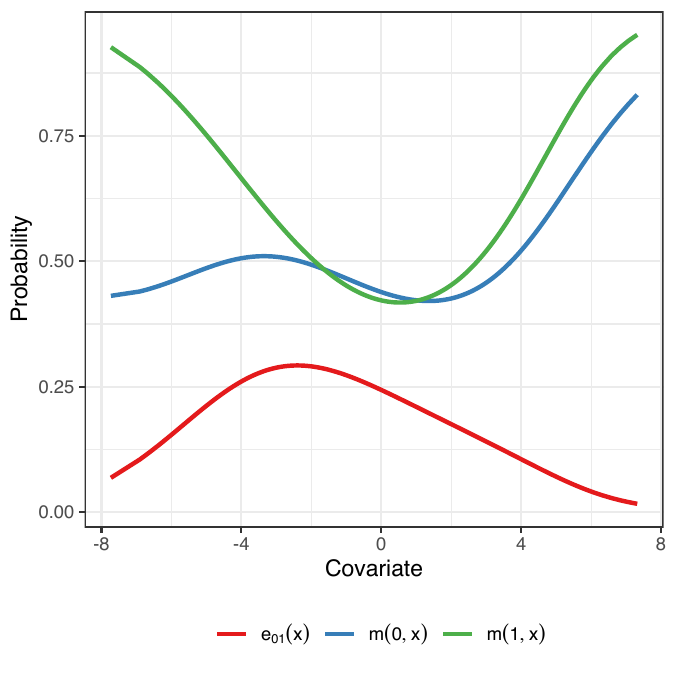} 
    }
    \vspace{-.2in}
    \caption{Setting}
    \label{fig:component_plot}
  \end{subfigure}
    \begin{subfigure}[t]{0.45\textwidth}  
  {\centering \includegraphics[width=\textwidth]{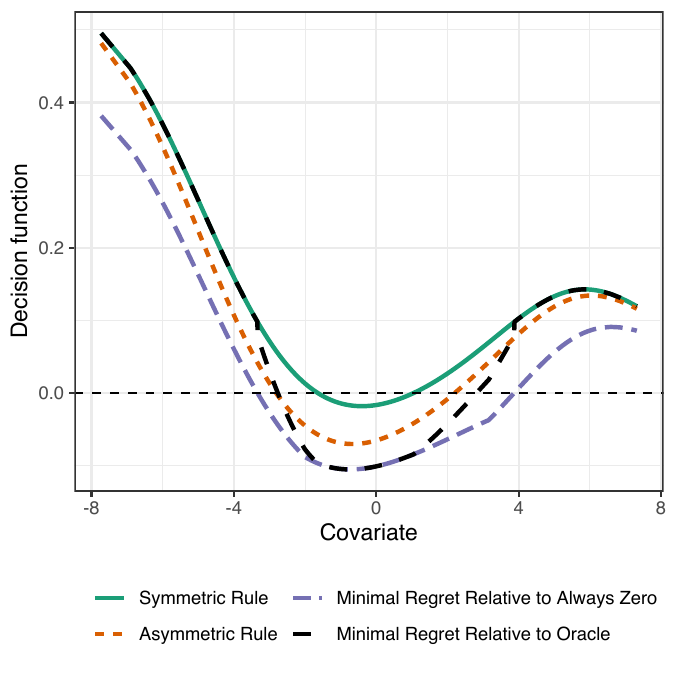} 
  }
  \vspace{-.2in}
    \caption{Decision rules}
    \label{fig:example_rules}
    \end{subfigure}\quad
    \caption{Example decision rules in a hypothetical example with a
      single covariate and cost $c = 0$. The left plot (a) presents
      the setting of the example where the values of the principal
      score $e_{01}(x)=\Pr(Y(1)=0, Y(0)=1 \mid X=x)$ are in red while
      the conditional expectations $m(d,x)=\E(Y(d) \mid X = x)$ are in
      blue for $d=0$ and green for $d=1$, respectively. The right plot
      (b) presents the decision rules corresponding to (i) the oracle
      in the symmetric case where $u_g = u_l$ (in green); (ii) the
      oracle in the asymmetric case where $u_g = 0.8$ and $u_l = 1$
      (in orange); (iii) the minimal expected utility loss solution relative to
      always not assigning the treatment $d=0$ (in purple); and (iv)
      the minimax regret solution relative to the oracle (in
      black). All rules have been transformed so that the policy takes
      decision 1 when the rule is greater than or equal to zero.}
  \label{fig:example}
\end{figure}

Figure~\ref{fig:component_plot} shows a one-dimensional example of
these decision rules where the cost is zero, $c = 0$.  The principal
score $e_{01}(x)$ is shown in red whereas the conditional expectations
$m(d,x)=\E(Y(d) \mid X=x)$ are shown in blue ($d=0$) and green
($d=1$).  Figure~\ref{fig:example_rules} shows the functions that make
up the decision rules in this example, centered so that the
corresponding policies assign $d = 1$ if the function is positive.
The symmetric case is shown in green, while the oracle in the
asymmetric case ($u_g=0.8$ and $u_l=1$) is in orange.  This plot also
shows two other solutions discussed in
Section~\ref{sec:max_regret_different_policies}; the minimal expected
utility loss solution relative to always not assigning the treatment
(purple), and the minimal regret solution relative to the oracle
(black).

In this example, providing the treatment $d=1$ leads to a higher
probability of the desirable outcome except near zero. Therefore, with
a symmetric utility function, the oracle policy would assign the
treatment in most cases (green). However, the asymmetric case is
different (orange). There is a region of the covariate space where the
principal score $e_{01}(x)$ is relatively high, leading to a
sufficiently high probability that the treatment causes the
undesirable outcome. Therefore, the asymmetric oracle rule has a
higher threshold for the treatment assignment, only providing the
treatment when the CATE $\tau(x)$ is large enough and the principal
score $e_{10}(x)$ is small enough.

\subsection{Partial identification and minimizing worst-case expected utility loss}
\label{sec:partial_identification_minimizing_worstcase}

Recall that the unidentifiability of the principal score $e_{01}(x)$
for any $x \in \mathcal{X}$ makes it impossible to identify the
expected utility loss in Eqn~\eqref{eq:regret} in the general
asymmetric case with $u_g \neq u_l$. However, we can \emph{partially
  identify} the principal score by deriving its sharp upper and lower
bounds, $L(x)$ and $U(x)$.  We then take a minimax approach, and find
the policy $\pi^\ast$ in the policy class $\Pi$ that minimizes the
maximal expected utility loss relative to an alternative policy
$\varpi$:
\begin{equation}
  \label{eq:minimax}
  \pi^\ast \in \underset{\pi \in \Pi}{\argmin}\  R_{\text{sup}}(\pi,
  \varpi) \quad \text{where} \quad R_{\text{sup}}(\pi, \varpi)  \ = \max_{e_{01}(x) \in [L(x), U(x)]} R_{e_{01}}(\pi, \varpi).
\end{equation}
Note that the maximum expected utility loss $R_{\text{sup}}(\pi, \varpi)$ is relative
to a \emph{particular} alternative policy $\varpi$, and is maximal
over all possible values for the principal score $e_{01}(x)$.
As we show below, the choice of this alternative policy will lead to
different objectives and optimal solutions \citep[see][for a recent
general discussion]{Cui2021_partial}.

Eqn~\eqref{eq:minimax} is an example of a treatment choice problem
under ambiguity \citep{Manski2005_partial,Manski2011}.  Such minimax
formulations of the problem have been widely considered in settings
where value functions depend on the marginal distribution of the
potential outcomes, but the CATE $\tau(x)$ is not point identified
\citep[e.g.][]{Manski2007,Stoye2012,Kallus2018,
  benmichael2021_safe,Ishihara2021,Yata2021,Zhang2022_safe,DAdamo2023}.
A key distinction between Eqn~\eqref{eq:minimax} and these other
problems, however, is that in our setting the value function depends
on the principal score $e_{10}(x)$, which is not point identified even
in randomized control trials.

To derive the sharp lower and upper bounds of the principal score, we
first define two classification functions:
\begin{equation}
  \label{eq:classification_funcs}
    \delta_+(x)  \ = \ \bbone\{m(0,x) + m(1,x) - 1 \geq 0\} \;\; \text{   and   } \;\; 
    \delta_\tau(x)  \ = \ \bbone\{\tau(x) \geq 0\}.
\end{equation}
Notice that the difference in the probability that both potential
outcomes are one and zero is given by
$e_{11}(x) - e_{00}(x) = m(0,x) + m(1,x) - 1$, which is the decision
function for classifier $\delta_+(x)$.  In other words, we have
$\delta_+(x) = 1$ if and only if $e_{00}(x) \leq e_{11}(x)$.  Thus, we
can view $\delta_+(x)$ as classifying whether there is a higher
probability that both potential outcomes are one rather than zero.
In contrast, noting that $\tau(x) = e_{10}(x) - e_{01}(x)$,
$\delta_\tau(x)$ classifies whether there is a higher probability
that the treatment is useful rather than harmful.
This corresponds to the symmetric oracle rule with cost
$c = 0$.

With these classifiers, we can use the Fr\'{e}chet bounds to find the
sharp lower and upper bounds for the principal score,
$e_{01}(x) \in [L(x), U(x)]$ for all $x$
\citep[e.g.][]{Heckman1997,jiang:16,Kallus2018}:
\begin{align}
    L(x) & \ = \ \max\{0, 1 - m(1, x) + m(0,x) -1 \} \ = \ \max\{0, -\tau(x)\} \ = \ -\tau(x)\{1 - \delta_\tau(x)\}\\
    U(x) & \ = \ \min\{m(0, x), 1 - m(1, x)\} \ = \ m(0, x) + \delta_+(x)\{1 - m(0,x) - m(1,x)\}.
\end{align}
These lower and upper bounds are sharp \citep{Ruschendorf1981} and are
point-identifiable from observable data.
With them we can create a point-identifiable objective.

\subsection{Worst case expected utility loss relative to different alternative policies}
\label{sec:max_regret_different_policies}

We now inspect the worst-case expected utility loss
$R_{\text{sup}}(\pi, \varpi)$ in Eqn~\eqref{eq:minimax} for different
choices of alternative policy $\varpi$. We consider three main
alternatives. First, the ``never-treat'' policy $\pi^\mathbbm{O}$ that
does not treat anyone, i.e., $\pi^\mathbbm{O}(x) = 0$ for all
$x$. Second, the ``always-treat'' policy $\pi^\mathbbm{1}$ that treats
everyone, i.e., $\pi^\mathbbm{1}(x) = 1$ for all $x$. In many cases,
the alternative to algorithmic decision making via a data-driven
policy is to take the same decision for everyone; thus these two
policies are of interest as they represent the standard of care in the
absence of an individualized policy (see Appendix~\ref{sec:maximin}
for connections to maximin policies that maximize the minimum expected utility).
We will denote the policies that minimize these worst-case losses as
$\pi^\ast_{\mathbbm{O}} \equiv \argmin_\pi R_{\sup}(\pi, \pi^\mathbbm{O})$ and
$\pi^\ast_{\mathbbm{1}} \equiv \argmin_\pi R_{\sup}(\pi, \pi^\mathbbm{1})$.

\edit{Finally, we consider the \emph{minimax regret policy} that
  minimizes the worst case expected utility relative to the value of
  the best-possible policy that has access to the principal scores
  $e_{01}(\cdot)$. Formally, the minimax regret policy is defined as
  $\pi^\ast_o \equiv \argmin_\pi \underset{e_{01}(x) \in [L(x),
    U(x)]}{\max} \max_{\pi'}R_{e_{01}}(\pi, \pi')$.  The definition of
  the oracle policy $\pi^o$ above implies that this is equivalent to
  choosing $\pi^o$ as the alternative policy.  That is,
  $\pi^\ast_o = \argmin_\pi R_{\sup}(\pi, \pi^o)$ where
  $R_\text{sup}(\pi, \pi^o) = \underset{e_{01}(x) \in [L(x),
    U(x)]}{\max} \max_{\pi'}R_{e_{01}}(\pi, \pi')$ is the
  \emph{regret} of policy $\pi$.}

Minimax regret policies are often studied in the policy learning
literature because alternatives, such as maximin policies, tend to be
too conservative \cite[see e.g.,][among many
others]{Manski2007,Manski2011,Stoye2012,Yata2021}.  \edit{Note that
  when defining the minimax regret policy across a constrained policy
  class, we compare to the best possible \emph{unconstrained} policy,
  i.e.,
  $\underset{\pi \in \Pi}{\argmin} \underset{e_{01}(x) \in [L(x),
    U(x)]}{\max} \max_{\pi'}R_{e_{01}}(\pi, \pi') = \underset{\pi \in
    \Pi}{\argmin} R_\text{sup}(\pi, \pi^o)$. The resulting policy will
  be different in general from the policy that minimizes the regret
  relative to the best-in class policy, and the unconstrained form of
  the regret will be larger.}

The following theorem shows that the worst-case expected utility loss relative to
each of these three policies takes a common form.

\begin{theorem}[Worst case expected utility loss] \spacingset{1}
  \label{thm:common_form}
  Let $\pi: \mathcal{X} \to \{0, 1\}$ be a deterministic policy.  For
  comparison policy $\varpi \in \{\pi^\mathbbm{O}, \pi^\mathbbm{1}, \pi^o\}$,
  the worst-case expected utility loss of $\pi$ relative to $\varpi$
  is
  \begin{equation}
    \label{eq:common_worstcase_regret}
    \begin{aligned}
      R_{\sup}(\pi, \varpi) & = C -\E\left[\pi(X)\left\{c_1^\varpi(X)m(1,X) + c_0^\varpi(X) m(0, X) + c^\varpi(X)\right\}\right]\\
      & =  C -\E\left[\pi(X)\left\{c^\varpi_1(X)\Gamma_1(X, D, Y) + c_0^\varpi(X) \Gamma_0(X, D, Y) + c^\varpi(X)\right\}\right],  
    \end{aligned}
  \end{equation}  
  where $C$ is a constant that does not depend on $\pi$, and $c_1^\varpi(\cdot), c_0^\varpi(\cdot),c^\varpi(\cdot):\mathcal{X} \to \R$ are functions
  that depend on $\delta_+(\cdot), \delta_\tau(\cdot), \pi^\ast_\mathbbm{O},$ or $\pi^\ast_\mathbbm{1}$.
\end{theorem}

The maximum expected utility loss objective in Theorem~\ref{thm:common_form} is a
weighted average of the expected potential outcomes under treatment
and no treatment plus a proxy for the cost. The choice of alternative
policy $\varpi$ determines these weights
$c^\varpi_0(\cdot), c^\varpi_1(\cdot)$ and cost $c^\varpi(\cdot)$, all
of which potentially vary with the covariates $X$; we give explicit formulas
for these functions in Appendix~\ref{sec:more_results}.
Note that the special case of a symmetric utility (Section~\ref{sec:symmetric})
is also of this form, with $c_1^\varpi(X) = -c_0^\varpi(X) = u_g$ and
$c^\varpi(X) = c$.  Similarly, the two
classifiers in Eqn~\eqref{eq:classification_funcs} have this
form, with $\delta_\tau$ corresponding to
$c_1^\varpi(X) = -c_0^\varpi(X) = 1$ and $c^\varpi(X) = 0$, and
$\delta_+$ corresponding to $c_1^\varpi(X) = c_0^\varpi(X) = 1$ and
$c^\varpi(X) = -1$.  

The second line of Eqn~\eqref{eq:common_worstcase_regret} shows how to
write the worst-case expected utility loss $R_{\sup}(\pi, \varpi)$ in
terms of observable data using the scoring functions $\Gamma_w$
(either IPW or DR) discussed in Section~\ref{sec:symmetric}.  So,
targeting the worst-case expected utility loss yields an objective
function that is identifiable, unlike the true expected utility
loss. As shown below, this allows us to construct decision rules based
on observable data that control the true expected utility loss by
minimizing the worst-case expected utility loss.

Constructing a utility function based on principal strata allows
decision makers to define their goals directly in terms of
individualized notions of useful and harmful treatments.
Nevertheless, Theorem~\ref{thm:common_form} shows that the minimax
expected utility loss problem reduces to a decision problem that only
involves the marginal distribution of the potential outcomes.  The
principal score $e_{01}(x)$ will not be involved in the remaining
estimation strategies and results, having been replaced with
point-identifiable upper and lower bounds.

However, the weighting and cost functions induced by the utility
function and choice of alternative policy correspond to a
covariate-dependent asymmetry in terms of the \emph{marginal}
potential outcomes.  Depending on the values of the nuisance
classifiers, Eqn~\eqref{eq:common_worstcase_regret} places more or
less weight on outcomes under treatment versus outcomes under
control. Thus, Eqn~\eqref{eq:common_worstcase_regret} is related to the
covariate-dependent loss minimization problem considered by
\citet{Babii2021} that depends on marginal outcomes, even though it
was derived from
placing an asymmetric counterfactual  utility on the principal strata. A key
distinction is that because Eqn~\eqref{eq:common_worstcase_regret}
involves the unknown nuisance classifiers, we must estimate the
corresponding loss function. We analyze the consequences of this in
Section~\ref{sec:theory_results}.

Finally, note that here we restrict to
\emph{deterministic} policies to derive the form of
the minimax expected utility loss in
Theorem~\ref{thm:common_form}. As \citet{Cui2021_partial} discusses, unlike
with the expected utility loss relative to the always treat or never
treat policies, allowing for \emph{stochastic} policies that randomize
between actions can lead to lower loss,
though this leads to a more complicated form. We leave further
understanding the implications for stochastic policies to future
work.

Next, we compute and inspect the policy that is the unconstrained
minimizer of the maximum expected utility loss in the population,
relative to each of the three alternative policies in turn. We will
then turn to estimating constrained policies in finite samples in
Section~\ref{sec:emp_asymm} below.

\subsubsection{Expected utility loss relative to a constant decision}
\label{sec:regret_constant}

We begin by considering the worst-case expected utility loss relative
to the never-treat policy.
\begin{corollary}[Minimax expected utility loss relative to the never-treat policy]
  \label{cor:regret_nobody} \spacingset{1}
  If $u_g \geq u_l$, the 
  solution to
  Eqn~\eqref{eq:minimax},
  $\pi^\ast_{\mathbbm{O}} \equiv \argmin_\pi R_{\sup}(\pi,
  \pi^\mathbbm{O})$ is the symmetric policy,
  \[
    \pi^\ast_\mathbbm{O}(x) = \mathbbm{1}\left\{\tau(x) \geq \frac{c}{u_g}\right\} = \pi^\symm(x).
  \]
  Otherwise, if $u_g < u_l$, it is given by,
  \[
    \pi^\ast_{\mathbbm{O}}(x) =
    \left\{
      \begin{array}{l r}
        \bbone\left\{m(1,x) \geq \frac{u_l}{u_g} m(0, x) + \frac{c}{u_g}\right\}, & \delta_+(x) = 0,\\
        \bbone\left\{m(1,x) \geq \frac{u_g}{u_l} m(0, x) + \frac{u_l - u_g + c}{u_l}\right\}, &   \delta_+(x) = 1.
      \end{array}\right.
  \]
\end{corollary}

Corollary~\ref{cor:regret_nobody} shows that the form of the minimax
expected utility loss policy depends on the direction of the
asymmetry. To build intuition, consider the case where the treatment
is costless ($c = 0$).  If $u_g > u_l$ --- so we would rather cause an
undesirable outcome than to fail to prevent it --- then the minimax
solution relative to the never-treat policy is the same as the optimal
rule under a symmetric utility function: assign the treatment when the
CATE is positive.  In this case, the unit is more likely to be in the
$(y_1,y_0) = (1,0)$ stratum than the $(y_1, y_0) = (0,1)$
stratum, and since $u_g > u_l$, it will be better to treat the unit
than to not.  Conversely, when the CATE is negative it may still be
better to treat the unit, but in the worst case it is not.  To
minimize the worst-case expected utility loss relative to never
treating, the minimax loss policy does not treat.
  
However, the minimax solution is different when $u_g < u_l$ --- i.e.,
when it is worse to cause an undesirable outcome than to fail to
prevent it. In this case, the oracle rule depends on the value of the
classifier $\delta_+(x) = \bbone\{e_{00}(x) \leq e_{11}(x)\}$. If both
potential outcomes are more likely to be zero than one, then the policy
only treats if the
probability that $Y(1)$ equals one is higher than the probability
that $Y(0)$ equals one by a factor of $\frac{u_l}{u_g} > 1$. Comparing
to the decision rule under the symmetric utility, we see that this
raises the threshold for assigning the treatment.

In contrast, if both potential outcomes are more likely to be one than
zero, the threshold is raised by adding a constant cost
$\frac{u_l - u_g}{u_l} > 0$, but the multiplicative factor on the
probability that $Y(0)$ equals one is $\frac{u_g}{u_l} < 1$.  Overall,
this has the effect of creating a more cautious policy that provides
the treatment less often.

Figure~\ref{fig:example_rules} shows the minimax decision rule
relative to $\pi^\mathbbm{O}$ (purple) in the one-dimensional example
where $u_g < u_l$ and $c = 0$ --- i.e., it is worse to cause
an undesirable outcome than fail to prevent it. In this case, we see that the
decision function is well below the symmetric rule shown in green
(i.e. the CATE), leading to a large part of the covariate space being
assigned no treatment even though the CATE is positive. In fact, this
policy is overly cautious: it does not assign the treatment even in
many cases where the oracle rule that knows the principal score would
provide the treatment. This is because the alternative policy is to
never treat anyone.

Appendix~\ref{sec:more_results} shows the result for the minimax
expected utility loss policy relative to the always-treat policy,
which is more aggressive than the symmetric policy. It is the mirror
image of the minimax loss policy relative to the never-treat policy,
with the relation to $u_g$ and $u_l$ reversed.

\subsubsection{\edit{The minimax regret policy}}
\label{sec:regret_oracle}

We next consider the \edit{policy that minimizes the} expected utility loss relative to the
oracle $\pi^o$ in Eqn~\eqref{eq:asymmetric_oracle}\edit{, or, equivalently, that minimizes the regret}.
For simplicity, we assume zero cost, i.e., $c = 0$;
when $c > 0$, there will be further terms (see the proof of
Theorem~\ref{thm:common_form} in Appendix~\ref{sec:proofs}).
\begin{corollary}[Minimax regret policy] \spacingset{1}
  \label{cor:regret_oracle}
  When $c= 0$, the minimax regret policy for
  $u_g \geq u_l$ is given by,

  {
    \singlespacing
  \[
    \pi^\ast_o(x) = \left\{
      \begin{array}{c c}
        1, & \delta_\tau(x) = 1,\\ 
        0, & \pi^\ast_{\mathbbm{1}}(x) = 0,\\
        \bbone\left\{m(1,x) \geq \frac{2u_l}{u_g + u_l} m(0,x)\right\}, & \delta_\tau(x) = 0, \delta_+(x) = 0,\\
        \bbone\left\{m(1,x) \geq \frac{u_g + u_l}{2u_l}m(0,x) + \frac{u_l - u_g}{2u_l}\right\}, & \delta_\tau(x) = 0, \delta_+(x) = 1,
      \end{array}\right.
  \]
  }

and for $u_g < u_l$ the minimax regret policy is given by,
  {
    \singlespacing
  \[
    \pi^\ast_o(x) = \left\{
      \begin{array}{c c}
        1, & \pi^\ast_{\mathbbm{O}}(x) = 1,\\
      0, & \delta_\tau(x) = 0,\\
        \bbone\left\{m(1,x) \geq   \frac{u_g + u_l}{2u_g} m(0,x)\right\}, & \delta_\tau(x) = 1, \delta_+(x) = 0,\\
        \bbone\left\{ m(1,x) \geq \frac{2u_g}{u_g + u_l}m(0,x) + \frac{u_l - u_g}{u_g + u_l}\right\}, & \delta_\tau(x) = 1, \delta_+(x) = 1.
      \end{array}\right.
  \]
  }
\end{corollary}

Corollary~\ref{cor:regret_oracle}
shows that we can write the worst-case regret and the minimax regret
policy in terms of observable data, just as for the constant policies
above. But, doing so requires \emph{four} classifiers rather than one:
(i) $\delta_+$, which classifies whether $e_{00}(x) \leq e_{11}(x)$;
(ii) $\delta_\tau$, which classifies whether the CATE is positive;
(iii) the minimax loss solution relative to $\pi^\mathbbm{1}$ and (iv) the
minimax loss solution relative to $\pi^\mathbbm{O}$.  Recall from
Corollary~\ref{cor:regret_nobody} that either $\pi^\ast_\mathbbm{O}$ (when
$u_g \geq u_l$) or $\pi^\ast_\mathbbm{1}$ (when $u_g < u_l$) is the
symmetric policy $\pi^\symm$.  Therefore, if the cost $c = 0$ as in
Corollary~\ref{cor:regret_oracle}, we only need \emph{three}
classifiers to construct the objective, since
$\pi^\symm = \delta_\tau$ in this case.

Inspecting the minimax solution relative to the oracle policy when
$u_g \geq u_l$, we see that it assigns the treatment if the symmetric rule does,
whereas it does not provide the treatment if the minimax
solution relative to the always-treat policy does not. In between
these two extremes, the decision rule lowers the threshold for the
treatment assignment relative to the symmetric rule.  The opposite is
true when $u_g < u_l$.  If the symmetric rule does not assign treatment,
the minimax solution relative to the oracle does not either, but it does provide
the treatment whenever the minimax solution relative to the
never-treat policy does.  In between these two cases, the threshold for
treatment assignment is higher than that under the symmetric rule.

Figure~\ref{fig:example_rules} shows the decision rule (black) in our
running one-dimensional example where $u_g < u_l$. The decision rule
is equivalent to $\pi^\ast_\mathbbm{O}$ (purple) when the CATE is
negative, and is equal to the CATE decision rule $\delta_\tau$ when
$\pi^\ast_\mathbbm{O}(x) = 1$.  When there is disagreement between the
CATE rule $\delta_\tau$ and $\pi^\ast_\mathbbm{O}$, the minimax oracle
rule interpolates between them, leading to a more aggressive policy  that treats more individuals
than $\pi^\ast_\mathbbm{O}$.  Comparing to the oracle rule
(orange), we see that this interpolation causes the decision
thresholds for the minimax oracle rule to be close to the best
possible decision thresholds.

\section{Learning a policy from data}
\label{sec:emp_asymm}

Having established the behavior and form of the minimax loss policy
$\pi^\ast$ in Eqn~\eqref{eq:minimax} in the population for an
unconstrained policy class, we now turn to the problem of learning a
policy $\hat{\pi}$ from observed data within a constrained policy class
$\Pi$.

\subsection{Estimation algorithms}

To begin, note that in finite samples we know neither the true scoring
functions $\Gamma_w$ nor the true weighting and cost functions
$c^\varpi_1(\cdot),c^\varpi_0(\cdot),c^\varpi(\cdot)$---which depend
on the nuisance classifiers---and so they must be estimated from data.
As mentioned in Section~\ref{sec:symmetric}, we can obtain estimates
of the DR score $\widehat{\Gamma}_w^\text{dr}$ by plugging in estimates
of the nuisance components.  Similarly, with estimates of the nuisance
classifiers, we can directly obtain estimates of the weighting and
cost functions $\hat{c}^\varpi_1(\cdot),\hat{c}^\varpi_0(\cdot)$, and
$\hat{c}^\varpi(\cdot)$ by plugging in to the formulas in
Theorem~\ref{thm:common_form}.

This leads to the following procedure.  First, obtain estimates of the
nuisance components $\hat{m}$ and $\hat{d}$ and construct the DR
scores.  Then, estimate the nuisance classifiers and follow
Theorem~\ref{thm:common_form} to construct estimates of the weighting
and cost functions.  To find a policy relative to either the
always-treat (if $u_g \geq u_l$) or never-treat (if $u_g < u_l$)
policies, we estimate a single nuisance classifier, $\hat{\delta}_+$.
Finding a policy relative to the oracle involves estimating
\emph{three} or \emph{four} nuisance classifiers: 
$\hat{\delta}_+$, $\hat{\delta}_\tau$, and the minimax loss policies
relative to never and always treating, $\hat{\pi}_\mathbbm{O}$ and
$\hat{\pi}_\mathbbm{1}$.  With these in hand, we then
find a data-driven policy $\hat{\pi}$ that solves the following
optimization problem (dropping the constant that does not depend on
the policy $\pi$):
\begin{align}
    \label{eq:opt_pol_emp}
  & \hat{\pi} \in \underset{\pi \in \Pi}{\argmin}  \ \hat{R}_{\sup}(\pi,
  \varpi) \\
   \text{where} \ & \hat{R}_{\sup}(\pi,
  \varpi) =
     -\frac{1}{n}\sum_{i=1}^n\pi(X_i)\left\{\hat{c}^\varpi_1(X_i)\widehat{\Gamma}^\text{dr}_1(X_i,
     D_i, Y_i) + \hat{c}^\varpi_0(X_i)
     \widehat{\Gamma}^\text{dr}_0(X_i, D_i, Y_i) +
     \hat{c}^\varpi(X_i)\right\}.\nonumber
\end{align}

There are two ways to estimate the nuisance classifiers.  The first is
an \emph{empirical risk minimzation} approach, where we solve
Eqn~\eqref{eq:opt_pol_emp} with the appropriate weighting and cost
functions.\footnote{For the nuisance classifiers $\delta_+$
  and $\delta_\tau$, the weighting and cost functions are known, and
  so need not be estimated.}  Appendix~\ref{sec:algos} explicitly
details this procedure.  As shown in Section \ref{sec:theory_results}
below, the estimated nuisance classifiers must have low regrets
relative to the true ones in order for our learned policy $\hat{\pi}$
to have low worst-case expected utility loss; therefore, we must
choose a flexible policy class.
This is in contrast to estimating our policy of interest
$\hat{\pi}$, whose performance we measure relative to the best
possible constrained policy.  An alternative is to take a
\emph{plug-in} approach, using our estimates of the conditional
expectation function $\hat{m}$ to directly create estimates of the
classifier; e.g.,
$\hat{\delta}_+(x) = \bbone\{\hat{m}(1,x) + \hat{m}(0,x)
\geq 1\}$ and
$\hat{\delta}_\tau(x) = \bbone\{\hat{m}(1,x) -
\hat{m}(0,x) \geq 0\}$.

\subsection{Excess worst-case expected utility loss}
\label{sec:theory_results}

To understand the statistical properties of our learned minimax policy
$\hat{\pi}$, we will compare it to the policy $\pi^\ast$ that
minimizes the worst-case expected utility loss in the population among
those in the policy class $\Pi$ by solving Eqn~\eqref{eq:minimax}.
For a given alternative policy $\varpi$, we will use the excess
worst-case expected utility loss
$R_\text{sup}(\hat{\pi},\varpi) - R_\text{sup}(\pi^\ast,\varpi)$ to
measure the quality of the learned minimax loss policy $\hat{\pi}$
since $R_\text{sup}(\pi^\ast,\varpi)$ is the best possible expected
utility loss in the worst case.  We assume that the nuisance
components and classifiers have been obtained from a separate sample,
and so can be treated as fixed for our finite sample results.
\edit{However, our results can be extended to solving
  Eqn~\eqref{eq:opt_pol_emp} by cross-fitting nuisance components and
  classifiers to obtain the estimates of
  $\widehat{\Gamma}_w^\text{dr}(X_i, D_i, Y_i)$ and
  $\hat{c}_w^\omega(X_i)$ \citep[see][and Appendix~\ref{sec:crossfit}]{Athey2021}.}

To state our results, we define several new quantities.
First, we measure the quality of the estimated nuisance classifiers,
$\hat{\delta}_+$ and $\hat{\delta}_\tau$ ,
by their regrets,
\begin{align*}
  R_+(\hat{\delta}_+) & \equiv \E\left[\bbone\{\hat{\delta}_+(X)
                        \neq \delta_+(X)\}\ | m(1,X) + m(0,X) - 1|\right]\\
  R_\tau(\hat{\delta}_\tau) & \equiv
                              \E\left[\bbone\{\hat{\delta}_\tau(X)
                              \neq \delta_\tau(X)\}\ | m(1,X) - m(0,X)|\right],
\end{align*}
where  $\hat{\delta}_+$  and $\hat{\delta}_\tau$ are treated as fixed and the
covariate $X$ is random.
Second, we
measure the complexity of the policy class $\Pi$ by its ability to
overfit to noise via the \emph{population Rademacher complexity}
$$
\mathcal{R}_n(\Pi) \ \equiv \ \E_{X, \varepsilon}\left[\sup_{\pi \in \Pi}\left|\frac{1}{n}\sum_{i=1}^n \varepsilon_i \pi(X_i)\right|\right],
$$
where $\varepsilon_i$ are i.i.d. random variables with
$\Pr(\varepsilon_i=1)=\Pr(\varepsilon_i=-1)=1/2$, and the expectation
is taken over both $\varepsilon_i$ and $X_i$ \citep[ \S
4]{wainwright_2019}.

We now present two finite sample bounds on the excess worst case
expected utility loss, one for learning a minimax loss policy relative
to the always or never treat policies
(Theorem~\ref{thm:excess_regret_ones_policy}), and the other for
learning a minimax loss policy relative to the oracle
(Theorem~\ref{thm:excess_regret_oracle}).
\begin{theorem}
  \label{thm:excess_regret_ones_policy} \spacingset{1}
  Let $\hat{\pi}$  solve Eqn~\eqref{eq:opt_pol_emp} with alternative policy 
  $\varpi = \pi^\mathbbm{O}$  (if $u_g < u_l$) or $\varpi = \pi^\mathbbm{1}$ 
  (if $u_g \geq u_l$),
  and with nuisance functions
  $\hat{m}$ and $\hat{d}$ and classifier $\hat{\delta}_+$ fit on a separate sample.
  Let $\pi^\ast$ solve the population problem in Eqn~\eqref{eq:minimax}.
  The excess worst-case expected utility loss of $\hat{\pi}$ relative to $\pi^\ast$ satisfies
  \begin{align*}
    R_\text{sup}(\hat{\pi}, \varpi) - R_\text{sup}(\pi^\ast, \varpi) & \leq 2U \times \left\{\frac{6 + \eta}{\eta} \times \left(2\mathcal{R}_n(\Pi) +  \frac{t}{\sqrt{n}}\right)+ \sum_{w=0}^1\left\|\hat{\gamma}_w - \gamma_w\right\|_2 \left\|\hat{m}(w,\cdot) - m(w,\cdot)\right\|_2\right\}\\
    & \qquad + (u_g - u_l) \times \left\{R_+(\hat{\delta}_+) + \frac{t}{2\sqrt{n}}\right\},
  \end{align*}
  with probability at least $1 - 2\exp\left(-\frac{t^2}{2}\right)$, where \edit{$\eta$ is the overlap parameter in Assumption~\ref{a:ignore},} $U$ is
  a constant depending on the utility values,
   $\left\|\hat{\gamma}_w - \gamma_w\right\|_2^2 \equiv \E\left[\left\{\hat{\gamma}_w(D,X) - \gamma_w(D,X)\right\}^2\right]$ and $\left\|\hat{m}(w,\cdot) - m(w,\cdot)\right\|_2^2 \equiv \E[\left\{\hat{m}(w,X) - m(w,X)\right\}^2]$.
\end{theorem}
\begin{theorem}
  \label{thm:excess_regret_oracle} \spacingset{1}
  Let $\hat{\pi}_o$  solve Eqn~\eqref{eq:opt_pol_emp} with alternative policy
  set to be the oracle, $\varpi = \pi^o$, and with nuisance functions
  $\hat{m}$ and $\hat{d}$ and classifiers $\hat{\delta}_+, \hat{\delta}_\tau,$
  $\hat{\pi}_\mathbbm{O}$, and $\hat{\pi}_\mathbbm{1}$ fit on a separate sample.
  Let $\pi^\ast_o$ solve the population problem in Eqn~\eqref{eq:minimax}.
  The excess worst-case regret of $\hat{\pi}_o$ relative to $\pi^\ast_o$ satisfies
  \begin{align*}
    R_\text{sup}(\hat{\pi}_0, \pi^o) - R_\text{sup}(\pi^\ast_o, \pi^o) & \leq U \times \left(\frac{6 + \eta}{\eta} \times \left(2\mathcal{R}_n(\Pi) +  \frac{t}{\sqrt{n}}\right)+\left\|\hat{\gamma} - \gamma\right\|_2 \sum_{w=0}^1\left\|\hat{m}(w,\cdot) - m(w,\cdot)\right\|_2 \right)\\
    & \qquad + 2 \times (R_\text{sup}(\hat{\pi}_\mathbbm{1}, \pi^\mathbbm{1}) -  R_\text{sup}(\pi^\ast_\mathbbm{1}, \pi^\mathbbm{1})) +  2 \times (R_\text{sup}(\hat{\pi}_\mathbbm{O}, \pi^\mathbbm{O}) -  R_\text{sup}(\pi^\ast_\mathbbm{O}, \pi^\mathbbm{O}))\\
    & \qquad +  (u_g - u_l) \times \left(R_+(\hat{\delta}_+) + R_\tau(\hat{\delta}_\tau) + \frac{t}{2\sqrt{n}}\right),
  \end{align*}
  with probability at least $1 - 2\exp\left(-\frac{t^2}{2}\right)$,
  where $U$ is a constant depending on the utility values.
\end{theorem}

Theorems~\ref{thm:excess_regret_ones_policy} and
\ref{thm:excess_regret_oracle} reveal three reasons why the
data-specific policy $\hat{\pi}$ can differ from the population policy
$\pi^\ast$.  First, as captured via the Rademacher complexity term,
even if the outcome model, propensity score model, and nuisance
classifiers were all known, $\hat{\pi}$ could simply over fit to noisy
data.  Fortunately, we can choose the complexity of $\Pi$ and often
prefer a relatively simple policy class for its interpretability and
transparency. The results above will be relative to the best possible
policy in the selected policy class.  Thus, we could control this by
limiting the complexity of our search space.  For example, if the
policy class $\Pi$ has a finite VC dimension $\nu$, the Rademacher
complexity scales like
$\mathcal{R}_n(\Pi) = O\left(\sqrt{\frac{\nu}{n}}\right)$ \citep[ \S
5]{wainwright_2019}.

Second, there is error in our estimates of the outcome and propensity
score models.  However, following \citet{Athey2021}, using the DR
scores protects against this error; only the product of the errors
enter the bound, which decreases faster than $1/\sqrt{n}$ under
typical assumptions.  These two sources of error occur in symmetric
policy learning problems. In the symmetric case when $u_g = u_l$ (and
so $\pi^\mathbbm{1} = \pi^\mathbbm{O} = \pi^{o}$),
Theorem~\ref{thm:excess_regret_ones_policy} is a special case of the results in
\citet{Athey2021}.

Finally, there is error in the nuisance classifiers, which is
particular to our setting.\footnote{This type of error structure
  appears in other policy learning settings with partial identification
  \citep{DAdamo2023}.}  For the minimax loss policy relative to never
or always treating, this error is measured by the regret for 
$\hat{\delta}_+$ : if it correctly classifies cases that are not very close
to the decision boundary (i.e. $|m(0,x) + m(1,x) - 1|$ is not near zero),
this component will be small.  Similarly for the minimax loss policy relative to
the oracle, there are additional terms from the regret of
$\hat{\delta}_\tau$ and the excess worst case expected utility for the
minimax loss policies relative to always and never treating.

If we estimate the nuisance classifiers via empirical risk
minimization, results from \citet{Kitagawa2018,Athey2021} (and
Theorem~\ref{thm:excess_regret_ones_policy} for
$\hat{\pi}_\mathbbm{1}$ and $\hat{\pi}_\mathbbm{O}$) imply that the
regret will primarily be controlled by the complexity of the policy
classes we optimize over for the nuisance classifiers.  Unlike for the
minimax loss policy class $\Pi$, unless the nuisance classifier class
contains the \emph{true} function, there will be irreducible
approximation error in the misclassification term.  Therefore, we
might choose more complex classes, in which case the regret of the
nuisance classifiers will primarily control the overall excess
expected utility loss.

\edit{To analyze the plug-in approach, we use a different characterization of the complexity of the learning problem:}
the proportion of cases that are close
to the decision boundary.
Focusing on the nuisance classifier $\delta_+$,
we follow \citet{Audibert2007} and characterize this via the following \emph{margin condition}.
\begin{assumption}[Margin condition] \spacingset{1}
  \label{a:margin_posclass}
  There exists an $\alpha > 0$ and a constant $C$ such that for any $t \geq 0$,
    $\Pr(|m(1,X) + m(0,X) - 1| \leq t) \leq Ct^\alpha$.
\end{assumption}
The margin parameter $\alpha$ determines how many cases are allowed to be close
to the boundary, with a larger value leading to a stronger assumption that
fewer cases are close; e.g. if $X$ has a bounded density, then $\alpha \geq 1$.
\edit{Note that the margin condition also leads to faster convergence rates for empirical risk minimization approaches, provided the policy class contains the true classifier.}
See \citet{Audibert2007} for further discussion.
Under this margin condition, we can further bound the regret
of the plug-in nuisance classifier $R_+(\hat{\delta}_+)$, leading
to the following corollary to Theorem~\ref{thm:excess_regret_ones_policy}.

\begin{corollary}
  \label{cor:excess_regret_ones_policy_plugin_posclass} \spacingset{1}
  Under Assumption~\ref{a:margin_posclass} and the conditions of
  Theorem~\ref{thm:excess_regret_ones_policy}, using the plug-in nuisance
  classifier $\hat{\delta}_+(x) = \bbone\{\hat{m}(1,x) + \hat{m}(0,x) \geq 1\}$,
  the excess worst-case regret of $\hat{\pi}$ relative to $\pi^\ast$ satisfies
  \begin{align*}
    R_\text{sup}(\hat{\pi}, \varpi) - R_\text{sup}(\pi^\ast, \varpi) & \leq 2U \times \left(\frac{6 + \eta}{\eta} \times \left(2\mathcal{R}_n(\Pi) +  \frac{t}{\sqrt{n}}\right)+\left\|\hat{\gamma} - \gamma\right\|_2 \sum_{w=0}^1\left\|\hat{m}(w,\cdot) - m(w,\cdot)\right\|_2 \right)\\
    & \qquad + (u_g - u_l) \times \left(2^{1+\alpha}C \|\hat{m} - m\|_\infty^{1+\alpha} + \frac{t}{2\sqrt{n}}\right),
  \end{align*}
  with probability at least $1 - 2\exp\left(-\frac{t^2}{2}\right)$,
  where $\|\hat{m} - m\|_\infty \equiv \sup_{w,x} \left|\hat{m}(w,x) - m(w,x)\right|$,
  and $U$ is a constant depending on the utility values.
\end{corollary}
With the plug-in nuisance classifier, $R(\hat{\delta}_+)$
is controlled by the error in the outcome model; however for
$\alpha > 0$ this error will be raised to a higher power, leading to a
faster rate.  In Appendix~\ref{sec:more_results}, we show an analogous
result for the minimax policy relative to the oracle using plug-ins
for all nuisance classifiers.  See \citet{DAdamo2023} for an
application of these techniques for policy learning in a different
partial identification setting, and \citet{Kallus2022_harm} for an
application to estimate bounds on $\Pr(Y(1) < Y(0))$.

Finally, although the minimax loss policies we consider are designed
to minimize the worst-case expected utility loss, in some cases it may
be possible that the true, unidentifiable expected utility loss may
also be small.  In Appendix~\ref{sec:sims}, we conduct a brief
simulation study to inspect how the misclassification rates and the
true expected utility loss behave in finite samples.

\section{Application to Right Heart Catheterization}
\label{sec:app}

We now apply the proposed methodology to a particular decision
problem: whether or not to use Right Heart Catheterization (RHC) in a
clinical setting. RHC is a diagnostic tool where a catheter is
inserted into the pulmonary artery. In a controversial observational
study, \citet{Connors1996} found that RHC led to an increase in
mortality on average.  RHC, however, can lead to life-saving treatment
for some patients.  In this section, we will use the data from
\citet{Connors1996} to learn policies for using RHC for certain
patients, inspecting how asymmetry in the policy maker's utility
function can lead to different data driven decision-making processes.

\subsection{Data and setup}
\label{sec:data_setup}

The data from \citet{Connors1996} include $n = 5,735$ ICU patients,
2,184 of whom had RHC applied.  We will code the outcome $Y(d) = 1$ as
survival by thirty days.  In this case, the utility value $u_g$
represents the utility gain in saving a patient's life under RHC who
would otherwise die without RHC, and $u_l$ represents the cost of RHC
leading to the death of a patient who would otherwise survive. In this
study, RHC use was not experimentally randomized and so we will be
relying on Assumption \ref{a:ignore}, using the same set of
socioeconomic and health characteristics as those used by
\citet{Hirano2001} in their propensity score-based analysis.

Throughout, we will use the estimated doubly robust score
$\widehat{\Gamma}^\text{dr}_w$. To do so, we need estimates of the
conditional expectation function $m(w, \cdot)$ and the propensity
score $d(\cdot)$.  We use a three-fold cross-fitting procedure to estimate these
conditional expectations
\edit{that we detail in Appendix~\ref{sec:crossfit}}.
With the
combined DR scores, we estimate that RHC leads to an overall increase
in 30-day mortality by 4.5 percentage points, with an estimated
standard error of 1.2 percentage points.  This result is consistent
with the findings of other existing analyses.

To fit each of these models, we use the full set of socioeconomic and
health characteristics in the data. We will consider, however,
decision rules that only use a subset of the covariates
$\mathcal{V} \subset \mathcal{X}$.  As we outline in
Appendix~\ref{sec:fewer_covs}, Theorem~\ref{thm:common_form} implies
that the worst-case utility loss will involve nuisance classifiers
based on the covariates $V$, using
$m_\mathcal{V}(w,v) \equiv \E[Y(w) \mid V = v]$, the expected
potential outcome given only the covariates $V$.  To adjust for
confounding, however, we still require the DR-scores using the full
set of covariates.  To construct plug-in estimates of the nuisance
classifiers, we use a variant of the DR-learner
\citep{Kennedy2022_drlearner}, regressing the DR scores
$\widehat{\Gamma}^\text{dr}_w$ on the covariates $V$ using gradient
boosted decision stumps.

\subsection{Threshold decision rules with two variables}

\begin{figure}[t!]
  \centering
        \begin{subfigure}[t]{0.45\textwidth}  
    {\centering \includegraphics[width=\maxwidth]{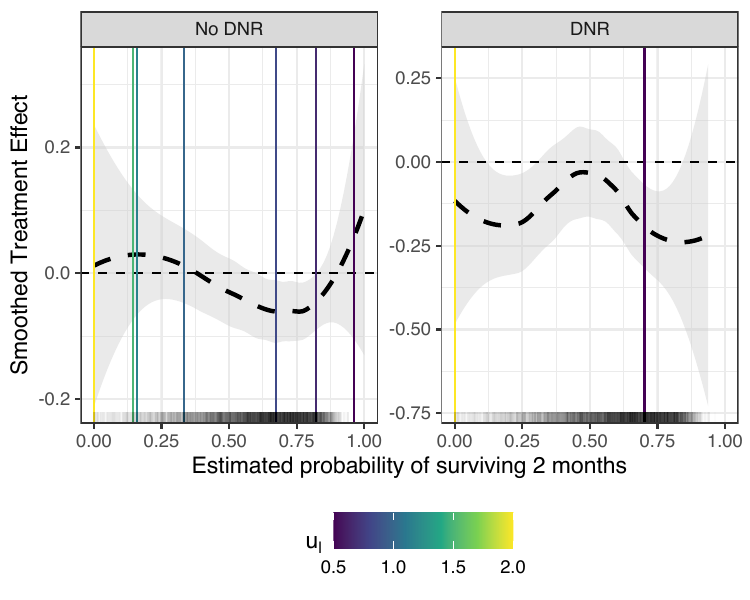}
    }
    \vspace{-.2in}
    \caption{Relative to always or never using RHC}
    \label{fig:two_dim_soc}
  \end{subfigure}
    \begin{subfigure}[t]{0.45\textwidth}  
  {\centering \includegraphics[width=\textwidth]{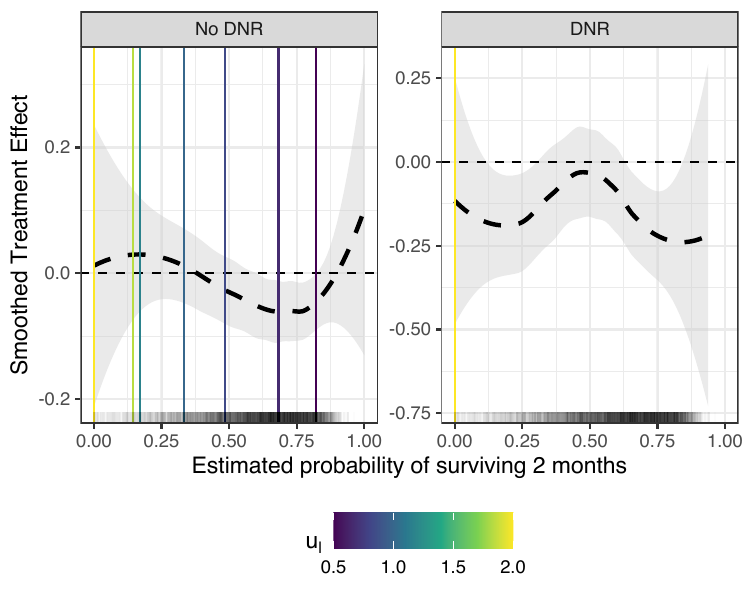}
  }
  \vspace{-.2in}
    \caption{Relative to oracle policy}
    \label{fig:two_dim_oracle}
    \end{subfigure}\quad
    \caption{Minimax threshold decision rules relative to (a) using
      RHC for all patients or no patients and (b) the oracle policy as $u_l$
      varies
      and $u_g = 1$ using two variables: the estimated probability of
      survival and whether a patient has a Do-Not-Resuscitate (DNR)
      order.
      The rules assign RHC if the estimated probability of survival is below
      the threshold.
      The results are shown separately for patients with
      (right) and without (left) an DNR order.  The dashed line and
      shaded area are the smoothed estimate of the conditional average treatment
      effect and 95\% confidence interval, respectively.}
  \label{fig:two_dim}
\end{figure}

We begin by considering decision rules that only use two clinical
variables: the estimated probability of surviving two months and
whether the patient has a Do-Not-Resuscitate (DNR) order.  Throughout
we will use threshold decision rules that assign RHC via a cutoff on
the estimated probability of surviving two months, using separate
thresholds for DNR and non-DNR patients.

First, we consider minimizing the worst-case expected utility loss
relative to using RHC for all patients or never using RHC.  We
estimate the nuisance classifier classifier
$\hat{\delta}_{+}$ via the plug-in approach
(Appendix Figure~\ref{fig:two_dim_pos_class} shows the resulting
 classifier).  We then create estimates of the weighting and cost
functions
$\hat{c}_0^\mathbbm{1}(\cdot), \hat{c}_1^\mathbbm{1}(\cdot),
\hat{c}^\mathbbm{1}(\cdot)$ and solve Eqn~\eqref{eq:opt_pol_emp} to
estimate the minimax policy $\hat{\pi}_\mathbbm{1}$ relative to never
using RHC (when $u_g < u_l$) and always using RHC (when
$u_g \geq u_l$).  We set $u_g = 1$ and vary $u_l \in [0.5, 2]$ so that
the utility loss from a harmful treatment moves between half and twice
as large as the utility loss from failing to give useful treatment.

Figure~\ref{fig:two_dim_soc} shows the resulting decision rules for
patients with (right) and without (left) a DNR order.  We also
estimate the CATE conditioned on the estimated probability of survival
and DNR status with the DR-learner using kernel smoothing
\citep{Kennedy2022_drlearner}. Note that the estimated CATE is
positive for non-DNR patients with less than a 50\% or greater than 80\%
probability of survival. Because we restrict to a single threshold,
in the symmetric case, this leads to a decision rule that
applies a threshold of ~50\% for non-DNR patients, while never
assigning RHC to DNR patients.  As the utility gain in saving a life
becomes greater than the cost of causing death, the estimated
threshold increases, leading to a decision rule that uses RHC for
non-DNR patients with a higher estimated probability of
surviving. Eventually the asymmetry is so large in favor of
prioritizing useful treatment that almost all non-DNR patients and
most DNR patients would be given RHC, even though the CATE is
negative.  Conversely, as avoiding harm becomes more important, the
decision threshold lowers, assigning RHC for fewer and fewer patients
until no patients would receive it.

Next, we consider finding the minimax regret policy relative to the
oracle $\hat{\pi}_o$, using plug-in estimates of the classifiers
$\hat{\delta}_+$, $\hat{\delta}_\tau$, $\hat{\pi}_\mathbbm{O}$, and
$\hat{\pi}_\mathbbm{1}$.\footnote{Note that we use plug-ins for
  $\hat{\pi}_\mathbbm{O}$ and $\hat{\pi}_\mathbbm{1}$ rather than the
  simple threshold decision rules above, to try to capture the best
  possible \emph{unconstrained} classifiers rather than the best
  possible constrained ones.}  Figure~\ref{fig:two_dim_oracle} shows
the decision functions.  Similar to the minimax loss policy relative
to always or never using RHC, as $u_l$ decreases the threshold for
non-DNR patients increases and as $u_l$ increases the threshold
decreases.  Even in the extreme case with $u_l = 0.5$, however, DNR
patients are not assigned RHC.  When $u_l = 2$, DNR patients with a
low probability of survival are still assigned RHC.  Mirroring our
discussion in Section~\ref{sec:max_regret_different_policies},
measuring regret relative to the best possible policy leads to a less
aggressive decision rule than measuring expected utility loss relative
to always using RHC.

\subsection{Decision trees with several clinical variables}
\label{sec:decision_trees}

Next we move to decision rules with several clinical variables.
Recall that Theorem~\ref{thm:common_form} shows how to cast the
minimax problem as a weighted policy learning problem; so we can find
policies from data by solving Eqn~\eqref{eq:opt_pol_emp} using
off-the-shelf policy optimization solvers.  Here, we focus on learning
depth-3 decision trees, using the \texttt{policytree} package
\citep{Sverdrup2020}.

Because finding the optimal decision tree scales super-linearly with
the number of covariates, we first select variables from the set of
clinical covariates \footnote{ Of the 66 covariates, 56 are clinical
  variables while the remaining are socioeconomic variables important
  to controlling for confounding. See \citet{Hirano2001}, Table 1, for
  a full list of covariates.\label{fn:variables} } by fitting a CATE
model given the clinical covariates using the DR-learner with random
forest regression.  We then measure variable importance as the
proportion of times a covariate is split on in the forest, weighted by
the node depth, using the \texttt{grf} package, and select the top 10
most important covariates.  See Appendix Figure~\ref{fig:importance}
for the variable importance measures for all clinical covariates.
As before,
we consider estimating minimax loss policies relative to always or never
using RHC as well as the minimax regret policy relative to the oracle, as the utility asymmetry changes with $u_g = 1$ and
$u_l \in [0.5, 2]$.
We again use plug-in estimates for the nuisance
classifiers with the 10 selected covariates.


\begin{figure}[t!]
  \centering
        \begin{subfigure}[t]{0.45\textwidth}  
    {\centering \includegraphics[width=\maxwidth]{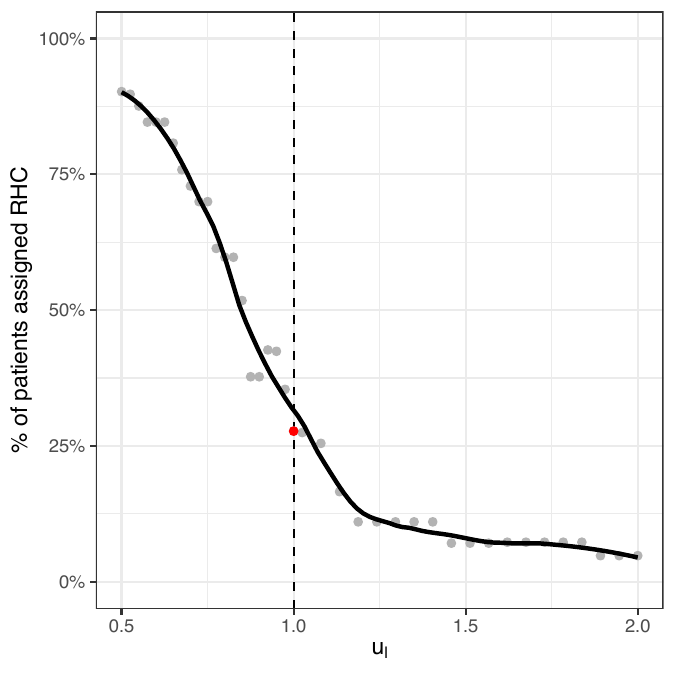} 
    }
    \vspace{-.2in}
    \caption{Relative to always or never using RHC}
    \label{fig:pct_treated_soc}
  \end{subfigure}
    \begin{subfigure}[t]{0.45\textwidth}  
  {\centering \includegraphics[width=\textwidth]{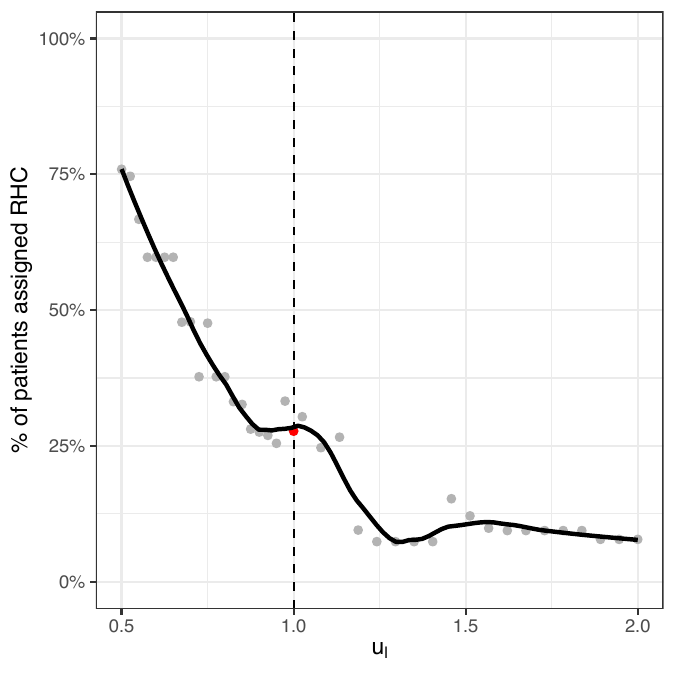} 
  }
  \vspace{-.2in}
    \caption{Relative to oracle policy}
    \label{fig:pct_treated_oracle}
    \end{subfigure}\quad
    \caption{Percent of patients assigned RHC under minimax decision rules
             relative to (a) always or never using RHC and (b) the oracle policy,
             with $u_l \in [0.5, 2]$ and $u_g = 1$. The line is a smoothed
             fit. In both panels, the 
             symmetric policy is highlighted in red.}
  \label{fig:pct_treated}
\end{figure}

Figure~\ref{fig:pct_treated} shows the percent of patients assigned
RHC under the different decision rules.  As we move away from the
symmetric case towards prioritizing using RHC for patients that will
benefit from it, the minimax loss policies relative to always using
RHC and to the oracle assign more patients to RHC.  In the other
direction, as we increase $u_l$ relative to $u_g$ and so seek to
prevent harming patients, the minimax loss policies relative to never
using RHC and the oracle assign fewer patients RHC.  However, the
minimax policy relative to the oracle is less extreme, consistent with
the two-covariate case in Figure~\ref{fig:two_dim} and our discussion
in Section~\ref{sec:regret_oracle}.

\begin{figure}[t!]
  \centering \includegraphics[width=0.65\maxwidth]{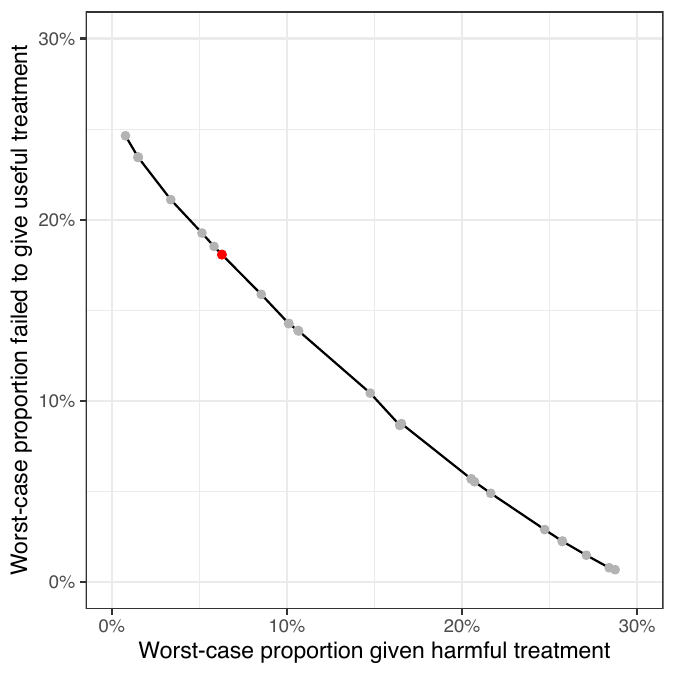}
  \caption{Upper bounds on the proportion of patients that do not
    receive a useful treatment --- the false negative rate --- and the
    proportion of patients given a harmful treatment --- the false
    positive rate --- for minimax depth-3 tree policies relative to
    always or never using RHC, as $u_l$ varies in $[0.5, 2]$ and
    $u_g = 1$. The red point is the
    symmetric depth-3 tree policy.}
\label{fig:pct_harmful_v_helpful}
\end{figure}

We can measure the impact of these policies in terms of
directly interpretable patient outcomes rather than the expected
utility loss,
by inspecting (a) the worst-case proportion of patients that are given a harmful
treatment; (b) the worst-case proportion of patients that are failed to be
given a useful treatment; and (c) the overall expected mortality.
Following the argument in Section~\ref{sec:pop_asymm}, we can find
upper bounds on the first two proportions (presented in Appendix~\ref{sec:more_results}),
and use the DR scores
$\widehat{\Gamma}_w^\text{dr}$ and the plug-in classifier $\hat{\delta}_+$
to get plug-in estimates of them.

Figure~\ref{fig:pct_harmful_v_helpful} shows the worst-case proportion
of patients for whom $\pi$ fails to give a useful treatment ($y$-axis)
versus the worst-case proportion for whom $\pi$ gives a
harmful treatment ($x$-axis)
as we vary $u_l \in [0.5, 2]$.  We observe the
trade-off between these two types of errors, in the worst-case.  If
$\pi$ treats almost no one, in the worst-case almost no patients receive a
harmful treatment, but $\pi$ could be failing to help up to 25\% of
patients.  Conversely, if $\pi$ treats almost everyone, then $\pi$
necessarily treats almost everyone for whom it is helpful, but could be
providing a harmful treatment for up to 30\% of patients. 
Figure~\ref{fig:pct_harmful_v_helpful} shows the intermediate points between
these two extreme scenarios, corresponding to a Pareto frontier between the
errors in the worst case.

We can also see the impact on overall mortality along the
frontier.  Appendix Figure~\ref{fig:mortality_v_pct_harmful} plots the
estimated expected mortality under each policy
against the worst-case proportion of patients  for
whom $\pi$ fails to give a useful treatment or gives a harmful
treatment as we vary $u_l \in [0.5, 2]$.  At one extreme, the policy
that uses RHC for a large number of patients has the highest expected
mortality, because on average RHC is harmful, but the upper bound
guarantees that it fails to give a useful treatment in almost no
cases.  The symmetric policy is at the other end of this tradeoff,
with a lower expected mortality but potentially failing to use RHC
when it is useful for over 18\% of patients.  On the other extreme,
the policy that rarely uses RHC has a lower expected mortality than
always using RHC, but a higher one than the symmetric policy. While
the symmetric policy reduces mortality, up to 6\% of patients will
receive RHC even though it is harmful for them.



\section{Discussion}
\label{sec:discussion}

In this paper, we developed a general policy learning framework that
allows for asymmetric counterfactual  utilities, reflecting common ethical principles
including the Hippocratic oath.
The asymmetry of utility functions leads to the unidentifiability of
expected utilities.  We addressed this problem by employing a partial
identification approach and then finding a policy that minimizes the
maximum expected utility loss relative to a particular policy.
We illustrated this framework by reanalyzing the study of Right Heart
Catheterization, finding that the minimax optimal policy varies
substantially with the asymmetry in the utility and choice of
reference policy.

There are several avenues for future research.  First, the
optimization problem in Eqn~\eqref{eq:minimax} is a form of
\emph{distributionally robust optimization} \citep[see][for a
review]{Bertsimas2011}.  Distributionally robust procedures have been
used for risk minimization and policy learning, often by assuming that
the true, unknown underlying distribution is close to some known
reference distribution.  \citep[see e.g.,][]{Duchi2021_dro,
  Kallus2021,Bertsimas2023_dr_causal}.  In contrast,
Eqn~\eqref{eq:minimax} considers all potential joint distributions
between the potential outcomes---i.e., all potential principal
scores---that agree with the point-identifiable marginal
distributions, without making additional assumptions.  A direction for
future work is to explicitly encode distributional assumptions. For
example, we could treat the case where the potential outcomes are
independent as a reference distribution, and assume that the true
joint distribution is close to it.

Second, we have
restricted our attention to binary outcomes and binary
treatments. However, many decision problems involve multiple potential
actions and categorical or continuous outcomes. This leads to more
principal strata and potential asymmetries in the utility
function.
We briefly discuss extensions to, and difficulties with, the continuous outcome
case in Appendix~\ref{sec:continuous}, and leave a more thorough analysis to
future work.
Third, we may consider decision problems with utility functions that
also depend on other post-treatment variables or mediators, leading to
a different principal stratification structure.  Finally, we can
consider further constraints that encode notions of fairness, such as
the concept of principal fairness that is based on on the principal
strata \citep{imai:jian:20}.

\newpage

{\bfseries\Large
Supplement to ``Policy Learning with Asymmetric Counterfactual Utilities''}

\noindent%
\noindent%

\renewcommand\thefigure{\thesection.\arabic{figure}}    
\setcounter{figure}{0}  
\renewcommand\thetable{\thesection.\arabic{table}}    
\setcounter{table}{0}  
\renewcommand\theassumption{\thesection.\arabic{assumption}}    
\setcounter{assumption}{0}  

\numberwithin{equation}{section}

\onehalfspacing

\appendix

\section{Constrained optimization formulation}
\label{sec:constrained}
While we have arrived at the objective defined in
Eqn~\eqref{eq:regret} through a utility-based framework, we can
also characterize this decision problem in the following constrained
form,
\begin{equation}
  \label{eq:constrained_form}
  \begin{aligned}
    \min_\pi \; & \Pr(Y(\pi(X)) < Y(1))\\
    \text{subject to } & \Pr(Y(\pi(X)) < Y(0)) \leq \delta,\\
    & \E[\pi(X)] \leq B,
  \end{aligned}
\end{equation}
where $ \Pr(Y(\pi(X)) < Y(0)) =  \Pr(Y(1) = 0, Y(0) = 1, \pi(X) = 1)$ and
$ \Pr(Y(\pi(X)) < Y(1)) =  \Pr(Y(1) = 1, Y(0) = 0, \pi(X) = 0)$ represent
the probabilities that policy $\pi$ gives a harmful treatment or fails
to give a useful treatment for a randomly selected member of the
population, respectively.

In this formulation, the goal is to find a policy $\pi$ that minimizes
the expected proportion of \emph{false negatives} --- failing to give a useful
treatment --- subject to a constraint on the expected proportion of
\emph{false positives} --- providing a harmful treatment --- and the
treatment \emph{budget} --- the proportion of units treated.
Thus, the decision problem given in Eqn~\eqref{eq:constrained_form} allows
the policy maker to explicitly state their preferences via the
constraint on the number of false positives and the budget,
rather than implicitly through the
utility function in $R_{e_{01}}(\pi, \varpi)$.
It is also possible to swap the constraints and the objective to minimize the
proportion of false positives subject to a constraint on the proportion of
false negatives.
We can also interpret Eqn~\eqref{eq:constrained_form} through the
lens of multiple testing, for each unit $i$ we have a null hypothesis
$H_{0i}: Y_i(1) < Y_i(0)$, i.e. that unit $i$ is harmed by treatment. We
can view the policy $\pi(X_i)$ as determining whether to reject $H_{0i}$ or not.
Then, the constraint on the proportion of false positives in
Eqn~\eqref{eq:constrained_form} is a scaling of the false detection rate,
where the budget constraint limits the number of rejections, and the objective
is a measure of the average power under the alternative
$H_{1i}: Y_i(1) > Y_i(0)$.

However, note that $\Pr(Y(\pi(X)) < Y(0)) = \E[\pi(X) e_{01}(X)]$ and
$\Pr(Y(\pi(X)) < Y(1)) = \E[(1 - \pi(X))(\tau(x) + e_{01}(X))]$.  Thus,
we can view the expected utility loss $R_{e_{01}}(\pi, \varpi)$ for a constant
comparison policy --- either always or never providing treatment ---
as a Lagrangian relaxation of the decision problem defined in
Eqn~\eqref{eq:constrained_form}, where some choice of the
false-positive constraint $\delta$ and budget $B$ will correspond to a
particular value of the utility ratio $\frac{u_g - u_l}{u_g}$ and cost
ratio $\frac{c}{u_g}$.  This is in contrast to the regret relative to
the oracle policy that maximizes the true value, which involves unidentifiable terms in the
relative weights on $\tau(x)$ and $e_{01}(x)$,  so it cannot be written
as a Lagrangian relaxation of Eqn~\eqref{eq:constrained_form}.

\section{Connection to maximin policies}
\label{sec:maximin}


Under the maximin approach, we find a policy $\pi$ that maximizes the worst-case expected utility.
In this appendix we connect the minimax loss policies relative to never
and always treating to maximin policies under particular choices of
the utility.
To do so, we need to specify the utilities under no treatment, $u(0; y_1, y_0)$. We consider two cases.

First, say that $u(0; y_1,y_0) = 0$ for all principal strata $y_1,y_0$.
In that case, the expected utility is
\begin{align*}
  V(\pi) = \E\left[\pi(X) \left\{u_g\tau(X) + (u_g - u_l)e_{01}(X) - c\right\}\right] = -R_e(\pi, \pi_\mathbbm{O}).
\end{align*}
Therefore the maximin policy is equivalent to the minimax loss policy relative to never treating, $\pi^\ast_\mathbbm{O}$.

Alternatively, say that the utility function under no treatment mirrors that under treatment, i.e.,
\[
u(0;0,0)  = u(0; 1, 1) = 0, \qquad u(0;0,1) = u_l, \qquad u(0; 1, 0) = -u_g.
\]
In this case, the expected utility is 
\[
V(\pi) = \E\left[(\pi(X) - 1) \left\{u_g\tau(X) + (u_g - u_l)e_{01}(X) - c\right\}\right]  - c = -R_e(\pi, \pi_\mathbbm{1}) - c.  
\]
So, the maximin policy is equivalent to the minimax loss policy
relative to always treating, $\pi^\ast_\mathbbm{1}$.

\section{Algorithms for learning minimax loss policies when estimating nuisance functions via  empirical risk minimization}
\label{sec:algos}

\begin{algorithm}[h]
  \begin{algorithmic}[1]
    \Require Policy classes $\Pi$ and $\Delta_+$
    \Ensure Estimated minimax policy $\hat{\pi}$ relative to  $\pi^\mathbbm{1}$  or  $\pi^\mathbbm{O}$
    \State Find $\hat{\delta}_+$ by solving 
    \[ \min_{\delta \in \Delta_+} - \frac{1}{n}\sum_{i=1}^n \delta(X_i)\left\{\widehat{\Gamma}_1(X_i, D_i, Y_i) + \widehat{\Gamma}_0(X_i, D_i, Y_i) - 1\right\}.\]
    \State Compute weighting and cost functions 
    \[\hat{c}^{\varpi}_1(x) =  u_g + \hat{\delta}_+(x)(u_l - u_g), \; \hat{c}^{\varpi}_0(x)  = -u_l - \hat{\delta}_+(x)(u_g - u_l) \text{ and } \hat{c}^{\varpi}(x)= \hat{\delta}_+(x)(u_g - u_l).\]
    \State Find a policy $\hat{\pi} \in \underset{\pi \in \Pi}{\argmin} \; \hat{R}_\text{sup}(\pi, \varpi)$.
  \end{algorithmic}
  \caption{Estimated minimax policy $\hat{\pi}$ relative to the always-treat policy $\pi^\mathbbm{1}$ (when $u_g \geq u_l$) and the never-treat policy $\pi^\mathbbm{O}$ (when $u_g < u_l$)}
  \label{algo:minimax_fixed}
    
\end{algorithm}

\begin{algorithm}[h]
  \begin{algorithmic}[1]
    \Require Policy classes $\Pi$, $\Pi'$, $\Delta_+$, and $\Delta_\tau$
    \Ensure Empirical minimax policy $\hat{\pi}$ relative to the oracle $\pi^o$
    \State Find $\hat{\delta}_+$ by solving 
    \[ \min_{\delta \in \Delta_+} - \frac{1}{n}\sum_{i=1}^n \delta(X_i)\left\{\widehat{\Gamma}_1(X_i, D_i, Y_i) + \widehat{\Gamma}_0(X_i, D_i, Y_i) - 1\right\}.\]
    \If{$u_g \geq u_l$}
    \State Find $\hat{\pi}_{\mathbbm{1}}$  via Algorithm~\ref{algo:minimax_fixed} with policy class $\Pi'$.
    \State Find  $\hat{\pi}_{\mathbbm{O}}$ by solving
    \[ \min_{\pi \in \Pi'} - \frac{1}{n}\sum_{i=1}^n \pi(X_i)\left[u_g\left\{\widehat{\Gamma}_1(X_i, D_i, Y_i) - \widehat{\Gamma}_0(X_i, D_i, Y_i)\right\} - c\right].\]
    \Else
    \State  Find $\hat{\pi}_{\mathbbm{O}}$  via Algorithm~\ref{algo:minimax_fixed} with policy class $\Pi'$.
    \State Find  $\hat{\pi}_{\mathbbm{1}}$ by solving
    \[ \min_{\pi \in \Pi'} - \frac{1}{n}\sum_{i=1}^n \pi(X_i)\left[u_g\left\{\widehat{\Gamma}_1(X_i, D_i, Y_i) - \widehat{\Gamma}_0(X_i, D_i, Y_i)\right\} - c\right].\]
    \EndIf
    \State Find $\hat{\delta}_\tau$ by solving 
    \[ \min_{\delta \in \Delta_\tau} - \frac{1}{n}\sum_{i=1}^n \delta(X_i)\left\{\widehat{\Gamma}_1(X_i, D_i, Y_i) - \widehat{\Gamma}_0(X_i, D_i, Y_i)\right\}.\]
    \State Compute weighting and cost functions $\hat{c}^{\pi^o}_1(x), \hat{c}^{\pi^o}_0(x), \hat{c}^{\pi^o}(x)$ via Theorem~\ref{thm:common_form}.
    \State Find the empirical minimax policy $\hat{\pi} \in \underset{\pi \in \Pi}{\argmin} \hat{R}_\text{sup}(\pi, \pi^o)$.
  \end{algorithmic}
  \caption{Empirical minimax policy $\hat{\pi}$ relative to the oracle policy $\pi^o$}
  \label{algo:minimax_oracle}
\end{algorithm}

\newpage 
\section{Asymmetric utilities based on observed outcomes}
\label{app:asymmetric}

Although it is possible to construct asymmetric utilities without
relying on principal strata \citep{Babii2021}, doing so places
additional restrictions on the structure of utilities.  Consider the
following utility function based on observed outcomes alone,
$u(d,Y(d)) = u_{11}dY_i(d)+u_{10}d\{1-Y_i(d)\}+u_{01}(1-d)Y(d)+
u_{00}(1-d)\{1-Y_i(d)\}$.
This utility function includes the
interaction between the decision and the observed outcome.  Indeed, for a
binary decision and outcome, this represents the most general utility
that could be specified using the observed outcome.

\begin{table}[t!]
  \begin{center}
    \doublespacing

  \begin{tabular}{cccc}
    \hline
     & $Y_i(0)=1 $ &  $Y_i(0)=0$ \\ \hline 
    \multirow{2}{*}{$Y_i(1)=1$} & Harmless & Useful \\
    & $u_{11} - u_{01}$ & $u_{11} - u_{00}$ \\ \cdashline{1-3} 
    \multirow{2}{*}{$Y_i(1)=0$} & Harmful & Useless \\
      & $u_{10} - 2u_{01}$ & $u_{10} - u_{01} - u_{00}$ \\   \hline 
      \end{tabular}
\singlespacing
\caption{Asymmetric utilities gain/loss for treating a unit,
  $u(1, Y_i(1)) - u(0, Y_i(0))$ based on the observed outcomes for each of the
  principal strata.  The
  utility function is given by
  $u(d,Y_i(d)) = u_{11}dY_i(d)+u_{10}d\{1-Y_i(d)\}+u_{01}(1-d)Y_i(d)+ u_{00}(1-d)\{1-Y_i(d)\}$.
  Each cell
  corresponds to the principal stratum defined by the values of the
  two potential outcomes, $Y_i(1)$ and $Y_i(0)$.  Each entry
  represents the utility gain/loss of treatment assignment, relative
  to no treatment, for a unit that belongs to the corresponding
  principal stratum. } \label{tab:utilities-obs}
\end{center}
\end{table}

Table~\ref{tab:utilities-obs} summarizes the utility gain/loss for treating a
unit that belongs to each principal stratum under this setting. With
an interaction term, this utility has different utility gains/losses
in principal strata $(Y(1)=1,Y(0)=0)$ and $(Y(1)=0,Y(0)=1)$, allowing
for the asymmetry in the utilities as required by the Hippocratic
principle.  This utility, however, still places restrictions on its
structure. In particular, it requires that the difference between the
utility gains in principal strata $(Y(1)=1,Y(0)=1)$ and
$(Y(1)=0,Y(0)=1)$ is the same as that between the utility losses in
principal strata $(Y(1)=1,Y(0)=0)$ and $(Y(1)=0,Y(0)=0)$. Therefore,
it might be violated if the difference between harmful and harmless
decisions is much greater than that between useful and useless
decisions.  Thus, a fully general construction of asymmetric utilities
requires the use of principal strata, and defining the utility
function based on both potential outcomes, $u(d;Y(1),Y(0))$, with utility
functions like the one above as a special case.

\section{Simulation study}
\label{sec:sims}

As the results in Section~\ref{sec:theory_results} show,
the misclassification rates of the underlying nuisance classifiers are important
in controlling the excess regret due to estimating the weighting and cost
functions that make up the worst-case regret.
Additionally, although the minimax policies we consider are designed to minimize
the worst-case regret, in some cases it may be possible that the true,
unidentifiable regret may also be small. 
To inspect how the misclassification rates and the true regret behave in finite
samples as the sample size increases,
we now conduct a brief simulation study, where we can know the true values of
the principal scores $e_{y_1y_0}(x)$.

We first generate $n$ 1-dimensional i.i.d. covariates 
$X_i \sim N(0, 2)$.  We then construct log-linear principal scores as
\[
  e_{y_1y_0}(x) = \frac{\exp\left(\alpha_{y_1y_0} + x \beta_{y_1y_0}\right)}{\sum_{y_1'=0}^1\sum_{y_0'=0}^1 \exp\left(\alpha_{y_1'y_0'} + x \beta_{y_1'y_0'}\right)},
\]
where $(\alpha_{00}, \alpha_{10}, \alpha_{01}, \alpha_{11}) = (.2, .15, 0, 0)$,
$\beta_{y_1y_0} \sim N(0, 40)$ for $(y_1,y_0) \in \{(0,0), (1,0), (0,1)\}$,
and $\beta_{11} = 0$. We then jointly sample potential outcomes
$\{Y_{i}(1), Y_i(0)\}$ according to the principal scores  at covariate value $X_i$.
In this simulation study, we consider a randomized control trial with 
binary treatment $D_i$ sampled independently as Bernoulli random variables with
probability one half.

\begin{figure}[t!]
        \begin{subfigure}[t]{0.45\textwidth}  
    {\centering \includegraphics[width=\maxwidth]{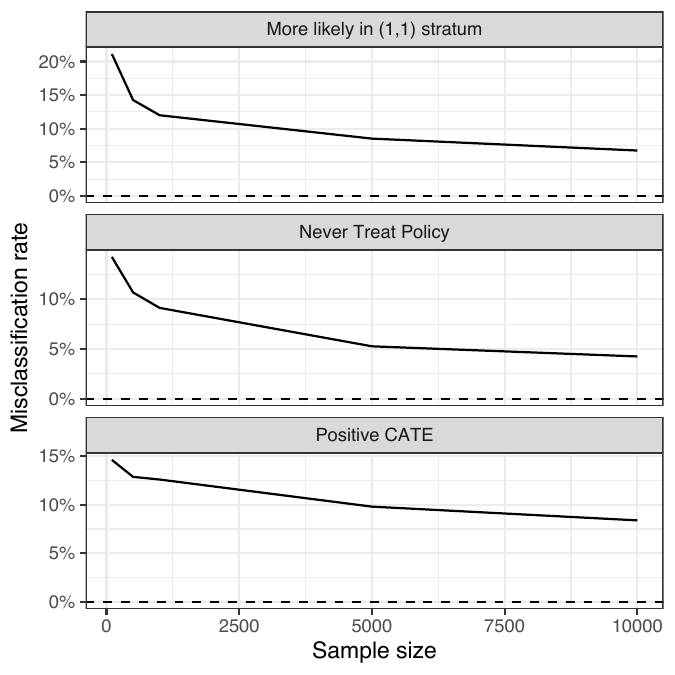}
    }
    \vspace{-.2in}
    \caption{Misclassification rate for nuisance classifiers}
    \label{fig:sim_mis_class}
  \end{subfigure}
    \begin{subfigure}[t]{0.45\textwidth}  
  {\centering \includegraphics[width=\textwidth]{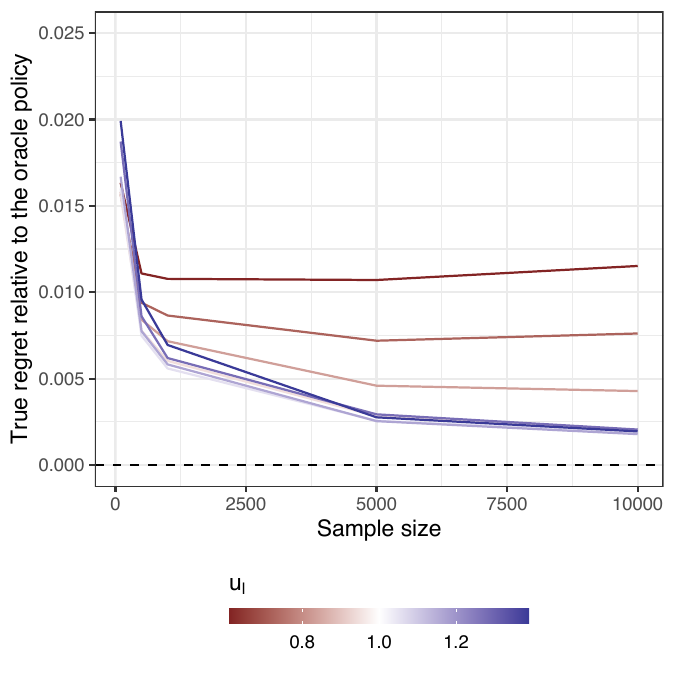}
  }
  \vspace{-.2in}
    \caption{Regret of minimax policy relative to oracle}
    \label{fig:true_regret}
    \end{subfigure}\quad
    \caption{Performance of nuisance classifiers and the minimax optimal
             policy relative to the oracle across simulation runs.
             Panel (a) shows the
             misclassification rate for the nuisance classifiers
             $\hat{\delta}_+$ (``More likely in (1,1) stratum'')
             and $\hat{\delta}_\tau$ (``Positive CATE''), as well as the
             minimax policy relative to the never-treat policy for $u_l = 0.833$.
             Panel (b) shows the true regret of the minimax optimal policy
             relative to the oracle, in the sample, for $u_g = 1$ and as $u_l$
             varies between 0.5 and 1.5.}
  \label{fig:sim_results}
\end{figure}

For each value of sample size $n \in \{100, 500, 1000, 5000, 10000\}$,
we draw 1,000 samples according to the above data generating process.
In each simulation run, we find the minimax optimal policy with
respect to the oracle following Algorithm~\ref{algo:minimax_oracle}
with zero cost $c = 0$, $u_g = 1$, and $u_l$ varying between 0.6 and
1.4, where the value of $u_l$ changes within each simulation run.

We use the IPW scoring function and restrict all policy classes to be
the set of linear thresholds, solving the optimization problem exactly
by direct search.  Figure~\ref{fig:sim_mis_class} shows the average
misclassification rate for the nuisance classifiers $\hat{\delta}_+$
and $\hat{\delta}_\tau$, as well as the misclassification rate for the
the minimax policy relative to always treating for $u_l = 0.833$.
As we expect, we see that these misclassification
rates decrease as the sample size increases.

Figure~\ref{fig:true_regret} shows the true regret of the minimax
policy relative to the oracle as $u_l$ varies. Since the oracle is the
best possible policy, this regret is always positive. The regret does
decrease along with the sample size, reflecting both the decrease in
the nuisance misclassification rate and the decrease in the worst-case
excess regret when the nuisance classifiers are known.  Notice,
however, that the regret does stop decreasing after a certain point,
flattening out at a different level depending on the asymmetry in the
utility function.  In highly asymmetric settings where $u_\ell$ is
small the regret is essentially flat.  This is due to the
fundamental identifiability problem, and even with infinite data we
cannot guarantee that the true regret will be zero.  In contrast, in
the symmetric setting the regret continues to decrease as the sample
size increases.

\section{Implementation details for application to RHC}
\label{sec:implement_details}

\subsection{\edit{Details on cross-fitting procedure}}
\label{sec:crossfit}

In the empirical application to Right Heart Catheterization in
Section~\ref{sec:app}, we use a three-fold cross fitting procedure to
estimate the nuisance functions. We then use the plug-in method to
estimate the nuisance classifiers. Below we present this procedure step-by-step

\begin{enumerate}
  \item Randomly split the data into three folds.
  \item For each fold $k = 1,2,3$, estimate the outcome model
  $\hat{m}^{-k}(\cdot, \cdot)$  and $\hat{d}^{-k}(\cdot)$ on the two other folds via gradient boosted decision stumps.
  \item For each unit $i$, denote $k[i]$ as the fold that it belongs to, then obtain estimates of the outcome model $\hat{m}^{-k[i]}(w, X_i)$, the propensity score $\hat{d}_w^{-k[i]}(X_i)$, and the IP weight $\hat{\gamma}_w^{-k[i]}(D_i, X_i)$.
  \item Use these to construct cross-fit estimates of the DR scoring rule:
  \[
    \widehat{\Gamma}_w^{-k[i]}(X_i, D_i, Y_i)  = \hat{m}^{-k[i]}(w, X_i) +  \{Y_i - \hat{m}^{-k[i]}(w, X_i)\}\hat{\gamma}_w^{-k[i]}(X_i, D_i),
  \]
  and cross-fit plug-in estimates of the classifiers
  \begin{align*}
    \hat{\delta}_{+}^{-k[i]}(X_i) & = \bbone\{\hat{m}^{-k[i]}(1, X_i) + \hat{m}^{-k[i]}(0, X_i) \geq 1\},\\
    \hat{\delta}^{-k[i]}_{\tau}(X_i) & = \bbone\{\hat{m}^{-k[i]}(1, X_i) - \hat{m}^{-k[i]}(0, X_i) \geq 0\},\\
    \hat{\pi}^{-k[i]}_{\mathbbm{O}}(X_i) & = \left\{
        \begin{array}{l r}
          \mathbbm{1}\left\{\hat{m}^{-k[i]}(1, X_i) - \hat{m}^{-k[i]}(0, X_i)  \geq \frac{c}{u_g}\right\} , & u_g \geq u_l,\\
          \bbone\left\{\hat{m}^{-k[i]}(1, X_i) \geq \frac{u_l}{u_g} \hat{m}^{-k[i]}(0, X_i) + \frac{c}{u_g}\right\}, & u_g < u_l \text{ and } \hat{\delta}^{-k[i]}_{+}(X_i) = 0,\\
          \bbone\left\{\hat{m}^{-k[i]}(1, X_i) \geq \frac{u_g}{u_l} \hat{m}^{-k[i]}(0, X_i) + \frac{u_l - u_g + c}{u_l}\right\}, &   u_g < u_l \text{ and } \hat{\delta}^{-k[i]}_{+}(X_i) = 1,
        \end{array}\right. \\
    \hat{\pi}^{-k[i]}_{\mathbbm{1}}(X_i) & = \left\{
      \begin{array}{l r}
        \mathbbm{1}\left\{\hat{m}^{-k[i]}(1, X_i) - \hat{m}^{-k[i]}(0, X_i)  \geq \frac{c}{u_g}\right\} , & u_g < u_l,\\
        \bbone\left\{\hat{m}^{-k[i]}(1, X_i) \geq \frac{u_l}{u_g} \hat{m}^{-k[i]}(0, X_i) + \frac{c}{u_g}\right\}, & u_g \geq u_l \text{ and } \hat{\delta}^{-k[i]}_{+}(X_i) = 0,\\
        \bbone\left\{\hat{m}^{-k[i]}(1, X_i) \geq \frac{u_g}{u_l} \hat{m}^{-k[i]}(0, X_i) + \frac{u_l - u_g + c}{u_l}\right\}, &   u_g \geq u_l \text{ and } \hat{\delta}^{-k[i]}_{+}(X_i) = 1.
      \end{array}\right. 
  \end{align*}
  Then plug in the cross-fit classifiers into the formulas in Appendix~\ref{sec:more_results} to create cross-fit estimates of $\hat{c}^{-k[i] \varpi}(X_i)$.
  \item Solve Eqn~\eqref{eq:opt_pol_emp} with the cross-fit estimates:
  \begin{align}
    \label{eq:opt_pol_emp}
  & \hat{\pi} \in \underset{\pi \in \Pi}{\argmin}
     -\frac{1}{n}\sum_{i=1}^n\pi(X_i)\left\{\hat{c}^{-k[i] \varpi}_1(X_i)\widehat{\Gamma}^{-k[i]}_1(X_i,
     D_i, Y_i) + \hat{c}^{ -k[i] \varpi}_0(X_i)
     \widehat{\Gamma}^{-k[i] }_0(X_i, D_i, Y_i) +
     \hat{c}^{-k[i] \varpi}(X_i)\right\}.\nonumber
\end{align}
\end{enumerate}

\subsection{Minimax loss policies using a subset of covariates}
\label{sec:fewer_covs}
It is often that case that we wish to construct minimax loss decision 
rules that only use a subset of the covariates $\mathcal{V} \subset \mathcal{X}$.
To consider this case, define $m_\mathcal{V}(w,v) \equiv \E[Y(w) \mid V = v]$ to be the expected potential outcome conditioned on the subset of
covariates $v$. Applying Theorem~\ref{thm:common_form} to this setting, 
we get that we can write the worst-case expected utility loss of $\pi$
relative to $\varpi$ as
\[
  R_{\sup}(\pi, \varpi) = C -\E\left[\pi(X)\left\{c_{1\mathcal{V}}^\varpi(V)m_\mathcal{V}(1,V) + c_{0\mathcal{V}}^\varpi(V) m_\mathcal{V}(0, V) + c_\mathcal{V}^\varpi(V)\right\}\right],
\]
where the weighting and cost functions $c_{1\mathcal{V}}^\varpi(\cdot), c_{0\mathcal{V}}^\varpi(\cdot), c_{\mathcal{V}}^\varpi(\cdot)$ depend on the nuisance classifiers given only the subset of the covariates $V$, i.e.
\begin{align*}
  \delta_{+\mathcal{V}}(v) & = \bbone\{m_\mathcal{V}(1,v) + m_\mathcal{V}(0, v) \geq 1\},\\
  \delta_{\tau\mathcal{V}}(v) & = \bbone\{m_\mathcal{V}(1,v) - m_\mathcal{V}(0, v) \geq 0\},\\
  \pi^\ast_{\mathbbm{O}\mathcal{V}}(v) & = \left\{
      \begin{array}{l r}
        \mathbbm{1}\left\{m_\mathcal{V}(1,v) - m_\mathcal{V}(0,v)  \geq \frac{c}{u_g}\right\} , & u_g \geq u_l,\\
        \bbone\left\{m_\mathcal{V}(1,v) \geq \frac{u_l}{u_g} m_\mathcal{V}(0, v) + \frac{c}{u_g}\right\}, & u_g < u_l \text{ and } \delta_{+\mathcal{V}}(v) = 0,\\
        \bbone\left\{m_\mathcal{V}(1,v) \geq \frac{u_g}{u_l} m_\mathcal{V}(0, v) + \frac{u_l - u_g + c}{u_l}\right\}, &   u_g < u_l \text{ and } \delta_{+\mathcal{V}}(v) = 1,
      \end{array}\right. \\
  \pi^\ast_{\mathbbm{1}\mathcal{V}}(v) & = \left\{
    \begin{array}{l r}
      \mathbbm{1}\left\{m_\mathcal{V}(1,v) - m_\mathcal{V}(0,v)  \geq \frac{c}{u_g}\right\} , & u_g < u_l,\\
      \bbone\left\{m_\mathcal{V}(1,v) \geq \frac{u_l}{u_g} m_\mathcal{V}(0, v) + \frac{c}{u_g}\right\}, & u_g \geq u_l \text{ and } \delta_{+\mathcal{V}}(v) = 0,\\
      \bbone\left\{m_\mathcal{V}(1,v) \geq \frac{u_g}{u_l} m_\mathcal{V}(0, v) + \frac{u_l - u_g + c}{u_l}\right\}, &   u_g \geq u_l \text{ and } \delta_{+\mathcal{V}}(v) = 1.
    \end{array}\right. .
\end{align*}
However, note that in order to use observable data, we must account for confounding, since in general $m_\mathcal{V}(w,v) \neq \E(Y \mid V = v, W = w)$ when $\mathcal{V}$ is a subset of $\mathcal{X}$. We can however,
still use the IPW or DR scoring functions since $m_\mathcal{V}(w,v) = \E[\Gamma_w(X, D, Y) \mid V = v]$. So we can write the worst-case
expected utility loss in terms of the scoring functions---where we condition on $X$---and the nuisance classifiers only conditioning on the
subset of covariates $V$:
\[
  R_{\sup}(\pi, \varpi) = C -\E\left[\pi(V)\left\{c_{1\mathcal{V}}^\varpi(V)\Gamma_1(X, D, Y)+ c_{0\mathcal{V}}^\varpi(V) \Gamma_0(X, D, Y) + c\mathcal{V}^\varpi(V)\right\}\right],
\]

Constructing plug-in estimates of the nuisance classifiers involves estimating $m_\mathcal{V}(w,v) = \E[\Gamma_w(X, D, Y) \mid V = v]$, which we can do by regressing the estimated DR scores on the
subset of the covariates $V$, a variant of the DR-learner \citep{Kennedy2022_drlearner}.

Overall, this leads to the following steps:
\begin{enumerate}
  \item Estimate the DR score $\widehat{\Gamma}_w(x, d, y)$  using \emph{all covariates} $X$ to account for confounding.
  \item Estimate the expected potential outcomes given the subset of covariates $V$, $\hat{m}_\mathcal{V}(w,v)$ using the DR-learner and regressing the estimates $\widehat{\Gamma}_w(X_i, D_i, Y_i)$ on $V$.
  \item Form plug in estimates of the nuisance classifiers, e.g. $\hat{\delta}_\tau(v) = \bbone\{\hat{m}_\mathcal{V}(1,v) - \hat{m}_\mathcal{V}(0,v) \}$ and $\hat{\delta}_+(v) = \bbone\{\hat{m}_\mathcal{V}(1,v) + \hat{m}_\mathcal{V}(0,v) - 1 \geq 0\}$.
  \item Get plug-in estimates of the weighting and cost functions $\hat{c}_{1\mathcal{V}}^\varpi(V_i), \hat{c}_{0\mathcal{V}}^\varpi(V_i), \hat{c}_{\mathcal{V}}^\varpi(V_i)$, using the estimates of the nuisance classifiers.
  \item Find the policy $\hat{\pi}:\mathcal{V} \to \{0,1\}$ by solving
  \[
    \min_{\pi \in \Pi} -\frac{1}{n}\sum_{i=1}^n \pi(V)\left\{\hat{c}_{1\mathcal{V}}^\varpi(V)\widehat{\Gamma}_1(X, D, Y)+ \hat{c}_{0\mathcal{V}}^\varpi(V) \widehat{\Gamma}_0(X, D, Y) + \hat{c}\mathcal{V}^\varpi(V)\right\}.
  \]
\end{enumerate}

\edit{Finally, note that as in Section~\ref{sec:crossfit} above, we can use cross-fit estimates here, where for each fold $k$, both $\widehat{\Gamma}^{-k}_w$ and $\hat{m}^{-k}_\mathcal{V}$ are fit on data not in fold $k$. In principle we could do a multi-stage cross-fitting procedure, where for each fold $k$, we further break up the fold into sub-folds and cross-fit $\hat{m}^{-k}_\mathcal{V}$ within the fold $k$. We opt to use a simpler cross-fitting procedure here, noting that it may impact the quality of the DR-learner estimate $\hat{m}^{-k}_\mathcal{V}$.}

\setcounter{figure}{0}  
\section{Additional figures}

\begin{figure}[H]
  \centering \includegraphics[width=0.7\maxwidth]{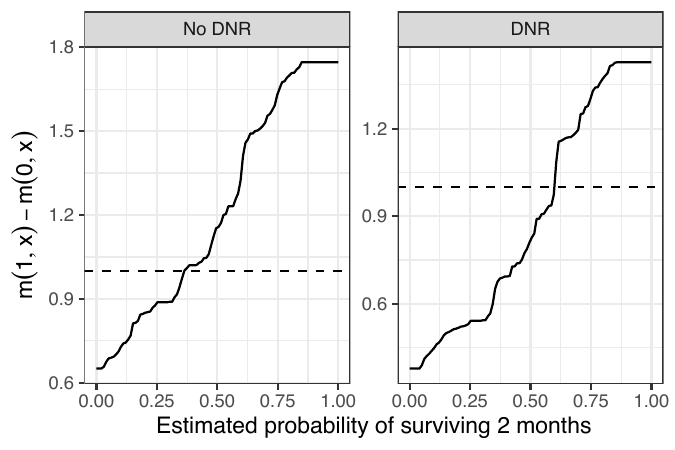}
\caption{Plug-in estimate of the decision rule $\hat{\delta}_+(v)$ to classify whether $m_\mathcal{V}(1,v) + m_\mathcal{V}(0,x) \geq 1$ using the estimated probability of survival and DNR status.}
\label{fig:two_dim_pos_class}
\end{figure}

\begin{figure}[h]
  \centering \includegraphics[width=0.75\maxwidth]{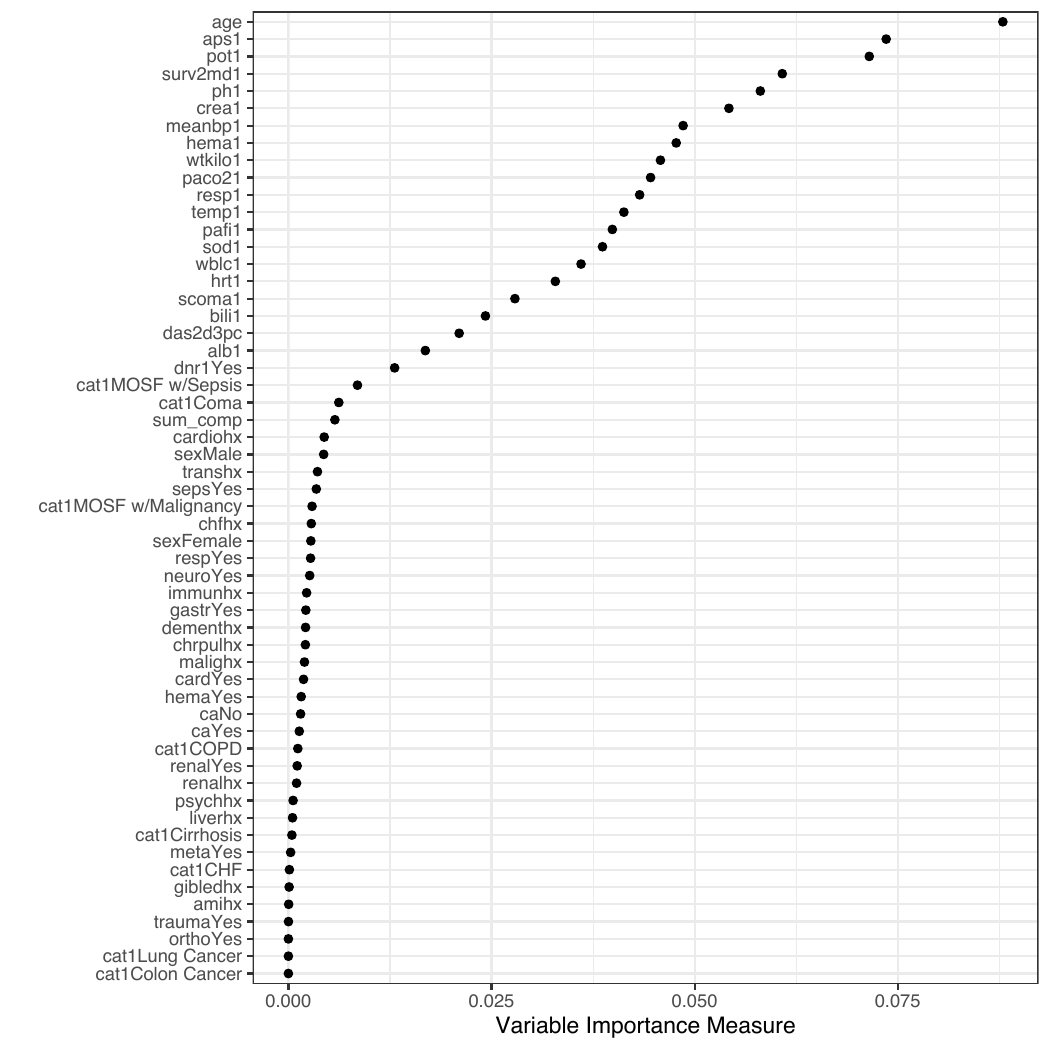}
\caption{Variable importance for estimated CATE.}
\label{fig:importance}
\end{figure}

\begin{figure}[h]
  \centering \includegraphics[width=0.75\maxwidth]{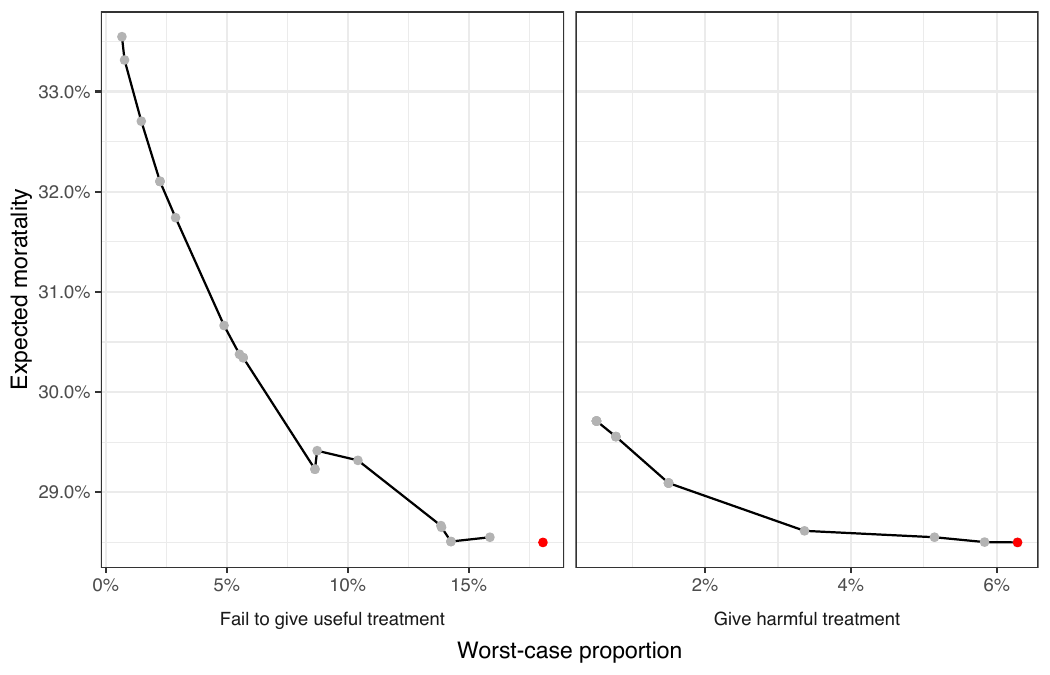}
\caption{Estimated expected mortality versus the worst case proportion of cases
 where the policy fails to give a useful treatment (left panel) and where the
policy gives a harmful treatment (right panel), for linear minimax policies relative to never using RHC, as $u_l$ varies in $[0.65, 1.2]$ and $u_g = 1$.
The red point is the symmetric linear policy.}
\label{fig:mortality_v_pct_harmful}
\end{figure}

\clearpage

\section{Additional results}
\label{sec:more_results}

{
  \paragraph{Full statement of Theorem~\ref{thm:common_form}}
\spacingset{1}

  Let $\pi: \mathcal{X} \to \{0, 1\}$ be a deterministic policy.  For
  comparison policy $\varpi \in \{\pi^\mathbbm{O}, \pi^\mathbbm{1}\}$,
  the worst-case expected utility loss of $\pi$ relative to $\varpi$
  is
  \begin{equation}
    \label{eq:common_worstcase_regret}
    \begin{aligned}
      R_{\sup}(\pi, \varpi) & = C -\E\left[\pi(X)\left\{c_1^\varpi(X)m(1,X) + c_0^\varpi(X) m(0, X) + c^\varpi(X)\right\}\right]\\
      & =  C -\E\left[\pi(X)\left\{c^\varpi_1(X)\Gamma_1(X, D, Y) + c_0^\varpi(X) \Gamma_0(X, D, Y) + c^\varpi(X)\right\}\right],  
    \end{aligned}
  \end{equation}  
  where $C$ is a constant that does not depend on $\pi$. For $u_g \geq u_l$,
  \vspace{-30pt}
  {\singlespacing
  \[
    \begin{aligned}
      c_1^{\pi^\mathbbm{O}}(x) &  = u_l + (u_g - u_l)\delta_\tau(x) &  \qquad
      & c_0^{\pi^\mathbbm{O}}(x) &  = -u_l - (u_g - u_l)\delta_\tau(x) & \qquad
      & c^{\pi^\mathbbm{O}}(x) & = -c,\\
      c^{\pi^\mathbbm{1}}_1(x) & = u_g + \delta_+(x)(u_l - u_g)&  \qquad
      & c^{\pi^\mathbbm{1}}_0(x) & = -u_l - \delta_+(x)(u_g - u_l) &  \qquad
      & c^{\pi^\mathbbm{1}}(x) &= \delta_+(x)(u_g - u_l) - c,
    \end{aligned}
    \]}
    and for $u_g < u_l$,
    {\singlespacing
    \vspace{-30pt}
    \[
      \begin{aligned}
        c^{\pi^\mathbbm{O}}_1(x) & = u_g + \delta_+(x)(u_l - u_g)&  \qquad
        & c^{\pi^\mathbbm{O}}_0(x) & = -u_l - \delta_+(x)(u_g - u_l) &  \qquad
        & c^{\pi^\mathbbm{O}}(x) &= \delta_+(x)(u_g - u_l) - c,\\
        c_1^{\pi^\mathbbm{1}}(x) &  = u_l + (u_g - u_l)\delta_\tau(x) &  \qquad
        & c_0^{\pi^\mathbbm{1}}(x) &  = -u_l - (u_g - u_l)\delta_\tau(x) & \qquad
        & c^{\pi^\mathbbm{1}}(x) & = -c.
      \end{aligned}
      \]
    }Define $\pi^\ast_{\mathbbm{O}} \equiv \argmin_\pi R_{\sup}(\pi,
    \pi^\mathbbm{O})$ and
    $\pi^\ast_{\mathbbm{1}} \equiv \argmin_\pi R_{\sup}(\pi,
    \pi^\mathbbm{1})$ as the minimax expected utility loss solutions
    relative to the never-treat policy and always-treat policy, respectively.
    The worst-case
    regret relative to the oracle policy $\pi^o$ is of the form in
    Eqn~\eqref{eq:common_worstcase_regret} where for $u_g \geq u_l$,

  \vspace{-30pt}
{
  \singlespacing
  \begin{align*}
     \left(\begin{array}{c}
      c^{\pi^o}_1(x)\\  c^{\pi^o}_0(x)\\ c^{\pi^o}(x)
    \end{array}
    \right) = &
    (1 - \pi^\ast_{\mathbbm{1}}(x))\left(\begin{array}{c}
      u_l + (u_g - u_l)\delta_\tau(x)\\ -u_l - (u_g - u_l)\delta_\tau(x)\\ - c
    \end{array}
    \right) + \pi^\ast_{\mathbbm{O}}(x)
    \left(\begin{array}{c}
      u_g - (u_g - u_l)\delta_+(x)\\ -u_l - (u_g - u_l)\delta_+(x)\\ (u_g - u_l)\delta_+(x) - c
    \end{array}
    \right)\\ & +
    (1-\pi^\ast_{\mathbbm{O}}(x))\pi^\ast_{\mathbbm{1}}(x)
    \left(\begin{array}{c}
      u_l + u_g  + (u_g - u_l)(\delta_\tau(x) - \delta_+(x))\\ -2u_l - (u_g - u_l)(\delta_\tau(x) + \delta_+(x))\\ (u_g - u_l)\delta_+(x) - 2c
    \end{array}
    \right),
  \end{align*}
}
\vspace{-30pt}
  and for $u_g < u_l$,
  {\singlespacing
  \begin{align*}
    \left(\begin{array}{c}
     c^{\pi^o}_1(x)\\  c^{\pi^o}_0(x)\\ c^{\pi^o}(x)
   \end{array}
   \right) = &
   (1 - \pi^\ast_{\mathbbm{1}}(x))\left(\begin{array}{c}
    u_g - (u_g - u_l)\delta_+(x)\\ -u_l - (u_g - u_l)\delta_+(x)\\ (u_g - u_l)\delta_+(x) - c
   \end{array}
   \right) + \pi^\ast_{\mathbbm{O}}(x)
   \left(\begin{array}{c}
    u_l + (u_g - u_l)\delta_\tau(x)\\ -u_l - (u_g - u_l)\delta_\tau(x)\\ - c
   \end{array}
   \right)\\ & +
   (1-\pi^\ast_{\mathbbm{O}}(x))\pi^\ast_{\mathbbm{1}}(x)
   \left(\begin{array}{c}
    u_l + u_g  + (u_g - u_l)(\delta_\tau(x) - \delta_+(x))\\ -2u_l - (u_g - u_l)(\delta_\tau(x) + \delta_+(x))\\ (u_g - u_l)\delta_+(x) - 2c
   \end{array}
   \right).
 \end{align*}}
}

\begin{corollary}[Minimax regret relative to the always-treat policy]
  \label{cor:regret_everybody}
  If $u_g \geq u_l$, the minimax regret solution to
  Equation~\eqref{eq:minimax}, $\pi^\ast_{\mathbbm{1}} \equiv
  \argmin_\pi R_{\sup}(\pi, \pi^\mathbbm{1})$, is 
  \[
    \pi^\ast_{\mathbbm{1}}(x) = \left\{
      \begin{array}{l r}
        \bbone\left\{m(1,x) \geq \frac{u_l}{u_g} m(0, x) + \frac{c}{u_g}\right\}, & \delta_+(x) = 0,\\
        \bbone\left\{m(1,x) \geq \frac{u_g}{u_l} m(0, x) + \frac{u_l - u_g + c}{u_l}\right\}, &   \delta_+(x) = 1.
      \end{array}\right. 
  \]
  Otherwise, if $u_g < u_l$, it is given by the symmetric policy,
  \[
    \pi^\ast_\mathbbm{1}(x) = \mathbbm{1}\left\{\tau(x) \geq \frac{c}{u_g}\right\} = \pi^\symm(x).
  \]
\end{corollary}
\begin{assumption}
  \label{a:margin}
  There exists an $\alpha > 0$ and a constant $C$ such that for any $t \geq 0$,
  \begin{enumerate}[label ={(\alph*)}, ref={\theassumption(\alph*)}]
    \item 
       $\Pr(|m(1,X) + m(0,X) - 1| \leq t) \leq Ct^\alpha.$
    \item \label{a:margin_tau} $\Pr(|m(1,X) - m(0,X)| \leq t) \leq Ct^\alpha.$
    \item \label{a:margin_upper} For $u_g > u_l$ and $c$,  
    \[
      \Pr\left(\left|\{u_g - (u_g - u_l)\delta_+(X)\}m(1,X) - \{u_l + (u_g - u_l)\delta_+(X)\}m(0,X) + (u_g - u_l)\delta_+(X) - c\right| \leq t\right) \leq Ct^\alpha.
    \]
    \item \label{a:margin_lower} For $u_g > u_l$ and $c$, 
    
    \[
      \Pr(| \{u_l + (u_g - u_l)\delta_\tau(X)\}\tau(X) - c | \leq t) \leq Ct^\alpha.
    \]
  \end{enumerate}
\end{assumption}

\begin{theorem}
  \label{thm:excess_regret_ones_policy_plugin}
  Let $u_g \geq u_l$, and define 
  \begin{align*}
    \widehat{L}_b(x) & = \{u_l + \hat{\delta}_\tau(x)(u_g - u_l)\} (\hat{m}(1,x) - \hat{m}(0,x)) - c,\\
    \widehat{U}_b(x) & = \{u_g - (u_g - u_l)\hat{\delta}_+(x)\}\hat{m}(1,x) - \{u_l + (u_g - u_l)\hat{\delta}_+(x)\}\hat{m}(0,x) + (u_g - u_l)\hat{\delta}_+(x) - c,
  \end{align*}
  and let $\hat{\pi}_\mathbbm{O}^\text{plug}(x) = \bbone\{\widehat{L}_b(x) \geq 0\}$ and 
  $\hat{\pi}_\mathbbm{1}^\text{plug}(x) = \bbone\{\widehat{U}_b(x) \geq 0\}$ be the plug-in
  estimates of the minimax optimal policies relative to never or always treating.
  Under Assumptions~\ref{a:margin_tau} and \ref{a:margin_lower}, the excess
  worst case regret for $\hat{\pi}_\mathbbm{O}^\text{plug}$ relative to
  $\pi_\mathbbm{O}^\ast$ is
  \[
    R_\text{sup}(\hat{\pi}^\text{plug}_\mathbbm{O}, \pi^\mathbbm{O}) -  R_\text{sup}(\pi^\ast_\mathbbm{O}, \pi^\mathbbm{O}) \leq 2^{1+\alpha} C 
    U \|m - \hat{m}\|_\infty^{1 + \alpha},
  \]
  where $U$ is a constant depending on the utility values, $\alpha$, and $C$.
  Under Assumptions~\ref{a:margin_posclass} and \ref{a:margin_upper}, the excess
  worst case regret for $\hat{\pi}_\mathbbm{1}^\text{plug}$ relative to
  $\pi_\mathbbm{1}^\ast$ is
  \[
    R_\text{sup}(\hat{\pi}^\text{plug}_\mathbbm{1}, \pi^\mathbbm{1}) -  R_\text{sup}(\pi^\ast_\mathbbm{O}, \pi^\mathbbm{O}) \leq 2^{1+\alpha} C 
    U \|m - \hat{m}\|_\infty^{1 + \alpha},
  \]
  where $U$ is a constant depending on the utility values, $\alpha$, and $C$.
\end{theorem}

\begin{corollary}
  \label{cor:excess_regret_oracle_plugin_all}
    Let $u_g \geq u_l$,
     $\hat{\pi}_o$ be a solution to Equation \eqref{eq:opt_pol_emp} with alternative policy $\varpi = \pi^o$ and with nuisance functions
    $\hat{m}$ and $\hat{d}$  fit on a separate sample and nuisance classifiers $\hat{\delta}_+(x) = \bbone\{m(1,x) + m(0,x) - 1 \geq 0\}, \hat{\delta}_\tau(m(1,x) - m(0,x) \geq 0), \hat{\pi}^\text{plug}_\mathbbm{O}$, and $\hat{\pi}^\text{plug}_\mathbbm{1}$, and let $\pi^\ast_o$ be a solution to Equation \eqref{eq:minimax}, with alternative policy $\varpi = \pi^o$. Under the strict overlap condition in  Assumption~\ref{a:ignore},
    the excess worst-case regret of $\hat{\pi}_o$ relative to $\pi^\ast_o$ satisfies
    \begin{align*}
      R_\text{sup}(\hat{\pi}_0, \pi^o) - R_\text{sup}(\pi^\ast_o, \pi^o) & \leq 2U_1 \times \left(\frac{6 + \eta}{\eta} \times \left(2\mathcal{R}_n(\Pi) +  \frac{t}{\sqrt{n}}\right)+ \|\hat{m} - m\|_2 \|\hat{\gamma} - \gamma\|_2\right)\\
      & \qquad +2^{2 + \alpha}C U_2 \|\hat{m} - m\|_\infty^{1 + \alpha} 
       + (u_g - u_l)\frac{t}{2\sqrt{n}},
    \end{align*}
    with probability at least $1 - 2\exp\left(-\frac{t^2}{2}\right)$, where $U_1$ is a constant depending on the utility values, and $U_2$ is a constant depending on the utility values, $\alpha$, and $C$.
  \end{corollary}

\paragraph{Upper bounds on worst-case proportion of units given a harmful treatment or are failed to be given a useful treatment.}
\mbox{}\\
First, note that
\begin{align*}
  \Pr(Y(\pi(X)) < Y(0)) & = \Pr(\pi(X) = 1, Y(0) = 1, Y(1) = 0) = \E\left[\pi(X)e_{01}(X)\right],\\
  \Pr(Y(\pi(X)) < Y(1)) & = \Pr(\pi(X) = 0, Y(0) = 0, Y(1) = 1) = \E\left[(1 - \pi(X))(\tau(X) + e_{01}(X))\right].
\end{align*}
Plugging in the upper and lower bounds on $e_{01}(X)$ in Section~\ref{sec:pop_asymm}, we get the following upper bounds:
\begin{align*}
  \sup_{e_{01}(x) \in [L(x), U(x)]} \Pr(Y(\pi(X)) < Y(0)) & = \E\left(\pi(X)\left[m(0, X) + \delta_+(X)\left\{1 - m(1,X) - m(0,X)\right\} \right]\right),\\
\sup_{e_{01}(x) \in [L(x), U(x)]} \Pr(Y(\pi(X)) < Y(1)) & = \E\left(\left\{1 - \pi(X)\right\}\left[m(1, X) + \delta_+(X)\left\{1 - m(1,X) - m(0,X)\right\} \right]\right).
\end{align*}

\section{Continuous outcomes}
\label{sec:continuous}

Here we briefly consider extending our framework to the case with a binary decision $D \in \{0,1\}$ but continuous potential outcomes $(Y(0), Y(1)) \in \R^{2}$. We define the utility function $u(d;y_1,y_0)$ as before and write the value function as
\[
  V(\pi) = \E\left[u(0; Y(1), Y(0)) + \pi(X) \times (u(1;Y(1),Y(0)) - u(0; Y(1), Y(0)))\right].
\]
Defining $e_{y_1 y_0}(x)$ as the conditional joint density of the potential
outcomes given $X = x$, the expected utility loss relative to $\varpi$ is
\[
  V(\varpi) - V(\pi) = \E\left[\pi(X) \int_{y_1}\int_{y_0} (u(1; y_1, y_0) - u(0; y_1, y_0)) e_{y_1 y_0}(x)dy_0 dy_1\right].
\]
With continuous outcomes, there are many potential ways to choose
the utility function. One  option is a utility function such that 
$u(1; y_1, y_0) - u(0; y_1, y_0) = y_1 - y_0 - u_\ell \bbone\{y_1 < y_0\}$.
This is analogous to the utility function with binary outcomes, with an explicit
utility gain/loss associated with a harmful ($Y(1) < Y(0)$) or useful
($Y(1) > Y(0)$) treatment. Defining the conditional probability of harm as
$h(x) = \Pr(Y(1) < Y(0) \mid X = x)$, we can write the expected utility loss as
\[
  V(\varpi) - V(\pi) = \E\left[\pi(X) \{\tau(x) - u_\ell h(x)\}\right].
\]

As in the binary case, we can use sharp bounds on the distribution of individual
treatment effects \citep{Fan2010}, $h(x) \in [L_h(x), U_h(x)]$, where
\begin{align*}
  L_h(x) &= \max\{\sup_y \{F_1(y \mid x) - F_0(y \mid x)\}, 0\},\\
  U_h(x) &= 1 + \min\{\inf_y \{F_1(y \mid x) - F_0(y \mid x)\}, 0\},
\end{align*}
where $F_1(\cdot \mid x), F_0(\cdot \mid x)$ are the marginal CDFs conditional
on $X = x$ for the potential outcomes under treatment and control, respectively.
Now we can again define the minimax expected utility loss policy as the policy
that solves
\[
  \min_\pi \; \max_{h(x) \in [L(x), U(x)]} \; \E\left[\pi(X) \{\tau(x) - u_l h(x)\}\right].
\]

While this again leads to a point-identifiable objective, we note two ways in
which this problem is more difficult than with binary outcomes.
First, the upper and lower bounds on the probability of harm involves the
conditional CDFs of $Y(1)$ and $Y(0)$. These can be more difficult to estimate
than the conditional expected outcomes. Second, the bounds involve supremums
and infimums over all $y \in \R$. This may require a more careful analysis and
stronger assumptions in order to ensure that the default plug-in approach that
we suggest for the binary outcome case will lead to reasonable guarantees on
the excess expected utility loss.

\section{Proofs and derivations}
\label{sec:proofs}


\subsection{Main results}

\paragraph{Derivation of the expected utility loss}

First, notice that the expected utility of policy $\pi$ is
\begin{align*}
  V(\pi) & = \E\left[\sum_{y_1 = 0}^1\sum_{y_0 = 0}^1  e_{y_1y_0}(X) u(0; y_1, y_0)\right] + \underbrace{\E\left[\sum_{y_1 = 0}^1\sum_{y_0 = 0}^1 \pi(X) e_{y_1y_0}(X) \left\{u(1; y_1, y_0) - u(0; y_1, y_0)\right\}\right]}_{(\ast)}.
\end{align*}
The second term can be written as
\begin{align*}
  (\ast) & = \E\left[\pi(X) \left\{e_{10}(X)(u_g - c) -e_{01}(X)(u_l + c) -e_{00}(X)c -e_{11}(X)c\right\}\right]\\
  & = \E\left[\pi(X)) \left\{e_{10}(X)u_g -e_{01}(X)u_l - c(e_{10}(X) + e_{01}(X) + e_{00}(X) +e_{11}(X))\right\}\right]\\
  &= \E\left[\pi(X) \left\{(\tau(X) + e_{01}(X))u_g -e_{01}(X)u_l - c\right\}\right]\\
  & = \E\left[\pi(X) \left\{u_g\tau(X) + (u_g - u_l)e_{01}(X) - c\right\}\right],
\end{align*}
where we have used the fact that $\tau(x) = e_{10}(x) - e_{01}(x)$.
So the expected utility loss of policy $\pi$ relative to policy $\varpi$ is
\[
\begin{aligned}
  V(\varpi) - V(\pi) & = \E\left[(\varpi(X) - \pi(X)) \left\{u_g\tau(X) + (u_g - u_l)e_{01}(X) - c\right\}\right].
\end{aligned}  
\]

\begin{proof}[Proof of Theorem~\ref{thm:common_form}]
  Define $b(x) = u_g\tau(x) + (u_g - u_l)e_{01}(X) - c$, and
  \[
    \begin{aligned}
      L_b(x) & = \min_{e(x) \in [L(x), U(x)]} \left\{u_g\tau(x) + (u_g - u_l)e_{01}(X) - c\right\},\\
      U_b(x) & = \max_{e(x) \in [L(x), U(x)]} \left\{u_g\tau(x) + (u_g - u_l)e_{01}(X) - c\right\}.
    \end{aligned}  
  \]
  Note that the worst-case regret relative to the always and never treat policies are
  \[
    \begin{aligned}
      R_\text{sup}(\pi, \pi^\mathbbm{O}) & =-\E[\pi(X) L_b(X)],\\
      R_{\text{sup}}(\pi, \pi^\mathbbm{1}) & = \E[\{1 - \pi(X)\} U_b(X)] = \E[U_b(X)] - \E[\pi(X)U_b(X)].
    \end{aligned}
  \]
  From this, we can find the unconstrained minimax regret policies
  \[
    \begin{aligned}
      \pi^\ast_\mathbbm{O} & = \underset{\pi}{\argmin} -\E[\pi(X) L_b(X)] = &  \mathbbm{1}\{L_b(x) \geq 0\},\\
      \pi^\ast_\mathbbm{1} & = \underset{\pi}{\argmin} -\E[\pi(X) U_b(X)] = &  \mathbbm{1}\{U_b(x) \geq 0\}.
    \end{aligned}
  \]
  Now, the oracle policy is $\pi^o(x) = \mathbbm{1}\{b(x) \geq 0\}$.
  So if $L_b(x) \geq 0 \Leftrightarrow \pi^\ast_\mathbbm{O}(x) = 1$ then $\pi^o(x) = 1$ for all possible values of the 
  principal score $e_{01}(x)$. In this case,
  \[
    \max_{e(x) \in [L(x), U(x)]} \{\pi^o(x) - \pi(x)\}b(x) = \{1 - \pi(x)\}U_b(x).
  \]
  Similarly, if $U_b(x) < 0 \Leftrightarrow \pi^\ast_\mathbbm{1}(x) = 0$ then $\pi^o(x) = 0$, and
  \[
    \max_{e(x) \in [L(x), U(x)]} \{\pi^o(x) - \pi(x)\}b(x) = -\pi(x)L_b(x).
  \]
  Finally, if $L_b(x) < 0$ and $U_b(x) \geq 0$ (so $\pi^\ast_\mathbbm{O}(x) = 0$
  and $\pi^\ast_\mathbbm{1}(x) = 1$), then the oracle policy can be
  either 0 or 1, $\pi^o(x) \in \{0,1\}$. Therefore, 
  \[
    \max_{e(x) \in [L(x), U(x)]} \{\pi^o(x) - \pi(x)\}b(x) =
      \max\{(1 - \pi(x)) U_b(x), -\pi(x) L_b(x)\} = U_b(x) - \pi(x)\{U_b(x) + L_b(x)\}.
  \]
  Putting together the pieces, the worst-case regret relative to the oracle is
  \[
    \begin{aligned}
      R_\text{sup}(\pi, \pi^o) & = \E([\pi^\ast_\mathbbm{O}(X)+ \{1 - \pi^\ast_\mathbbm{O}(X)\} \pi^\ast_\mathbbm{1}(X)]U_b(X))\\
      & \hspace{0.5cm}- \E[\pi(X)\left\{\pi^\ast_\mathbbm{O}(X)U_b(X) + (1 - \pi^\ast_\mathbbm{1}(X))L_b(X) + (1 - \pi^\ast_\mathbbm{O}(X))\pi^\ast_\mathbbm{1}(X)(U_b(X) + L_b(X))\right\}],
    \end{aligned}
  \]
  and the unconstrained minimizer is
  \[
    \pi^\ast_o  = \underset{\pi}{\argmin} R_\text{sup}(\pi, \pi^o) =
    \left\{\begin{array}{c c}
      \pi^\ast_\mathbbm{1}(x), & \pi^\ast_\mathbbm{O}(x) = 1,\\
      \pi^\ast_\mathbbm{O}(x), & \pi^\ast_\mathbbm{1}(x) = 0,\\
      \mathbbm{1}\{U_b(x) \geq - L_b(x)\}, & \pi^\ast_\mathbbm{O}(x) = 0,\pi^\ast_\mathbbm{1}(x) = 1.
    \end{array} 
    \right. 
  \]
  Now notice that $\pi^\ast_\mathbbm{O}(x) = 1 \Leftrightarrow L_b(x) \geq 0
  \Rightarrow U_b(x) \geq 0 \Leftrightarrow \pi^\ast_\mathbbm{1}(x) = 1$, and
  $\pi^\ast_\mathbbm{1}(x) = 0 \Leftrightarrow U_b(x) < 0
  \Rightarrow L_b(x) < 0 \Leftrightarrow \pi^\ast_\mathbbm{O}(x) = 0$, so
  we can simplify this to
  \[
    \pi^\ast_o  = \underset{\pi}{\argmin} R_\text{sup}(\pi, \pi^o) =
    \left\{\begin{array}{c c}
      1, & \pi^\ast_\mathbbm{O}(x) = 1,\\
      0, & \pi^\ast_\mathbbm{1}(x) = 0,\\
      \mathbbm{1}\{U_b(x) \geq - L_b(x)\}, & \pi^\ast_\mathbbm{O}(x) = 0,\pi^\ast_\mathbbm{1}(x) = 1.
    \end{array} 
    \right. 
  \]

  To complete the proof, we need to compute $L_b(x)$ and $U_b(x)$. 
  First, we begin with the case where $u_g \geq u_l$. In this case,
  \[
    \begin{aligned}
      L_b(x) & = \{u_l + (u_g - u_l)\delta_\tau(x)\}\tau(x) - c = \{u_l + (u_g - u_l)\delta_\tau(x)\}m(1,x) - \{u_l + (u_g - u_l)\delta_\tau(x)\}m(0,x) - c,\\
      U_b(x) & = \{u_g - (u_g - u_l)\delta_+(x)\}m(1,x) - \{u_l + (u_g - u_l)\delta_+(x)\}m(0,x) + (u_g - u_l)\delta_+(x) - c.
    \end{aligned}  
  \]
  This gives the form of the worst-case regret relative to $\pi^\mathbbm{1}$ and
  $\pi^\mathbbm{O}$.
  For the worst-case regret relative to the oracle, we collect terms to get
  \[
    \hspace{-1.9cm}
    \left(\begin{array}{c}
      c^{\pi^o}_1(x)\\  c^{\pi^o}_0(x)\\ c^{\pi^o}(x)
    \end{array}
    \right) =  \left\{
      \begin{array}{l l}
        (u_l + (u_g - u_l)\delta_\tau(x), -u_l - (u_g - u_l)\delta_\tau(x), - c), & \pi^\ast_{\mathbbm{1}}(x) = 0,\\
        (u_g - (u_g - u_l)\delta_+(x), -u_l - (u_g - u_l)\delta_+(x), (u_g - u_l)\delta_+(x) - c), & \pi^\ast_{\mathbbm{O}}(x) = 1,\\
        (u_l + u_g  + (u_g - u_l)(\delta_\tau(x) - \delta_+(x)), -2u_l - (u_g - u_l)(\delta_\tau(x) + \delta_+(x)), (u_g - u_l)\delta_+(x) - 2c), & \pi^\ast_{\mathbbm{O}}(x) \neq \pi^\ast_{\mathbbm{1}}(x).
      \end{array}\right. 
  \]

  Now for the case where $u_g < u_l$, the lower and upper bounds switch:
  \[
    \begin{aligned}
      L_b(x) & = \{u_g - (u_g - u_l)\delta_+(x)\}m(1,x) - \{u_l + (u_g - u_l)\delta_+(x)\}m(0,x) + (u_g - u_l)\delta_+(x) - c,\\
      U_b(x) & = \{u_l + (u_g - u_l)\delta_\tau(x)\}\tau(x) - c = \{u_l + (u_g - u_l)\delta_\tau(x)\}m(1,x) - \{u_l + (u_g - u_l)\delta_\tau(x)\}m(0,x) - c.
    \end{aligned}  
  \]
  For the worst-case regret relative to the oracle, we collect terms to get
  \[
    \hspace{-2cm}
    \left(\begin{array}{c}
      c^{\pi^o}_1(x)\\  c^{\pi^o}_0(x)\\ c^{\pi^o}(x)
    \end{array}
    \right) =  \left\{
      \begin{array}{l l}
        (u_g - (u_g - u_l)\delta_+(x), -u_l - (u_g - u_l)\delta_+(x), (u_g - u_l)\delta_+(x) - c), & \pi^\ast_{\mathbbm{1}}(x) = 0,\\
        (u_l + (u_g - u_l)\delta_\tau(x), -u_l - (u_g - u_l)\delta_\tau(x), - c), & \pi^\ast_{\mathbbm{O}}(x) = 1,\\
        (u_l + u_g  + (u_g - u_l)(\delta_\tau(x) - \delta_+(x)), -2u_l - (u_g - u_l)(\delta_\tau(x) + \delta_+(x)), (u_g - u_l)\delta_+(x) - 2c), & \pi^\ast_{\mathbbm{O}}(x) \neq \pi^\ast_{\mathbbm{1}}(x).
      \end{array}\right. 
  \]
  
\end{proof}

For the Proofs of Theorems ~\ref{thm:excess_regret_ones_policy}
\ref{thm:excess_regret_ones_policy_plugin} and \ref{thm:excess_regret_oracle},
we prove the result for the case where $u_g \geq u_l$.
The case where $u_g < u_l$ follows in the same way, with $\pi_\mathbbm{O}$ taking the place for $\pi_\mathbbm{1}$ 

\begin{proof}[Proof of Theorem~\ref{thm:excess_regret_ones_policy}]
  This follows directly from combining Lemmas~\ref{lem:regret_decomp} and \ref{lem:regret_posclass} below via the union bound.
\end{proof}

\begin{proof}[Proof of Theorem~\ref{thm:excess_regret_oracle}]
  This follows directly from combining Lemmas~\ref{lem:regret_decomp} and \ref{lem:regret_oracle} below via the union bound.
\end{proof}

\begin{proof}[Proof of Corollary~\ref{cor:excess_regret_ones_policy_plugin_posclass}]
  This follows by combining Theorem~\ref{thm:excess_regret_ones_policy} and Lemma~\ref{lem:regret_posclass_margin} below.
\end{proof}

\begin{proof}[Proof of Theorem~\ref{thm:excess_regret_ones_policy_plugin}]
  This follows from Lemmas~\ref{lem:regret_posclass_margin} and \ref{lem:regret_plugin_margin} below.
\end{proof}

  \begin{proof}[Proof of Corollary~\ref{cor:excess_regret_oracle_plugin_all}]
    This follows by combining Theorem~\ref{thm:excess_regret_oracle}, Theorem~\ref{thm:excess_regret_ones_policy_plugin}, and Lemma~\ref{lem:regret_posclass_margin} below.
  \end{proof}

\subsection{Auxiliary lemmas}

\begin{lemma}
  \label{lem:regret_decomp}
  Let $\hat{\pi}$ be a solution to Equation \eqref{eq:opt_pol_emp} with nuisance functions
  $\hat{m}$ and $\hat{d}$ fit on a separate sample, and let $\pi^\ast$ be a solution to Equation \eqref{eq:minimax}. Under the strict overlap condition in Assumption~\ref{a:ignore},
  the excess worst-case regret between $\hat{\pi}$ and $\pi^\ast$ is bounded by
  \begin{align*}
    R_\text{sup}(\hat{\pi}, \varpi) - R_\text{sup}(\pi^\ast, \varpi) & \leq U \times \left\{\frac{6 + \eta}{\eta} \times \left(2\mathcal{R}_n(\Pi) +  \frac{t}{\sqrt{n}}\right)+   \sum_{w=0,1}\left\|\hat{\gamma}_w - \gamma_w\right\|_2 \left\|\hat{m}(w,\cdot) - m(w,\cdot)\right\|_2\right\}\\
    & \qquad + \sup_{\pi \in \Pi}\left|\tilde{R}(\pi, \hat{c}(\cdot), m(\cdot); \varpi) - \tilde{R}(\pi, c(\cdot), m(\cdot); \varpi)\right|,
  \end{align*}
  with probability at least $1 - \exp\left(-\frac{t^2}{2}\right)$, where
  \[
  \tilde{R}_\text{sup}(\pi, c(\cdot), m(\cdot); \varpi) = \frac{1}{n}\sum_{i=1}^n \pi(X_i) \left\{c_1(X_i) m(1, X_i) + c_0(X_i)m(0, X_i) + c(X_i) \right\}.
\]
and $U$ is a constant depending on the utility values.

\end{lemma}

\begin{proof}[Proof of Lemma~\ref{lem:regret_decomp}]
  First, note that the excess regret can be decomposed into
  \[
    \begin{aligned}
      R_\text{sup}(\hat{\pi}, \varpi) - R_\text{sup}(\pi^\ast, \varpi) & =
      R_\text{sup}(\hat{\pi}, \varpi)  - \hat{R}_\text{sup}(\hat{\pi}, \varpi) + \underbrace{\hat{R}_\text{sup}(\hat{\pi}, \varpi) - \hat{R}_\text{sup}(\pi^\ast, \varpi)}_{\leq 0}
      + \hat{R}_\text{sup}(\pi^\ast, \varpi) - R_\text{sup}(\pi^\ast, \varpi)\\
      & \leq 2 \sup_{\pi \in \Pi}|\hat{R}_\text{sup}(\pi, \varpi) - R_\text{sup}(\pi, \varpi)|,
    \end{aligned}
  \]
  where we have used that $\hat{\pi}$ minimizes  $\hat{R}_\text{sup}(\pi^\ast, \varpi)$.
  
  We further decompose $\hat{R}_\text{sup}(\pi, \varpi) - R_\text{sup}(\pi, \varpi)$ into
  \begin{align*}
    \hat{R}_\text{sup}(\pi, \varpi) - R_\text{sup}(\pi, \varpi) & = \hat{R}_\text{sup}(\pi, \varpi)  - \tilde{R}(\pi, \hat{c}(\cdot), m(\cdot); \varpi) \tag{\text{a}}\\
    & + \tilde{R}(\pi, c(\cdot), m(\cdot); \varpi) - R_\text{sup}(\pi, \varpi) \tag{\text{b}}\\
    & + \tilde{R}(\pi, \hat{c}(\cdot), m(\cdot); \varpi) - \tilde{R}(\pi, c(\cdot), m(\cdot); \varpi)
  \end{align*}

  We will now control terms (a) and (b), following closely the proof of Lemma 4 in \citet{Athey2021}.
  First note that we have the decompositions
  \begin{align*}
    \hat{\Gamma}_1(X, D, Y) - m(1,X) &= \hat{m}(1,X) - m(1,X) + \frac{D}{\hat{d}(X)}\left\{Y - \hat{m}(1,X)\right\}\\
    & = \left\{\hat{m}(1,X) - m(1,X)\right\} \times \left(1 - \frac{D}{d(X)}\right) + \frac{D}{\hat{d}(X)} \left\{Y - m(1,X)\right\}\\
    & \qquad  + \left(\frac{D}{\hat{d}(X)} - \frac{D}{d(X)}\right) \times \left\{m(1,X) - \hat{m}(1,X)\right\}
  \end{align*}
and
\begin{align*}
  \hat{\Gamma}_0(X, D, Y) - m(0,X) &= \hat{m}(0,X) - m(0,X) + \frac{1 - D}{1 - \hat{d}(X)}\left\{Y - \hat{m}(0,X)\right\}\\
  & = \left\{\hat{m}(0,X) - m(0,X)\right\} \times \left(1 - \frac{1 - D}{1 - d(X)}\right) + \frac{1 - D}{1 - \hat{d}(X)} \left\{Y - m(0,X)\right\}\\
  & \qquad  + \left(\frac{1 - D}{1 - \hat{d}(X)} - \frac{1 - D}{1 - d(X)}\right) \times \left\{m(0,X) - \hat{m}(0,X)\right\}.
\end{align*}

With this, we can compute the expectation of term (a):
\begin{align*}
  \E[(a)] &= \E\left[\pi(X) \left(\hat{c}_1(X)\left\{\hat{\Gamma}_1(X, D, Y) - m(1,X)\right\} + \hat{c}_0(X)\left\{\hat{\Gamma}_0(X, D, Y) - m(0,X)\right\}\right)\right]\\
  & = \E\left[\pi(X) \hat{c}_1(X) \left(\frac{D}{\hat{d}(X)} - \frac{D}{d(X)}\right) \times \left(m(1,X) - \hat{m}(1,X)\right)\right]\\
  & \qquad + \E\left[\pi(X) \hat{c}_0(X) \left(\frac{1 - D}{1 - \hat{d}(X)} - \frac{1 - D}{1 - d(X)}\right) \times \left(m(0,X) - \hat{m}(0,X)\right)\right],
\end{align*}
where we have used the fact that
\begin{align*}
  \E\left[\pi(X) \hat{c}_1(X) \left(\hat{m}(1,X) - m(1,X)\right) \times \left(1 - \frac{D}{d(X)}\right)\right] & = 0,\\
  \E\left[\pi(X) \hat{c}_1(X)  \frac{D}{\hat{d}(X)} \left\{Y - m(1,X)\right\} \right] & = 0,\\
  \E\left[\pi(X) \hat{c}_0(X) \left\{\hat{m}(0,X) - m(0,X)\right\} \times \left(1 - \frac{1 - D}{1 - d(X)}\right)\right] & = 0,\\
  \E\left[\pi(X) \hat{c}_0(X)  \frac{1 - D}{1 - \hat{d}(X)} \left\{Y - m(0,X)\right\} \right]& = 0,\\
\end{align*}
because $\hat{c}$, $\hat{m}$, and $\hat{d}$ come from a different sample.
  
The expectation of term (b) is
\begin{align*}
  \E[(b)] & = \E[\pi(X)\left\{c_1(X)m(1,X) + c_0(X) m(0,X) + c(X)\right\}] - R_\text{sup}(\pi, \varpi) = 0.
\end{align*}
Now define a function  $f:\mathcal{X} \to \R$ as
\begin{eqnarray*}
  f_\pi(x,d , y) &\equiv & \pi(x) \left[\hat{c}_1(x)\left\{\hat{\Gamma}_1(x, d, y) - m(1,x)\right\} + \hat{c}_0(x)\left\{\hat{\Gamma}_0(x, d, y) - m(0,x)\right\} \right]\\
  &&  + \pi(x) \left\{c_1(x)m(1,x) + c_0(x) m(0,x) + c(x)\right\}
\end{eqnarray*}
and the function class $\mathcal{F} \equiv \{f_\pi \mid \pi \in \Pi\}$ as the set of all functions $f$ as we vary $\pi$ in $\Pi$.

With this notation, we can write the sum of terms (a) and (b) as

\[
  (a) + (b)  = \frac{1}{n}\sum_{i=1}^n f_\pi(X_i, D_i, Y_i) - R_\text{sup}(\pi, \varpi),
\]
and from above the expectation of $f_\pi(X_i, D_i, Y_i)$ is
\begin{align*}
  \E[f_\pi(X, D, Y)] & = R_\text{sup}(\pi, \varpi) + \E\left[\pi(X) \hat{c}_1(X) \left(\frac{D}{\hat{d}(X)} - \frac{D}{d(X)}\right) \times \left(m(1,X) - \hat{m}(1,X)\right)\right]\\
  & \qquad + \E\left[\pi(X) \hat{c}_0(X) \left(\frac{1 - D}{1 - \hat{d}(X)} - \frac{1 - D}{1 - d(X)}\right) \times \left(m(0,X) - \hat{m}(0,X)\right)\right].
\end{align*}

Putting together the pieces, we can write
\begin{align*}
  |(\text{a}) + (\text{b})| & = \left| \frac{1}{n}\sum_{i=1}^n f_\pi(X_i, D_i, Y_i) - R_\text{sup}(\pi, \varpi) \right|\\
  & = \left| \frac{1}{n}\sum_{i=1}^n f_\pi(X_i, D_i, Y_i) - \E[f_\pi(X, D, Y)] + \E[f_\pi(X, D, Y)] -  R_\text{sup}(\pi, \varpi)\right|\\
  & \leq \left| \frac{1}{n}\sum_{i=1}^n f_\pi(X_i, D_i, Y_i) - \E[f_\pi(X, D, Y)] \right|\\
  & \qquad  +    \left| \E\left[\pi(X) \hat{c}_1(X) \left(\frac{D}{\hat{d}(X)} - \frac{D}{d(X)}\right) \times \left(m(1,X) - \hat{m}(1,X)\right)\right]\right|\\
  & \qquad + \left|\E\left[\pi(X) \hat{c}_0(X) \left(\frac{1 - D}{1 - \hat{d}(X)} - \frac{1 - D}{1 - d(X)}\right) \times \left(m(0,X) - \hat{m}(0,X)\right)\right] \right|.
\end{align*}



Now notice that for $\varpi \in \{\pi^\mathbbm{O}, \pi^\mathbbm{1}, \pi^o\}$, $|c_1(x)m(1,x) + c_0(x) m(0,x) + c(x)|$, $|c_1(x)|$, and $|c_0(x)|$ are bounded by some constant $U$ depending on the utilities.
From the decompositions above, by the strict overlap condition in Assumption~\ref{a:ignore}, and because $Y_i \in \{0,1\}$,
\begin{align*}
  \left|\hat{\Gamma}_1(X_i, D_i, Y_i) - m(1,x)\right| & \leq  \left|\{\hat{m}(1,X_i) - m(1,X_i)\} \times \left(1 - \frac{D_i}{d(X_i)}\right)\right|\\
  & \qquad + \left|\frac{D_i}{\hat{d}(X_i)} \times \{Y_i - m(1, X_i)\}\right|\\
  & \qquad + \left|\left(\frac{D_i}{d(X_i)} - \frac{D_i}{\hat{d}(X_i)}\right)\times \left\{\hat{m}(1,X_i) - m(1,X_i)\right\}\right|\\
  & \leq \frac{1 - \eta}{\eta}\|\hat{m} - m\|_\infty + \frac{1}{\eta} + \left\|\frac{1}{d} - \frac{1}{\hat{d}}\right\|_\infty\|\hat{m} - m\|_\infty\\
  & \leq \frac{1 - \eta}{\eta} + \frac{1}{\eta} + \frac{1}{\eta} - \frac{1}{1-\eta}\\
  & \leq \frac{3}{\eta}.
\end{align*}
Similarly,
\[
  \left|\hat{\Gamma}_0(X_i, D_i, Y_i) - m(0,x) \right| \leq \frac{1 - \eta}{\eta}\|\hat{m} - m\|_\infty + \frac{1}{\eta} + \left\|\frac{1}{1 - d} - \frac{1}{1 - \hat{d}}\right\|_\infty\|\hat{m} - m\|_\infty \leq \frac{3}{\eta}.
\]
This combines to give that for any $x,d,y$, 
\[
  |f_\pi(x,d,y)| \leq U \times \frac{6 + \eta}{\eta}.
\]
This also shows that the Rademacher complexity of $\mathcal{F}$ is:
\[
  \mathcal{R}_n(\mathcal{F}) =  2U \times \frac{6 + \eta}{\eta} \times \mathcal{R}_n(\Pi).
\]
So by \citet{wainwright_2019} Theorem 4.2, for any $n \geq 1$ and $t \geq 0$,
\[
  \sup_{f \in \mathcal{F}} \left|\frac{1}{n}\sum_{i=1}^n f(X_i) - \E[f(X)]\right| \leq 2U \times \frac{6 + \eta}{\eta} \times \left(2\mathcal{R}_n(\Pi) +  \frac{t}{\sqrt{n}}\right),
\]
with probability at least $1 - \exp\left(-\frac{t^2}{2}\right)$.

Finally, notice that by the Cauchy-Schwarz inequality,
\begin{align*}
  & \left| \E\left[\pi(X) \hat{c}_1(X) \left(\frac{D}{\hat{d}(X)} - \frac{D}{d(X)}\right) \times \left(m(1,X) - \hat{m}(1,X)\right)\right]\right|\\
  & \qquad \leq U\sqrt{\E\left[\left(\frac{D}{\hat{d}(X)} - \frac{D}{d(X)}\right)^2\right] \E\left[\left(m(1,X) - \hat{m}(1,X)\right)^2\right] },
\end{align*}

and
\begin{align*}
  & \left| \E\left[\pi(X) \hat{c}_0(X) \left(\frac{1 - D}{1 - \hat{d}(X)} - \frac{1 - D}{1 - d(X)}\right) \times \left(m(0,X) - \hat{m}(0,X)\right)\right]\right|\\
  & \qquad \leq U\sqrt{\E\left[\left(\frac{1 - D}{1 - \hat{d}(X)} - \frac{1 - D}{1 - d(X)}\right)^2\right] \E\left[\left(m(0,X) - \hat{m}(0,X)\right)^2\right] }.
\end{align*}

Combining these two bounds gives the result.
\end{proof}

\begin{lemma}
  \label{lem:regret_posclass}
    For $u_g \geq u_l$,
    \[
      \sup_{\pi \in \Pi}\left|\tilde{R}_\text{sup}(\pi, \hat{c}, m; \pi^\ast_{\mathbbm{1}}) - \tilde{R}_\text{sup}(\pi, c, m, \pi^\ast_{\mathbbm{1}})\right| \leq (u_g - u_l) \times \left(R_+(\hat{\delta}_+) + \frac{t}{2\sqrt{n}}\right),
    \]
    with probability at least $1 - e^{-\frac{t^2}{2}}$.
\end{lemma}

\begin{proof}[Proof of Lemma~\ref{lem:regret_posclass}]

First we have the bound, 
  \begin{align*}
    \tilde{R}_\text{sup}(\pi, \hat{c}, m; \pi^\ast_{\mathbbm{1}}) - \tilde{R}_\text{sup}(\pi, c, m, \pi^\ast_{\mathbbm{1}}) & = \frac{u_g - u_l}{n}\sum_{i=1}^n\pi(X_i)\left\{\hat{\delta}_+(X_i) - \delta_+(X_i)\right\}\left\{m(1, X_i)  + m(0, X_i) - 1\right\}\\
    & \leq \frac{u_g - u_l}{n}\sum_{i=1}^n\bbone\left\{\hat{\delta}_+(X_i) \neq \delta_+(X_i)\right\}\left|m(1, X_i)  + m(0, X_i) - 1\right|.
  \end{align*}
  
  Now note that
  \begin{align*}
     \E\left[\bbone\left\{\hat{\delta}_+(X_i) \neq \delta_+(X_i)\right\}\left|m(1, X_i)  + m(0, X_i) - 1\right|\right] = R_+(\hat{\delta}_+)
  \end{align*}

For each $i$, since $\bbone\left\{\hat{\delta}_+(X_i) \neq \delta_+(X_i)\right\}\left|m(1, X_i)  + m(0, X_i) - 1\right|$ is bounded between 0 and 1, it is sub-Gaussian with scale parameter 1. Furthermore, they are independent across $i=1,\ldots,n$, so by the Hoeffding bound,
  \[
    \Pr\left(\frac{1}{n}\sum_{i=1}^n\bbone\left\{\hat{\delta}_+(X_i) \neq \delta_+(X_i)\right\}\left|m(1, X_i)  + m(0, X_i) - 1\right| \leq R_+(\hat{\delta}_+)+ \frac{t}{\sqrt{n}}\right) \geq 1 - \exp\left(-2t^2\right).
  \]

  Combining this with the deterministic bound above gives the result.

\end{proof}

\begin{lemma}
  \label{lem:regret_oracle}
  For $u_g \geq u_l$,
  \begin{eqnarray*}
   && \sup_{\pi \in \Pi} \left|\tilde{R}_\text{sup}(\pi, \hat{c}, m; \pi^\ast_{o}) - \tilde{R}_\text{sup}(\pi, c, m, \pi^\ast_{o})\right| \\
    & \leq &
    2 \times \left\{R_\text{sup}(\hat{\pi}_\mathbbm{1}, \pi^\mathbbm{1}) -  R_\text{sup}(\pi^\ast_\mathbbm{1}, \pi^\mathbbm{1})\right\} +  2 \times \left\{R_\text{sup}(\hat{\pi}_\mathbbm{O}, \pi^\mathbbm{O}) -  R_\text{sup}(\pi^\ast_\mathbbm{O}, \pi^\mathbbm{O})\right\}\\
    && \quad +  (u_g - u_l) \times \left(R_+(\hat{\delta}_+) + R_\tau(\hat{\delta}_\tau) + \frac{t}{2\sqrt{n}}\right),
  \end{eqnarray*}
  with probability at least $1 - 2e^{-\frac{t^2}{2}}$.
\end{lemma}

\begin{proof}[Proof of Lemma~\ref{lem:regret_oracle}]
  Define
  \begin{align*}
    \check{L}_b(x)& =  \{u_l + (u_g - u_l)\hat{\delta}_\tau(x)\}m(1,x) - \{u_l + (u_g - u_l)\hat{\delta}_\tau(x)\}m(0,x) - c,\\
    \check{U}_b(x) & = \{u_g - (u_g - u_l)\hat{\delta}_+(x)\}m(1,x) - \{u_l + (u_g - u_l)\hat{\delta}_+(x)\}m(0,x) + (u_g - u_l)\hat{\delta}_+(x) - c,\\
    Q(x) & = \pi^\ast_\mathbbm{O}(x)U_b(x) + (1 - \pi^\ast_{\mathbbm{1}}(x)) L_b(x) + (1 - \pi^\ast_\mathbbm{O}(x))\pi^\ast_\mathbbm{1}(x)(U_b(x) + L_b(x)),\\
    \tilde{Q}(x) & = \hat{\pi}_\mathbbm{O}(x)U_b(x) + (1 - \hat{\pi}_{\mathbbm{1}}(x)) L_b(x) + (1 - \hat{\pi}_\mathbbm{O}(x))\hat{\pi}_\mathbbm{1}(x)(U_b(x) + L_b(x)),\\
    \check{Q}(x) & = \hat{\pi}_\mathbbm{O}(x)\check{U}_b(x) + (1 - \hat{\pi}_{\mathbbm{1}}(x)) \check{L}_b(x) + (1 - \hat{\pi}_\mathbbm{O}(x))\hat{\pi}_\mathbbm{1}(x)(\check{U}_b(x) + \check{L}_b(x)).
  \end{align*}

  With these definitions, we can write
  \begin{align*}
    \left|\tilde{R}_\text{sup}(\pi, \hat{c}, m; \pi^\ast_{o}) - \tilde{R}_\text{sup}(\pi, c, m, \pi^\ast_{o})\right| & = \left|\frac{1}{n}\sum_{i=1}^n \pi(X) \{Q(X_i) - \check{Q}(X_i)\}\right|\\
    & = \left|\frac{1}{n}\sum_{i=1}^n \pi(X) \{Q(X_i) - \tilde{Q}(X_i)\} + \frac{1}{n}\sum_{i=1}^n \pi(X) \{\tilde{Q}(X_i) - \check{Q}(X_i)\}\right|\\
    & \leq \frac{1}{n}\sum_{i=1}^n \left|Q(X_i) - \tilde{Q}(X_i)\right| + \frac{1}{n}\sum_{i=1}^n \left|\tilde{Q}(X_i) - \check{Q}(X_i)\right|.
  \end{align*}
  Working with the first term:
  \begin{align*}
     Q(x) - \tilde{Q}(x)  & = (\hat{\pi}_\mathbbm{1}(x) - \pi^\ast_\mathbbm{1}(x)) U_b(x) - (\hat{\pi}_\mathbbm{1}(x) \hat{\pi}_\mathbbm{O}(x) -  \pi^\ast_\mathbbm{1}(x)\pi^\ast_\mathbbm{O}(x) )(U_b(x) + L_b(x)) + (\hat{\pi}_\mathbbm{O}(x) - \pi^\ast_\mathbbm{O}(x)) U_b(x)\\
     & = (\hat{\pi}_\mathbbm{O}(x) - \pi^\ast_\mathbbm{O}(x)) \times (-L_b(x) \pi^\ast_\mathbbm{1}(x) + (1 - \pi^\ast_\mathbbm{1}(x)) U_b(x)) \tag{$\ast$}\\
     & \qquad + (\hat{\pi}_\mathbbm{1}(x) - \pi^\ast_\mathbbm{1}(x)) \times (-L_b(x) \hat{\pi}_\mathbbm{O}(x) + (1 - \hat{\pi}_\mathbbm{O}(x)) U_b(x) \tag{$\ast \ast$}) 
  \end{align*}

  Notice that $\pi^\ast_\mathbbm{1}(x) = 0 \Leftrightarrow U_b(x) \leq 0$, since $L_b(x) \leq U_b(x)$, this implies that when $\pi^\ast_\mathbbm{1}(x) = 0$, $|U_b(x)| \leq |L_b(x)|$. Therefore,
  \[
    |(\ast)| \leq \bbone\{\hat{\pi}_\mathbbm{O}(x)  \neq \pi^\ast_\mathbbm{O}(x) \} |L_b(x)|.
  \]
  Similarly, if $\pi^\ast_\mathbbm{O}(x) = 1$, then $0 \leq L_b(x) \leq U_b(x)$, so $|L_b(x)| \leq |U_b(x)|$. So,
  \begin{align*}
    |(\ast \ast)| & \leq \bbone\{\hat{\pi}_\mathbbm{1}(x)  \neq \pi^\ast_\mathbbm{1}(x) \} \bbone\{\hat{\pi}_\mathbbm{O}(x)  = \hat{\pi}_\mathbbm{O}(x)\}|-L_b(x) \pi^\ast_\mathbbm{O}(x) + (1 - \pi^\ast_\mathbbm{O}(x)) U_b(x)|\\
    & \qquad  + \bbone\{\hat{\pi}_\mathbbm{O}(x)  \neq \pi^\ast_\mathbbm{O}(x) \} \bbone\{\hat{\pi}_\mathbbm{1}(x)  \neq \pi^\ast_\mathbbm{1}(x) \} |-\hat{\pi}_\mathbbm{O}(x)L_b(x) + (1 - \hat{\pi}_\mathbbm{O}(x))U_b(x)|\\
    & \leq  \bbone\{\hat{\pi}_\mathbbm{1}(x)  \neq \pi^\ast_\mathbbm{1}(x) \} \bbone\{\hat{\pi}_\mathbbm{O}(x)  = \hat{\pi}_\mathbbm{O}(x)\}|U_b(x)|\\
    & \qquad + \bbone\{\hat{\pi}_\mathbbm{O}(x)  \neq \pi^\ast_\mathbbm{O}(x) \} \bbone\{\hat{\pi}_\mathbbm{1}(x)  \neq \pi^\ast_\mathbbm{1}(x) \} |L_b(x)| + \bbone\{\hat{\pi}_\mathbbm{O}(x)  \neq \pi^\ast_\mathbbm{O}(x) \} \bbone\{\hat{\pi}_\mathbbm{1}(x)  \neq \pi^\ast_\mathbbm{1}(x) \} |U_b(x)|\\
    & \leq  \bbone\{\hat{\pi}_\mathbbm{1}(x)  \neq \pi^\ast_\mathbbm{1}(x) \} |U_b(x)| + \bbone\{\hat{\pi}_\mathbbm{O}(x)  \neq \pi^\ast_\mathbbm{O}(x) \} |L_b(x)| + \bbone\{\hat{\pi}_\mathbbm{1}(x)  \neq \pi^\ast_\mathbbm{1}(x) \} |U_b(x)|\\
    & \leq 2\bbone\{\hat{\pi}_\mathbbm{1}(x)  \neq \pi^\ast_\mathbbm{1}(x) \} |U_b(x)|  + \bbone\{\hat{\pi}_\mathbbm{O}(x)  \neq \pi^\ast_\mathbbm{O}(x) \} |L_b(x)| .
  \end{align*}

  Putting together the pieces, we get that
  \[
    |Q(x) - \tilde{Q}(x) | \leq 2\bbone\{\hat{\pi}_\mathbbm{O}(x)  \neq \pi^\ast_\mathbbm{O}(x) \} |L_b(x)| + 2\bbone\{\hat{\pi}_\mathbbm{1}(x)  \neq \pi^\ast_\mathbbm{1}(x) \} |U_b(x)|.
  \]
  So the expectation is bounded by two regret terms:
  \begin{align*}
    \E\left[\frac{1}{n}\sum_{i=1}^n \left|Q(X_i) - \tilde{Q}(X_i)\right| \right] & \leq 2 \E\left[\bbone\{\hat{\pi}_\mathbbm{O}(X)  \neq \pi^\ast_\mathbbm{O}(X) \} |L_b(X)| \right] + 2\E\left[\bbone\{\hat{\pi}_\mathbbm{1}(X)  \neq \pi^\ast_\mathbbm{1}(X) \} |U_b(X)|\right]\\
    & = 2 \times \{R_\text{sup}(\hat{\pi}_\mathbbm{1}, \pi^\mathbbm{1}) -  R_\text{sup}(\pi^\ast_\mathbbm{1}, \pi^\mathbbm{1})\} +  2 \times \{R_\text{sup}(\hat{\pi}_\mathbbm{O}, \pi^\mathbbm{O}) -  R_\text{sup}(\pi^\ast_\mathbbm{O}, \pi^\mathbbm{O})\}.
  \end{align*}
  Next, $\left|Q(X_i) - \tilde{Q}(X_i)\right| $ is bounded between 0 and $u_g - u_l$, so by the Hoeffding bound it concentrates around its expectation:
  \begin{eqnarray*}
  &&  \Pr\left( \frac{1}{n}\sum_{i=1}^n \left|Q(X_i) - \tilde{Q}(X_i)\right| \leq   2 \{R_\text{sup}(\hat{\pi}_\mathbbm{1}, \pi^\mathbbm{1}) -  R_\text{sup}(\pi^\ast_\mathbbm{1}, \pi^\mathbbm{1})\}+  2\{R_\text{sup}(\hat{\pi}_\mathbbm{O}, \pi^\mathbbm{O}) -  R_\text{sup}(\pi^\ast_\mathbbm{O}, \pi^\mathbbm{O})\}  + \frac{t}{\sqrt{n}}\right)\\
  &  \geq& 1 - \exp\left(-\frac{2t^2}{(u_g - u_l)^2}\right).    
  \end{eqnarray*}

  Now for the second term:
  \begin{align*}
    |\tilde{Q}(x) - \check{Q}(x)| & =  \left|(L_b(x) - \check{L}_b(x)) (1 - \hat{\pi}_\mathbbm{1} + (1 - \hat{\pi}_\mathbbm{O} )\hat{\pi}_\mathbbm{1} ) +  (U_b(x) - \check{U}_b(x))  (\hat{\pi}_\mathbbm{O} +  (1 - \hat{\pi}_\mathbbm{O})\hat{\pi}_\mathbbm{1}) \right|\\
    & \leq | L_b(x) - \check{L}_b(x)| +  |U_b(x) - \check{U}_b(x)|.
  \end{align*}
  To re-write this, notice that
  \begin{align*}
    | L_b(x) - \check{L}_b(x)| & = (u_g - u_l) \bbone\{\hat{\delta}_\tau(x) \neq \delta_\tau(x)\}|m(1,x) - m(0,x)|,\\
    | U_b(x) - \check{U}_b(x)| & = (u_g - u_l) \bbone\{\hat{\delta}_+(x) \neq \delta_+(x)\}|m(1,x) + m(0,x) - 1|.
  \end{align*}
  So,
  \[
    \frac{|\tilde{Q}(x) - \check{Q}(x)|}{u_g - u_l} \leq \bbone\{\hat{\delta}_\tau(x) \neq \delta_\tau(x)\}|m(1,x) - m(0,x)| + \bbone\{\hat{\delta}_+(x) \neq \delta_+(x)\}|m(1,x) + m(0,x) - 1|.
  \]

  Taking the expectation, we see that it is bounded by:
  \begin{align*}
    \frac{1}{u_g - u_l}\frac{1}{n}\sum_{i=1}^n \left|\tilde{Q}(X_i) - \check{Q}(X_i)\right| & \leq \E\left[ \bbone\{\hat{\delta}_\tau(x) \neq \delta_\tau(X)\}|m(1,X) - m(0,X)| \right]\\
    & \qquad + \E\left[\bbone\{\hat{\delta}_+(X) \neq \delta_+(x)\}|m(1,X) + m(0,X) - 1|\right]\\
    & = R_+(\hat{\delta}_+)  + R_\tau(\hat{\delta}_\tau) .
  \end{align*}
  Again noting that $|\tilde{Q}(X_i) - \check{Q}(X_i)|$ is bounded between 0 and $u_g - u_l$, and applying the Hoeffding inequality gives

  \begin{align*}
    \Pr\left( \frac{1}{n}\sum_{i=1}^n \left|\tilde{Q}(X_i) - \check{Q}(X_i)\right| \leq   (u_g - u_l) \times \left(R_+(\hat{\delta}_+) + R_\tau(\hat{\delta}_\tau)  + \frac{t}{\sqrt{n}}\right)\right)
    \geq 1 - \exp\left(-2t^2\right).    
  \end{align*}
  Combining these two bounds via the union bound gives the result.

\end{proof}

\begin{lemma}
  \label{lem:regret_posclass_margin}
  Let $\hat{\delta}_+(x) = \bbone\{\hat{m}(1,x) + \hat{m}(0,x) - 1 \geq 0\}$ and $\hat{\delta}_\tau(x) = \bbone\{\hat{m}(1,x) - \hat{m}(0,x)\}$.
  Under Assumption~\ref{a:margin_posclass},
  \begin{align*}
    R_+(\hat{\delta}_+) & \leq 2^{1+\alpha}C\|\hat{m} - m\|_\infty^{1 + \alpha},\\
    \Pr(\hat{\delta}_+(X) \neq \delta_+(X)) & \leq 2^{\alpha}C\|\hat{m} - m\|_\infty^\alpha.
  \end{align*}
  Under Assumption~\ref{a:margin_tau},
  \begin{align*}
    R_\tau(\hat{\delta}_\tau)  & \leq 2^{1+\alpha}C\|\hat{m} - m\|_\infty^{1 + \alpha},\\
    \Pr(\hat{\delta}_\tau(X) \neq \delta_\tau(X)) & \leq 2^{\alpha}C\|\hat{m} - m\|_\infty^\alpha.
  \end{align*}

\end{lemma}

\begin{proof}[Proof of Lemma~\ref{lem:regret_posclass_margin}]
  This Lemma directly follows Lemma 5.1 in \citet{Audibert2007}.
  Note that if $\hat{\delta}_+(x) \neq \delta_+(x)$, then the error 
  is greater than the margin, i.e., 
  \[|\hat{m}(1,x) - m(1,x) + \hat{m}(0,x) - m(0,x)| \geq |m(1,X) + m(0,X) - 1|\]
  So,
  \begin{align*}
    \Pr(\hat{\delta}_+(X) \neq \delta_+(X))  & \leq \Pr(|\hat{m}(1,X) - m(1,X) + \hat{m}(0,X) - m(0,X)| \geq |m(1,X) + m(0,X) - 1|)\\
    & \leq C(\|\hat{m}(1,\cdot) - m(1,\cdot)\|_\infty + \|\hat{m}(0,\cdot) - m(0,\cdot)\|_\infty)^\alpha.
  \end{align*}
  By a similar argument,
  \begin{align*}
    R_+(\hat{\delta}_+) - R_+(\delta_+)  & = \E\left[\bbone\left\{\hat{\delta}_+(X) \neq \delta_+(X)\right\}\left|m(1, X)  + m(0, X) - 1\right|\right]\\
    & \leq \E\left[\bbone\left\{|\hat{m}(1,X) - m(1,X) + \hat{m}(0,X) - m(0,X)| \geq |m(1,X) + m(0,X) - 1|\right\}\right.\\
    & \qquad \times \left. \left|m(1, X)  + m(0, X) - 1\right|\right]\\
    & \leq \E\left[\bbone\left\{|\hat{m}(1,X) - m(1,X) + \hat{m}(0,X) - m(0,X)| \geq |m(1,X) + m(0,X) - 1|\right\}\right.\\
    & \qquad \times \left. \left|m(1, X) -\hat{m}(1,X) + m(0, X) - \hat{m}(0, X)\right|\right]\\
    & \leq (\|\hat{m}(1,\cdot) - m(1,\cdot)\|_\infty + \|\hat{m}(0, \cdot) - m(0,\cdot)\|_\infty)\\
    & \qquad \times  \Pr(|\hat{m}(1,X) - m(1,X) + \hat{m}(0,X) - m(0,X)| \geq |m(1,X) + m(0,X) - 1|)\\
    & \leq C(\|\hat{m}(1,\cdot) - m(1,\cdot)\|_\infty + \|\hat{m}(0,\cdot) - m(0,\cdot)\|_\infty)^{1 + \alpha}.
  \end{align*}

  Similarly, if $\hat{\delta}_\tau(x) \neq \delta_\tau(x)$, then 
  \[
    |m(1,x) - \hat{m}(1,x) - m(0,x) + \hat{m}(0,x)| \geq |m(1,x) - m(0,x)|.
  \]
  By the same argument as above,
  \[
    \Pr(\hat{\delta}_\tau(X) \neq \delta_\tau(X)) \leq C(\|\hat{m}(1,\cdot) - m(1,\cdot)\|_\infty + \|\hat{m}(1,\cdot) - m(1,\cdot)\|_\infty)^\alpha,
  \]
  and
  \begin{align*}
    R_\tau(\hat{\delta}_\tau) - R_\tau(\delta_\tau)  & = \E\left[\bbone\left\{\hat{\delta}_\tau(X_i) \neq \delta_\tau(X_i)\right\}\left|m(1, X_i)  - m(0, X_i)\right|\right]\\
    & \leq \E\left[\bbone\left\{|m(1,X) - \hat{m}(1,X) - m(0,X) +  \hat{m}(0,X)| \geq |m(1,X) - m(0,X)|\right\}\right.\\
    & \qquad \times \left. \left|m(1, X)  - m(0, X) \right|\right]\\
    & \leq \E\left[\bbone\left\{|m(1,X) - \hat{m}(1,X) - m(0,X) +  \hat{m}(0,X)| \geq |m(1,X) - m(0,X)|\right\}\right.\\
    & \qquad \times \left. \left|m(1,X) - \hat{m}(1,X) - m(0,X) +  \hat{m}(0,X)\right|\right]\\
& \leq (\|\hat{m}(1,\cdot) - m(1,\cdot)\|_\infty + \|\hat{m}(0, \cdot) - m(0,\cdot)\|_\infty)\\
    & \qquad \times  \Pr(|m(1,X) - \hat{m}(1,X) - m(0,X) +  \hat{m}(0,X)| \geq |m(1,X) - m(0,X)|)\\
    & \leq C(\|\hat{m}(1,\cdot) - m(1,\cdot)\|_\infty + \|\hat{m}(0,\cdot) - m(0,\cdot)\|_\infty)^{1 + \alpha}.
  \end{align*}

\end{proof}

\begin{lemma}
  \label{lem:regret_plugin_margin}
  Let $u_g \geq u_l$. Define 
  \begin{align*}
    \widehat{L}_b(x) & = \{u_l + \hat{\delta}_\tau(x)(u_g - u_l)\} \{\hat{m}(1,x) - \hat{m}(0,x)\} - c,\\
    \widehat{U}_b(x) & = \{u_g - (u_g - u_l)\hat{\delta}_+(x)\}\hat{m}(1,x) - \{u_l + (u_g - u_l)\hat{\delta}_+(x)\}\hat{m}(0,x) + (u_g - u_l)\hat{\delta}_+(x) - c.
  \end{align*}
  and let $\hat{\pi}_\mathbbm{O}^\text{plug}(x) = \bbone\{\widehat{L}_b(x)\geq 0\}$ and 
  $\hat{\pi}_\mathbbm{1}^\text{plug}(x) = \bbone\{\widehat{U}_b(x) \geq 0\}$ be the plug-in
  estimates of the minimax optimal policies relative to never or always treating.
  Under Assumption~\ref{a:margin_lower}, the excess
  worst case regret for $\hat{\pi}_\mathbbm{O}^\text{plug}$ relative to
  $\pi_\mathbbm{O}^\ast$ is
  \[
    R_\text{sup}(\hat{\pi}^\text{plug}_\mathbbm{O}, \pi^\mathbbm{O}) -  R_\text{sup}(\pi^\ast_\mathbbm{O}, \pi^\mathbbm{O}) \leq u_g^\alpha C(2\|m - \hat{m}\|_\infty)^{1 + \alpha} + 2u_g C\|m - \hat{m}\|_\infty \Pr\left(\hat{\delta}_\tau(X) \neq \delta_\tau(X)\right) + (u_g - u_l) R_\tau(\hat{\delta}_\tau).
  \]
  Under Assumption~\ref{a:margin_upper}, the excess
  worst case regret for $\hat{\pi}_\mathbbm{1}^\text{plug}$ relative to
  $\pi_\mathbbm{1}^\ast$ is
  \[
    R_\text{sup}(\hat{\pi}^\text{plug}_\mathbbm{1}, \pi^\mathbbm{1}) -  R_\text{sup}(\pi^\ast_\mathbbm{O}, \pi^\mathbbm{O}) \leq u_g^\alpha C(2\|m - \hat{m}\|_\infty)^{1 + \alpha} + 2u_g C\|m - \hat{m}\|_\infty \Pr\left(\hat{\delta}_+(X) \neq \delta_+(X)\right) + (u_g - u_l) R_+(\hat{\delta}_+).
  \]
\end{lemma}

\begin{proof}[Proof of Lemma~\ref{lem:regret_plugin_margin}]

  First, as in the proof of Lemma~\ref{lem:regret_posclass_margin}, note that
  $\hat{\pi}_\mathbbm{O}^\text{plug}(x) \neq \pi_\mathbbm{O}^\ast(x)$ implies
  that $|L_b(x) - \widehat{L}_b(x)| \geq |L_b(x)|$. Now, if $\hat{\delta}_\tau(x) = \delta_\tau(x)$,
  then 
  \begin{align*}
    |L_b(x) - \widehat{L}_b(x)| & = |((1 - \delta_\tau(x))u_l + \delta_\tau(x) u_g)(m(1,x) - \hat{m}(1,x) - m(0,x) + \hat{m}(0,x))|\\
    & \leq u_g |m(1,x) - \hat{m}(1,x) - m(0,x) + \hat{m}(0,x)|,
  \end{align*}
  because $|(1 - \delta_\tau(x)) u_l + \delta_\tau(x)u_g| = |u_l + (u_g - u_l) \delta_\tau(X)| \leq \max\{u_g, u_l\} \leq u_g$ in the case where $u_g \geq u_l$.
  If $\hat{\delta}_\tau(x) \neq \delta_\tau(x)$ and $\delta_\tau(x) = 1$, we have that
  \begin{align*}
    |L_b(x) - \widehat{L}_b(x)| & = |u_l(m(1,x) - \hat{m}(1,x) - m(0,x) + \hat{m}(0,x)) + (u_g - u_l)(m(1,x) - m(0,x))|\\
    & \leq u_l |m(1,x) - \hat{m}(1,x) - m(0,x) + \hat{m}(0,x)| + (u_g - u_l)|m(1,x) - m(0,x)|\\
    & \leq u_g |m(1,x) - \hat{m}(1,x) - m(0,x) + \hat{m}(0,x)| + (u_g - u_l)|m(1,x) - m(0,x)|.
  \end{align*}
  Similarly, if  $\hat{\delta}_\tau(x) \neq \delta_\tau(x)$ and $\delta_\tau(x) = 0$,
  \begin{align*}
    |L_b(x) - \widehat{L}_b(x)| & = |u_g(m(1,x) - \hat{m}(1,x) - m(0,x) + \hat{m}(0,x)) - (u_g - u_l)(m(1,x) - m(0,x))|\\
    & \leq u_g |m(1,x) - \hat{m}(1,x) - m(0,x) + \hat{m}(0,x)| + (u_g - u_l)|m(1,x) - m(0,x)|.
  \end{align*}

  Putting together the pieces, we get that
  \begin{align*}
    R_\text{sup}(\hat{\pi}^\text{plug}_\mathbbm{O}, \pi^\mathbbm{O}) -  R_\text{sup}(\pi^\ast_\mathbbm{O}, \pi^\mathbbm{O}) & = \E\left[ \bbone\{\hat{\pi}^\text{plug}_\mathbbm{O} \neq \pi^\ast_\mathbbm{O}\}|L_b(x)|\right]\\
    & \leq \E\left[\bbone\{|L_b(X) - \widehat{L}_b(X)| \geq |L_b(X)|\}|L(X)|\right]\\
    & \leq \E\left[\bbone\{|L_b(X) - \widehat{L}_b(X)| \geq |L_b(X)|\}|L_b(X) - \widehat{L}_b(X)|\right]\\
    & = \E\left[\bbone\{|L_b(X) - \widehat{L}_b(X)| \geq |L_b(X)|\}|L_b(X) - \widehat{L}_b(X)| \bbone\{\hat{\delta}_\tau(X) = \delta_\tau(X)\}\right] \tag{$\ast$}\\
    & \qquad + \E\left[\bbone\{|L_b(X) - \widehat{L}_b(X)| \geq |L_b(X)|\}|L_b(X) - \widehat{L}_b(X)| \bbone\{\hat{\delta}_\tau(X) \neq \delta_\tau(X)\}\right]. \tag{$\ast \ast$}
  \end{align*}
  By H\"{o}lder's inequality and the margin condition (Assumption~\ref{a:margin_lower}), the first term is
  \begin{align*}
    (\ast) & \leq \E\left[\bbone\{|((1 - \delta_\tau(x))u_l + \delta_\tau(X) u_g)(m(1,x) - \hat{m}(1,x) - m(0,x) + \hat{m}(0,x))|\geq |L(X)|\} \right.\\
    & \qquad \times \left.|m(1,X) - \hat{m}(1,X) - m(0,X) + \hat{m}(0,X)| \right]\\
    & \leq \E\left[\bbone\{u_g|m(1,x) - \hat{m}(1,x) - m(0,x) + \hat{m}(0,x)|\geq |L(X)|\} \right.\\
    & \qquad \times \left.|m(1,X) - \hat{m}(1,X) - m(0,X) + \hat{m}(0,X)| \right]\\
    & \leq \E\left[\bbone\{u_g|m(1,x) - \hat{m}(1,x) - m(0,x) + \hat{m}(0,x)|\geq |L(X)|\} \right] \times 2\|m - \hat{m}\|_\infty\\
    & \leq Cu_g^\alpha (2\|m - \hat{m}\|_\infty)^{1 + \alpha}.
  \end{align*}
  
  Similarly, we can bound the second term as 
  \begin{align*}
    (\ast \ast) & \leq \E\left[|L_b(X) - \widehat{L}_b(X)| \bbone\{\hat{\delta}_\tau(X) \neq \delta_\tau(X)\}\right]\\
    & \leq  \E\left[u_g|m(1,X) - \hat{m}(1,X) - m(0,X) + \hat{m}(0,X)|\bbone\{\hat{\delta}_\tau(X) \neq \delta_\tau(X)\}\right]\\
    & \qquad + (u_g - u_l)\E\left[|m(1,X) - m(0,X)|\bbone\{\hat{\delta}_\tau(X) \neq \delta_\tau(X)\}\right]\\
    & \leq u_g 2C\|m - \hat{m}\|_\infty \Pr\left(\hat{\delta}_\tau(X) \neq \delta_\tau(X)\right) + (u_g - u_l) R_\tau(\hat{\delta}_\tau).
  \end{align*}
  Combining these two terms gives the first result.

  Now, also note that
  $\hat{\pi}_\mathbbm{1}^\text{plug}(x) \neq \pi_\mathbbm{1}^\ast(x)$ implies
  that $|U_b(x) - \widehat{U}_b(x)| \geq |U_b(x)|$. We again break this error term
  into cases depending on $\hat{\delta}_+(x)$ and $\delta_+(x)$.
  First, if $\hat{\delta}_+(x) = \delta_+(x)$, then
  \begin{align*}
    |U_b(x) - \widehat{U}_b(x)| & = \left\{
      \begin{array}{c c}
        |u_g(m(1,x) - \hat{m}(1,x)) - u_l (m(0,x) - \hat{m}(0,x))|, & \delta_+(x) = 0\\
        |u_l(m(1,x) - \hat{m}(1,x)) - u_g (m(0,x) - \hat{m}(0,x))|, & \delta_+(x) = 1\\
      \end{array}
    \right.\\
    & \leq u_g |m(1,x) - \hat{m}(1,x)| + u_g|m(0,x) - \hat{m}(0,x)|.
  \end{align*}
  If $\hat{\delta}_+(x) \neq \delta_+(x)$
  \begin{align*}
    |U_b(x) - \widehat{U}_b(x)| & = \left\{
      \begin{array}{c c}
        |u_g(m(1,x) - \hat{m}(1,x)) - u_l (m(0,x) - \hat{m}(0,x)) + (u_g - u_l) (m(1,x) + m(0,x) - 1)|, & \delta_+(x) = 0\\
        |u_l(m(1,x) - \hat{m}(1,x)) - u_g (m(0,x) - \hat{m}(0,x))| - (u_g - u_l) (m(1,x) + m(0,x) - 1), & \delta_+(x) = 1\\
      \end{array}
    \right.\\
    & \leq u_g |m(1,x) - \hat{m}(1,x)| + u_g|m(0,x) - \hat{m}(0,x)| + (u_g - u_l)|m(1,x) + m(0,x) - 1| .
  \end{align*}
  Mirroring the decomposition above, we have that
    
  \begin{align*}
    R_\text{sup}(\hat{\pi}^\text{plug}_\mathbbm{1}, \pi^\mathbbm{1}) -  R_\text{sup}(\pi^\ast_\mathbbm{1}, \pi^\mathbbm{1}) & = \E\left[ \bbone\{\hat{\pi}^\text{plug}_\mathbbm{1} \neq \pi^\ast_\mathbbm{1}\}|U_b(x)|\right]\\
    & \leq \E\left[\bbone\{|U_b(X) - \widehat{U}_b(X)| \geq |U_b(X)|\}|U_b(X)|\right]\\
    & \leq \E\left[\bbone\{|U_b(X) - \widehat{U}_b(X)| \geq |U_b(X)|\}|U_b(X) - \widehat{U}_b(X)|\right]\\
    & = \E\left[\bbone\{|U_b(X) - \widehat{U}_b(X)| \geq |U_b(X)|\}|U_b(X) - \widehat{U}_b(X)| \bbone\{\hat{\delta}_+(X) = \delta_+(X)\}\right]\\
    & \qquad + \E\left[\bbone\{|U_b(X) - \widehat{U}_b(X)| \geq |U_b(X)|\}|U_b(X) - \widehat{U}_b(X)| \bbone\{\hat{\delta}_+(X) \neq \delta_+(X)\}\right]\\
    & \leq \E\left[\bbone\{u_g |m(1,X) - \hat{m}(1,X)| + u_g|m(0,X) - \hat{m}(0,X)| \geq |U_b(X)|\} \right.\\
    & \qquad \qquad \left.  \times \left(u_g|m(1,X) - \hat{m}(1,X)| + u_g|m(0,X) - \hat{m}(0,X)|\right) \right]\\
    & \qquad + \E\left[u_g |m(1,x) - \hat{m}(1,x)| \bbone\{\hat{\delta}_+(X) \neq \delta_+(X)\}\right]\\
    & \qquad  + \E\left[u_g|m(0,x) - \hat{m}(0,x)| \bbone\{\hat{\delta}_+(X) \neq \delta_+(X)\}\right]\\
    & \qquad + \E\left[(u_g - u_l)|m(1,x) + m(0,x) - 1| \bbone\{\hat{\delta}_+(X) \neq \delta_+(X)\}\right]\\
    & \leq u_g^\alpha C (2\|m - \hat{m}\|_\infty)^{1 + \alpha} + u_g C 2\|m - \hat{m}\|_\infty P(\hat{\delta}_+(X) \neq \delta_+(X))\\
    & \qquad  + (u_g - u_l) R_+(\hat{\delta}_+).
  \end{align*}
\end{proof}

\pdfbookmark[1]{References}{References}
\spacingset{1.56}
\bibliographystyle{apalike}
\bibliography{citations}

\end{document}